\documentclass{article}

\PassOptionsToPackage{round}{natbib}

\usepackage[final]{neurips_2019}

\usepackage[noend]{algorithmic}
\usepackage[utf8]{inputenc} \usepackage[T1]{fontenc}    \usepackage{url}            \usepackage{booktabs}       \usepackage{amsfonts}       \usepackage{nicefrac}       \usepackage{microtype}      \usepackage{floatrow}

\usepackage{bm,xcolor}
\usepackage{mathdots}
\usepackage{nicefrac}       \usepackage{microtype}      \usepackage{amsmath,amsthm}
\usepackage{natbib}
\usepackage{caption}
\usepackage{pifont}
\definecolor{mydarkblue}{rgb}{0,0.08,0.45}
\usepackage[colorlinks=true,
    linkcolor=mydarkblue,
    citecolor=mydarkblue,
    filecolor=mydarkblue,
    urlcolor=mydarkblue]{hyperref} 
\usepackage{todonotes}
\usepackage{mathrsfs} 
\usepackage{booktabs}
\usepackage{dsfont}
\usepackage{subcaption}
\usepackage[export]{adjustbox}
\usepackage[inline, shortlabels]{enumitem}    \usepackage{arydshln} \usepackage{amsmath,amsfonts,bm}
\usepackage{algorithm,algorithmic}

\def\1{\bm{1}}

\def\eps{{\epsilon}}

\def\vmu{{\bm{\mu}}}
\def\vtheta{{\bm{\theta}}}
\def\vphi{{\bm{\varphi}}}
\def\vpi{{\bm{\pi}}}

\def\vomega{{\bm{\omega}}}
\def\valpha{{\bm{\alpha}}}
\def\va{{\bm{a}}}
\def\vb{{\bm{b}}}
\def\vc{{\bm{c}}}
\def\vd{{\bm{d}}}

\def\vg{{\bm{g}}}

\def\vm{{\bm{m}}}

\def\vv{{\bm{v}}}

\def\mA{{\bm{A}}}

\DeclareMathAlphabet{\mathsfit}{\encodingdefault}{\sfdefault}{m}{sl}
\SetMathAlphabet{\mathsfit}{bold}{\encodingdefault}{\sfdefault}{bx}{n}

\newcommand{\E}{\mathbb{E}}

\newcommand{\R}{\mathbb{R}}

\DeclareMathOperator*{\argmin}{arg\,min}

\DeclareMathOperator{\geom}{Geom}

\newlength\myindent
\definecolor{mygreen}{rgb}{0.032, 0.6392, 0.2039}
\definecolor{mypurple}{HTML}{B266FF}

\usepackage{todonotes}

\def\RR{{\mathbb R}}
\def\NN{{\mathbb N}}

\def\uu{{\boldsymbol u}}

\def\vv{{\boldsymbol v}}

\def\LL{{\mathcal L}}

\def\defas{:=}

\newtheorem{assumption}{Assumption}
\newtheorem{theorem}{Theorem}
\newtheorem{lemma}{Lemma}
\newtheorem{proposition}{Proposition}

\newtheorem{definition}{Definition}

\newtheorem*{rep@theorem}{\rep@title}
\newcommand{\newreptheorem}[2]{
\newenvironment{rep#1}[1]{
 \def\rep@title{#2 \ref{##1}}
 \begin{rep@theorem}}
 {\end{rep@theorem}}}
\newreptheorem{theorem}{Theorem'}
\newtheorem*{rep@proposition}{\rep@title}
\newcommand{\newrepproposition}[2]{
\newenvironment{rep#1}[1]{
 \def\rep@title{#2 \ref{##1}}
 \begin{rep@proposition}}
 {\end{rep@proposition}}}
\newreptheorem{proposition}{Proposition'}
\usepackage[verbose]{wrapfig}

\newcommand{\ras}[1]{\renewcommand{\arraystretch}{#1}}  \newfloatcommand{capbtabbox}{table}[][\FBwidth]
\newfloatcommand{capbalgbox}{algorithm}[][\FBwidth]

\newcounter{nbdrafts}
\makeatletter
\newcommand{\checknbdrafts}{
\ifnum \thenbdrafts > 0
\@latex@warning@no@line{**********************************************************************}
\@latex@warning@no@line{* The document contains \thenbdrafts \space draft note(s)}
\@latex@warning@no@line{**********************************************************************}
\fi}

\makeatother
\definecolor{ffcolor}{rgb}{0.0,0.4,0.5}
\title{Reducing Noise in GAN Training with Variance Reduced Extragradient}
\author{  Tatjana Chavdarova\thanks{equal contribution} \\    Mila, Universit\'e de Montr\'eal\\
  Idiap, École Polytechnique Fédérale de Lausanne
  \And
  Gauthier Gidel\textsuperscript{$*$} \\
  Mila, Universit\'e de Montr\'eal \\
  Element AI 
  \And 
  François Fleuret\\
  Idiap, École Polytechnique Fédérale de Lausanne
  \And
  Simon Lacoste-Julien\thanks{Canada CIFAR AI Chair} \\
  Mila, Universit\'e de Montr\'eal
}
\begin{document}

\maketitle

\begin{abstract}
We study the effect of the stochastic gradient noise on the training of generative adversarial networks (GANs) and show that it can prevent the convergence of standard game optimization methods, while the batch version converges. We address this issue with a novel stochastic variance-reduced extragradient (SVRE) optimization algorithm, which for a large class of games improves upon the previous convergence rates proposed in the literature. We observe empirically that SVRE performs similarly to a batch method on MNIST while being computationally cheaper, and that SVRE yields more stable GAN training on standard datasets.
\end{abstract}

\section{Introduction} 
Many empirical risk minimization algorithms rely on gradient-based optimization methods. These iterative methods handle large-scale training datasets by computing gradient estimates on a subset of it, a \textit{mini-batch}, instead of using all the samples at each step, the \textit{full batch}, resulting in a method called \emph{stochastic gradient descent} (SGD,~\citet{robbins1951stochastic, bottou2010SGD}).

SGD methods are known to efficiently minimize \textit{single} objective loss functions, such as cross-entropy for classification or squared loss for regression. Some algorithms go beyond such training objective and define multiple agents with different or competing objectives. The associated optimization paradigm requires a multi-objective joint minimization. An example of such a class of algorithms are the generative adversarial networks~(GANs, \citealp{goodfellow2014generative}), which aim at finding a Nash equilibrium of a two-player \textit{minimax} game, where the players are deep neural networks (DNNs).

As of their success on supervised tasks, SGD based algorithms have been adopted for GAN training as well. Recently, \citet{gidel2019variational}~proposed to use an optimization technique coming from the variational inequality literature called \emph{extragradient}~\citep{korpelevich1976extragradient} with provable convergence guarantees to optimize games (see \S~\ref{sec:gans_as_a_game}).
However, convergence failures, poor performance (sometimes referred to as ``mode collapse''), or hyperparameter susceptibility are more commonly reported compared to classical supervised DNN optimization.

We question naive adoption of such methods for game optimization so as to address the reported training instabilities. We argue that as of the two player setting, noise impedes drastically more the training compared to single objective one.
More precisely, we point out that the noise due to the stochasticity may break the convergence of the extragradient method, by considering a simplistic stochastic bilinear game for which it provably does \textit{not} converge.

The theoretical aspect we present in this paper is further supported empirically, since using larger mini-batch sizes for GAN training has been shown to considerably improve the quality of the samples produced by the resulting generative model: \citet{brock2018large} report a relative improvement of $46\%$ of the Inception Score metric (see \S~\ref{par:metrics}) on ImageNet if the batch size is increased $8$--fold. This notable improvement raises the question if noise reduction optimization methods can be extended to game settings.  In turn, this would allow for a principled training method with the practical benefit of omitting to empirically establish this multiplicative factor for the batch size.

In this paper, we investigate the interplay between noise and multi-objective problems in the context of GAN training. Our contributions can be summarized as follows:
\begin{enumerate*}[series = tobecont, itemjoin = \quad, label=(\roman*)]
\item we show in a motivating example how the noise can make stochastic extragradient fail (see \S~\ref{sub:stochasticity_breaks_extragradient_}).
\item we propose a new method ``stochastic variance reduced extragradient'' (SVRE) that combines variance reduction and extrapolation (see Alg.~\ref{alg:svre} and \S~\ref{sub:extra_svrg}) and show experimentally that it effectively reduces the noise. \item we prove the convergence of SVRE under local strong convexity assumptions, improving over the known rates of competitive methods for a large class of games (see \S~\ref{sub:extra_svrg} for our convergence result and Table~\ref{tab:rates} for comparison with standard methods).
\item we test SVRE empirically to train GANs on several standard datasets, and observe that it can improve SOTA deep models in the late stage of their optimization (see \S~\ref{sec:experiments}).
\end{enumerate*}

\begin{figure}
\begin{minipage}{\linewidth}
\begin{minipage}[t]{.47\linewidth}
\begin{table}[H]
\begin{tabular}{lc@{}c@{}c@{}c}
          \toprule 
           \hspace{-1mm}Method &  \!\!Complexity & \!\!\!\!$\mu$-adaptivity\\
                    \midrule
          \hspace{-1mm}SVRG &  $\ln(\frac{1}{\epsilon}) {\times} (n +  \frac{\bar L^2}{\mu^2})$   & no \\
          \hspace{-1mm}Acc. SVRG 
          &  $\ln(\frac{1}{\epsilon}) {\times} (n + \sqrt{n}\frac{\bar L}{\mu})$  & no  \\
                    \hspace{-1mm}\textbf{SVRE}~\S\ref{sub:extra_svrg}  & $\ln(\frac{1}{\epsilon}) {\times} (n + \frac{\bar \ell}{\mu})$ &  {\small if $\bar \ell = O(\bar L)$} \\           \bottomrule
        \end{tabular}
        \vspace{-2mm}
        \caption{  \small
        Comparison of variance reduced methods for games for a $\mu$-strongly monotone operator with $L_i$-Lipschitz stochastic operators. Our result makes the assumption that the operators are $\ell_i$-cocoercive. Note that $\ell_i \in [L_i,L_i^2/\mu]$, more details and a tighter rate are provided in \S\ref{sub:extra_svrg}.
                                                        The SVRG variants are proposed by~\citet{palaniappan2016stochastic}.
        \emph{$\mu$-adaptivity} indicates if the hyper-parameters that guarantee convergence (step size \& epoch length) depend on the strong monotonicity parameter $\mu$: if not, the algorithm is adaptive to local strong monotonicity. Note that in some cases the constant $\ell$ may depend on $\mu$ but SVRE is adaptive to strong convexity when $\bar \ell$ remains close to $\bar L$ (see for instance Proposition~\ref{prop:coco}).
                                 }
      \label{tab:rates}
      \end{table}
\end{minipage}
\scalebox{.93}{
\begin{minipage}[t]{.55\linewidth}
\begin{algorithm}[H]
\begin{algorithmic}[1]
   \STATE {\bfseries Input:} 
               Stopping time $T$,
              learning rates $\eta_\vtheta, \eta_\vphi$, initial weights $\vtheta_0$, $\vphi_0$. $t=0$
      \WHILE{$t \leq T$}
                    \STATE $\vphi^{\mathcal{S}} = \vphi_t \,$ and $\,\vmu_{\vphi}^\mathcal{S} = \frac{1}{n} \sum_{i=1}^{n} \nabla_{\vphi}  \LL^{D}_i(\vtheta^{\mathcal{S}},\vphi^{\mathcal{S}})$   
        \STATE $\vtheta^{\mathcal{S}} = \vtheta_t\,$ and $\,\vmu_{\vtheta}^\mathcal{S} = \frac{1}{n} \sum_{i=1}^{n} \nabla_{\vtheta}  \LL^{G}_i(  \vtheta^{\mathcal{S}},\vphi^{\mathcal{S}})$
        \STATE $N \sim \geom\big(1/n\big)$ \hfill \emph{(Sample epoch length)}   
        \FOR[\emph{Beginning of the epoch}]{$i=0$ {\bfseries to} $N{-}1$} \label{l:epoch}
            \STATE \textbf{Sample} $i_\vtheta, i_\vphi\sim \pi_\vtheta,\pi_\vphi$, do \textbf{extrapolation:}
            \STATE $\tilde \vphi_{t} = \vphi_t - \eta_\vphi \vd_{i_{\vphi}}^D(\vtheta_t,\vphi_t,\vtheta^{\mathcal{S}},\vphi^{\mathcal{S}})$   
            \hfill $\triangleright$~\eqref{eq:theta_svrg_dir}
            \STATE $\tilde \vtheta_{t} = \vtheta_t - \eta_\vtheta \vd_{i_\vtheta}^{G}(\vtheta_t,\vphi_t,\vtheta^{\mathcal{S}},\vphi^{\mathcal{S}})$ 
            \hfill $\triangleright$~\eqref{eq:theta_svrg_dir}
            \STATE \textbf{Sample} $i_\vtheta, i_\vphi\sim \pi_\vtheta,\pi_\vphi$ and do \textbf{update:}
            \STATE $\vphi_{t+1} = \vphi_t - \eta_\vphi \vd_{i_{\vphi}}^D(\tilde \vtheta_{t},\tilde \vphi_{t},\vtheta^{\mathcal{S}},\vphi^{\mathcal{S}})$      
            \hfill $\triangleright$~\eqref{eq:theta_svrg_dir}
            \STATE $\vtheta_{t+1} = \vtheta_t - \eta_\vtheta \vd_{i_{\vtheta}}^G(\tilde \vtheta_{t}, \tilde \vphi_{t},\vtheta^{\mathcal{S}},\vphi^{\mathcal{S}})$  
            \hfill $\triangleright$~\eqref{eq:theta_svrg_dir}
            \STATE $t \leftarrow t+1$
           \ENDFOR
   \ENDWHILE   
   \STATE {\bfseries Output:} $\vtheta_T$, $\vphi_T$   
\end{algorithmic}
   \caption{Pseudocode for SVRE.}
   \label{alg:svre}
\end{algorithm}
\end{minipage}
}
\end{minipage}
\end{figure}

\section{GANs as a Game and Noise in Games} 
\label{sec:gans_as_a_game}

\subsection{Game theory formulation of GANs}
The models in a GAN are a generator $G$, that maps an embedding space to the signal space, and should eventually map a fixed noise distribution to the training data distribution, and a discriminator $D$ whose purpose is to allow the training of the generator by classifying genuine samples against generated ones. At each iteration of the algorithm, the discriminator $D$ is updated to improve its ``real vs. generated'' classification performance, and the generator $G$ to degrade it.

From a game theory point of view, GAN training is a differentiable two-player game where the generator $G_{\vtheta}$ and the discriminator $D_{\vphi}$ aim at minimizing  their own cost function
$\LL^{G}$ and $\LL^{D}$, resp.:
\begin{equation} \label{eq:two_player_games}
\tag{2P-G}
  \vtheta^* \in \argmin_{\vtheta \in \Theta}\LL^{G}(\vtheta,\bm{\vphi}^*) 
  \qquad \text{and}\qquad
  \bm{\vphi}^* \in \argmin_{\vphi \in \Phi} \LL^{D}(\vtheta^*,\vphi) \,.  
\end{equation} 
When $\LL^{D} = - \LL^{G}=: \LL$ this game is called a \emph{zero-sum game} and~\eqref{eq:two_player_games} is a minimax problem:
\vspace{-1mm}
\begin{equation}\label{eq:zero_sum_game}
\tag{SP}
  \min_{\vtheta \in \Theta} \max_{\vphi \in\Phi} \, \LL(\vtheta,\vphi)
\end{equation}
The gradient method does not converge for some convex-concave examples~\citep{mescheder_numerics_2017,gidel2019variational}. To address this,~\citet{korpelevich1976extragradient} proposed to use the \emph{extragradient} method\footnote{For simplicity, we focus on \emph{unconstrained} setting where $\Theta = \R^d$. For the \emph{constrained} case, a Euclidean projection on the constraints set should be added at every update of the method.
} which performs a lookahead step in order to get signal from an \textit{extrapolated} point:
\begin{equation}
\label{eq:extragradient}
\tag{EG}
  \text{Extrapolation:} 
  \left\{
  \begin{aligned}
  \tilde \vtheta_{t} &= \vtheta_t - \eta \nabla_\vtheta \LL^{G}(\vtheta_t,\vphi_t) \\
  \tilde \vphi_{t} &= \vphi_t - \eta \nabla_{\vphi}\LL^{D} (\vtheta_t ,\vphi_t)
  \end{aligned}
\right. \quad\;
\text{Update:} 
  \left\{\begin{aligned}
  \vtheta_{t+1} &= \vtheta_t - \eta \nabla_\vtheta \LL^{G}(\tilde \vtheta_{t},\tilde \vphi_{t}) \\
  \vphi_{t+1} &= \vphi_t - \eta \nabla_{\vphi}\LL^{D} (\tilde \vtheta_{t},\tilde \vphi_{t})
  \end{aligned}
  \right.
\end{equation}
Note how $\vtheta_t$ and $\vphi_t$ are updated with a gradient from a different point, the \emph{extrapolated} one.
In the context of a zero-sum game, for any \emph{convex-concave} function $\LL$ and any closed convex sets $\Theta$ and $\Phi$, the extragradient method converges~\citep[Thm. 12.1.11]{harker1990finite}. 

\subsection{Stochasticity Breaks Extragradient} 
\label{sub:stochasticity_breaks_extragradient_}
\begin{figure}
\centering
\floatbox[{\capbeside\thisfloatsetup{capbesideposition={right,top},capbesidewidth=6.8cm}}]{figure}[\FBwidth]
{
\captionsetup{singlelinecheck=off}
\caption[.]{Illustration of the discrepancy between games and minimization on simple examples:
$$ \text{\emph{min:}} \;
\min_{\theta,\phi \in \R} \theta^2 + \phi^2\,, 
\quad  \text{\emph{game:}}\; \min_{\theta \in \R} \max_{\phi \in \R} \theta \cdot \phi \,.$$
\textbf{Left: Minimization.} Up to a neighborhood, the noisy gradient always points to a direction that make the iterate closer to the minimum  ({\color{blue}$\bm{\star}$}).
\textbf{Right: Game.} The noisy gradient may point to a direction (red arrow) that push the iterate away from the Nash Equilibrium ({\color{red}$\bm{\star}$}).
}\label{fig:min_vs_game}}
{
\begin{subfigure}{.46\linewidth}
\centering
\includegraphics[width= 1.1\linewidth, height= 1.4 \linewidth]{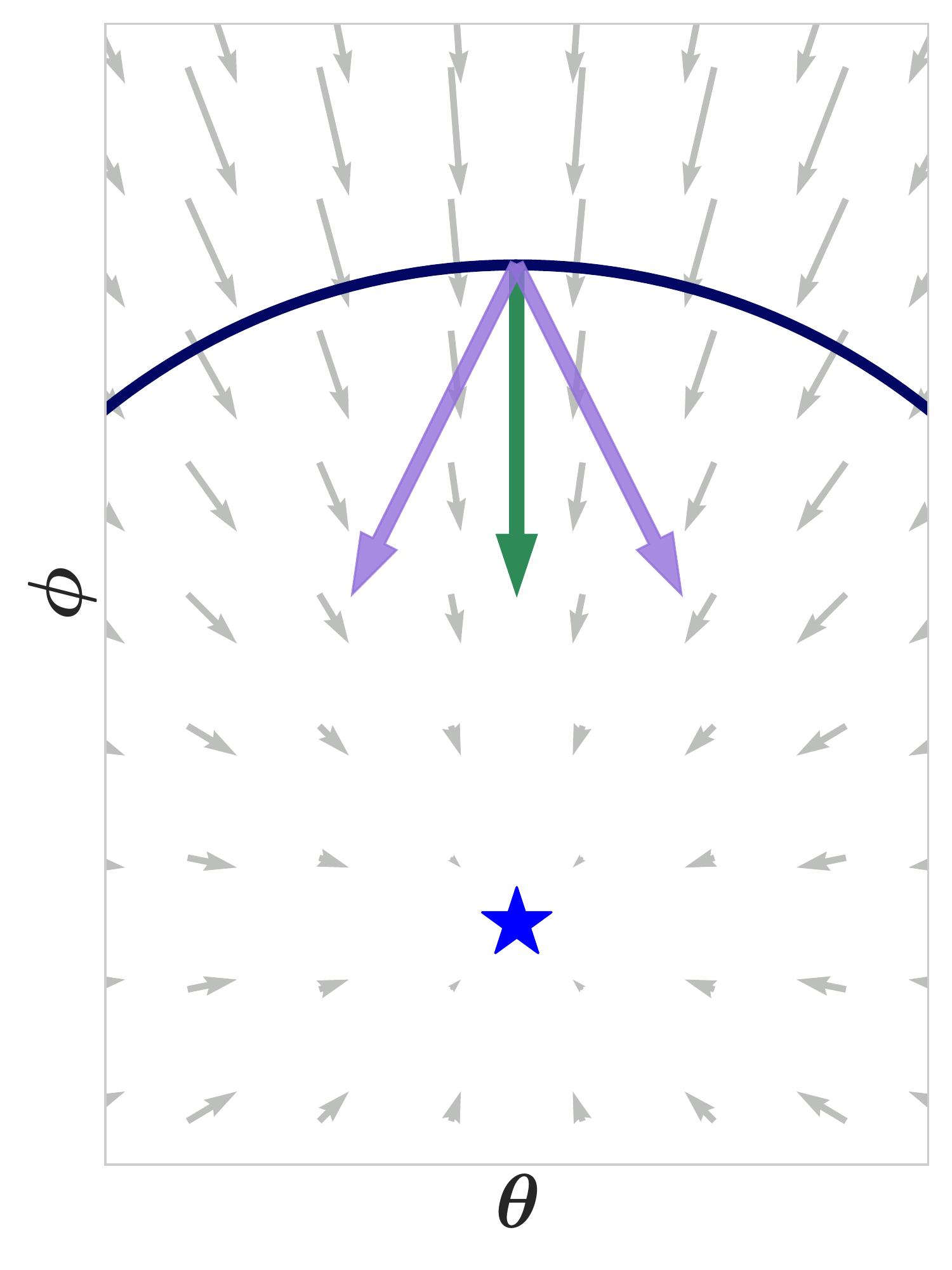}   \end{subfigure}
\hspace{2mm}
\begin{subfigure}{.46\linewidth}
\centering
\includegraphics[width= \linewidth, height= 1.4 \linewidth]{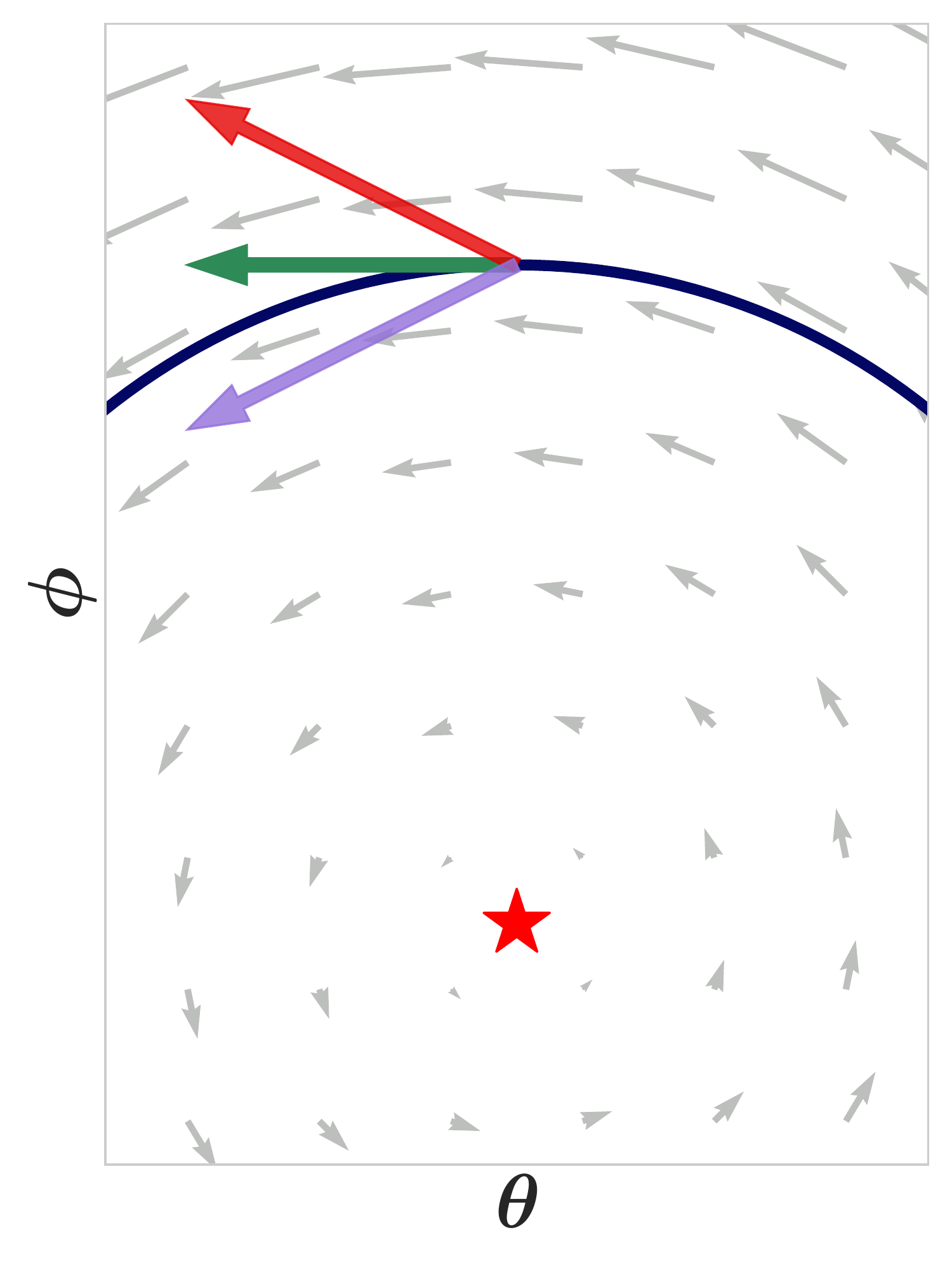}  \end{subfigure}
}
\end{figure}

As the~\eqref{eq:extragradient} converges for some examples for which gradient methods do not, it is reasonable to expect that so does its stochastic counterpart (at least to a neighborhood). However, the resulting noise in the gradient estimate may interact in a problematic way with the oscillations due to the \emph{adversarial component} of the game\footnote{\citet{gidel19-momentum} formalize the notion of ``adversarial component'' of a game, which yields a rotational dynamics in gradients methods (oscillations in parameters), as illustrated by the gradient field of Fig.~\ref{fig:min_vs_game} (right).}.  
We depict this phenomenon in Fig.~\ref{fig:min_vs_game}, where we show the direction of the noisy gradient on single objective minimization example and contrast it with a multi-objective one.

We present a simplistic example where the extragradient method \emph{converges linearly}~\citep[Corollary 1]{gidel2019variational} using the full gradient but \emph{diverges geometrically} when using stochastic estimates of it. Note that standard gradient methods, both batch and stochastic, diverge on this example.

In particular, we show that: 
\begin{enumerate*}[series = tobecont, itemjoin = \quad, label=(\roman*)]
\item if we use standard stochastic estimates of the gradients of $\LL$ with a simple finite sum formulation, then the iterates $\vomega_t := (\vtheta_t,\vphi_t)$ produced by the stochastic extragradient method (SEG) diverge geometrically, and on the other hand 
\item the full-batch extragradient method does converge to the Nash equilibrium $\vomega^*$ of this game~\citep[Thm. 12.1.11]{harker1990finite}. 
\end{enumerate*}

\begin{theorem}[Noise may induce divergence] 
\label{thm:diverge}
For any $\epsilon \geq 0$ There exists a zero-sum $\frac{\epsilon}{2}$-strongly monotone stochastic game such that if $\vomega_0 \neq \vomega^*$, then for any step-size $\eta>\epsilon$, the iterates $(\vomega_t)$ computed by the stochastic extragradient method diverge geometrically, i.e., there exists $\rho >0,$ such that $\E[\|\vomega_t - \vomega^*\|^2]> \|\vomega_0 - \vomega^*\|^2 (1+\rho)^t$.
\end{theorem}
\emph{Proof sketch.}
All detailed proofs can be found in~\S~\ref{sec:proof_of_theorems} of the appendix. 
We consider the following stochastic optimization problem (with $d=n$):
\begin{equation}\label{eq:problem_eg_fails}
   \frac{1}{n}  \sum_{i=1}^n \frac{\epsilon}{2} \theta_i^2 +\bm{\theta}^\top \bm{A}_{i} \bm{\varphi} - \frac{\epsilon}{2}\varphi_i^2  
  \quad
  \text{where} \quad [\mA_{i}]_{kl} = 1 \;  \text{if} \; k=l=i \; \text{and} \; 0 \; \text{otherwise}.
\end{equation}

Note that this problem is a simple dot product between $\vtheta$ and $\vphi$ with an $(\epsilon/n)$-$\ell_2$ norm penalization, thus we can compute the batch gradient and notice that the Nash equilibrium of this problem is $(\vtheta^*,\vphi^*) = (\bm{0},\bm{0})$. However, as we shall see, this simple problem \emph{breaks} with standard stochastic optimization methods.

Sampling a mini-batch without replacement $I \subset \{1,\ldots,n\}$, we denote $\mA_{I}:= \sum_{i \in I} \mA_{i}$.
The extragradient update rule can be written as:
\begin{equation}
\left\{
\begin{aligned}
  &\vtheta_{t+1} = (1-\eta \mA_{I} \epsilon) \vtheta_t - \eta \mA_{I} ( (1-\eta \mA_{J} \epsilon)\vphi_t + \eta \mA_{J} \vtheta_t)  \\
  &\vphi_{t+1} = (1-\eta \mA_{I} \epsilon)\vphi_t +\eta \mA_{I} ((1-\eta\mA_{J}\epsilon)\vtheta_t - \eta \mA_{J}\vphi_t) \,,
\end{aligned}
\right.
\end{equation}
where $I$ and $J$ are the mini-batches sampled for the update and the extrapolation step, respectively.
Let us write $N_t := \|\vtheta_t\|^2 + \|\vphi_t\|^2$. 
Noticing that $[\mA_{I}\vtheta]_i = [\vtheta]_i$ if $i\in I$ and $0$ otherwise, we have,
\begin{align}
 \mathbb E[N_{t+1}] 
 &= \left(1 -  \tfrac{|I|}{n}(2\eta \epsilon -  \eta^2(1+\epsilon^2)) - \tfrac{|I|^2}{n^2}(  2\eta^2 - \eta^4(1+\epsilon^2))\right) \mathbb E[N_t] \,.
\\\end{align}
Consequently, if the mini-batch size is smaller than half of the dataset size, i.e. $2|I|\leq n$, we have that $\forall \eta >\epsilon \,, \, \exists \rho > 0\,, \; s.t.\,,\; \E[N_{t}] > N_0 (1+\rho)^t$. For the theorem statement, we set $n=2$ and $|I|=1$.
\endproof

This result may seem contradictory with the standard result on SEG~\citep{juditsky2011solving} saying that the average of the iterates computed by SEG does converge to the Nash equilibrium of the game. However, an important assumption made by~\citeauthor{juditsky2011solving} is that the iterates are projected onto a compact set and that estimator of the gradient has finite variance. These assumptions break in this example since the variance of the estimator is proportional to the norm of the (unbounded) parameters.
Note that constraining the optimization problem~\eqref{eq:problem_eg_fails} to bounded domains $\Theta$ and $\Phi$,
 would make the finite variance assumption from~\citet{juditsky2011solving} holds. Consequently, the averaged iterate $\bar \vomega_t := \frac{1}{t} \sum_{s=0}^{t-1} \vomega_s$ would converge to $\vomega^*$. In \S~\ref{sub:why_is_convergence_of_last_iterate_preferable_}, we explain why in a \emph{non-convex setting}, the convergence of the \emph{last iterate} is preferable.

\section{Reducing Noise in Games with Variance Reduced Extragradient}
\label{sec:reducing_noise_with_vr_methods}
One way to reduce the noise in the estimation of the gradient is to use mini-batches of samples instead of one sample. However, mini-batch stochastic extragradient fails to converge on~\eqref{eq:problem_eg_fails} if the mini-batch size is smaller than half of the dataset size (see \S~\ref{sub:proof_of_theorem_thm:diverges}). In order to get an estimator of the gradient with a vanishing variance, the optimization literature proposed to take advantage of the finite-sum formulation that often appears in machine learning~\citep[and references therein]{schmidt2017minimizing}.

\subsection{Variance Reduced Gradient Methods} 
\label{sub:variance_methods}
Let us assume that the objective in~\eqref{eq:two_player_games} can be decomposed as a finite sum such that\footnote{The ``noise dataset'' in a GAN is not finite though; see \S~\ref{sub:practical_aspect} for details on how to cope with this in practice.} \begin{equation} \label{eq:two_player_games_finite_theta}
  \LL^{G}(\vomega) = \frac{1}{n}\sum_{i=1}^n \LL^{G}_i(\vomega)  \quad \text{and} \quad 
  \LL^{D}(\vomega) = \frac{1}{n}\sum_{i=1}^n \LL^{D}_i(\vomega)
  \quad \text{where} \quad 
  \vomega := (\vtheta,\vphi) \,.
\end{equation} 

\citet{johnson2013accelerating} propose the ``stochastic variance reduced gradient'' (SVRG) as an \emph{unbiased} estimator of the gradient with a smaller variance than the vanilla mini-batch estimate. The idea is to occasionally take a snapshot $\vomega^\mathcal{S}$ of the current model's parameters, and store the full batch gradient $\vmu^{\mathcal{S}}$ at this point. Computing the full batch gradient $\vmu^{\mathcal{S}}$ at $\vomega^\mathcal{S}$ is an expensive operation but not prohibitive if done infrequently (for instance once every dataset pass).

Assuming that we have stored $\vomega^{\mathcal{S}}$ and $\vmu^{\mathcal{S}} := (\vmu_\vtheta^{\mathcal{S}},\vmu_\vphi^{\mathcal{S}})$, the \emph{SVRG estimates} of the gradients are:
\begin{equation}  \vd^{G}_i\left(\vomega\right) := \tfrac{\nabla \LL^{G}_i(\vomega) - \nabla \LL^{G}_i\left(\vomega^\mathcal{S}\right)}{n \pi_i}  + \vmu_\vtheta^{\mathcal{S}} \;,
  \quad
  \label{eq:theta_svrg_dir}
  \vd^{D}_i\left(\vomega\right) := \tfrac{\nabla \LL^{D}_i(\vomega) - \nabla \LL^{D}_i\left(\vomega^\mathcal{S}\right)}{n \pi_i}+ \vmu_\vphi^{\mathcal{S}}.
\end{equation}
These estimates are unbiased: $\E[\vd^{G}_i(\vomega)] = \frac{1}{n}\sum_{i=1}^n \nabla \LL^{G}_i(\vomega) = \nabla \LL^{G}(\vomega)$, where the expectation is taken over $i$, picked with probability $\pi_i$.
The non-uniform sampling probabilities $\pi_i$ are used to bias the sampling according to the Lipschitz constant of the stochastic gradient in order to sample more often gradients that change quickly.
This strategy has been first introduced for variance reduced methods by~\citet{xiao2014SVRG_NUS} for SVRG and has been discussed for saddle point optimization by~\citet{palaniappan2016stochastic}.

Originally, SVRG was introduced as an epoch based algorithm with a \emph{fixed epoch size}: in Alg.~\ref{alg:svre}, one epoch is an inner loop of size $N$ (Line~\ref{l:epoch}). However, \citet{hofmann2015variance} proposed instead to \emph{sample} the size of each epoch from a geometric distribution, enabling them to analyze SVRG the same way as SAGA under a unified framework called $q$-memorization algorithm. We generalize their framework to handle the extrapolation step~\eqref{eq:extragradient} and provide a convergence proof for such $q$-memorization algorithms for games in \S~\ref{sub:proof_of_theorem_3}.

One advantage of~\citet{hofmann2015variance}'s framework is also that the sampling of the epoch size does not depend on the condition number of the problem, whereas the original proof for SVRG had to consider an epoch size larger than the condition number (see~\citet[Corollary 16]{leblond2018improved} for a detailed discussion on the convergence rate for SVRG). Thus, this new version of SVRG with a random epoch size becomes \emph{adaptive to the local strong convexity} since none of its hyper-parameters depend on the strong convexity constant.

However, because of some new technical aspects when working with monotone operators,~\citet{palaniappan2016stochastic}'s proofs (both for SAGA and SVRG) require a step-size (and epoch length for SVRG) that depends on the strong monotonicity constant making these algorithms not adaptive to local strong monotonicity. This motivates the proposed SVRE algorithm, which may be adaptive to local strong monotonicity, and is thus more appropriate for non-convex optimization.

\subsection{SVRE: Stochastic Variance Reduced Extragradient}
\label{sub:extra_svrg}
We describe our proposed algorithm called stochastic variance reduced extragradient (SVRE) in Alg.~\ref{alg:svre}. In an analogous manner to how~\citet{palaniappan2016stochastic} combined SVRG with the gradient method, SVRE combines SVRG estimates of the gradient~\eqref{eq:theta_svrg_dir} with the \emph{extragradient method}~\eqref{eq:extragradient}. 

With SVRE we are able to improve the convergence rates for variance reduction for a large class of stochastic games (see Table~\ref{tab:rates} and Thm.~\ref{thm:extra-svrg}), and we show in \S~\ref{sub:motivating_example} that it is the only method which empirically converges on the simple example of \S~\ref{sub:stochasticity_breaks_extragradient_}.

We now describe the theoretical setup for the convergence result. A standard assumption in convex optimization is the assumption of strong convexity of the function.
However, in a game, the operator, 
\begin{equation}\label{eq:game_operator}
 \vv:\vomega \mapsto \left[\nabla_\vtheta \LL^{G}(\vomega)\, , \; \nabla_\vphi \LL^{D}(\vomega)\right]^\top  \,,
\end{equation}
associated with the updates is no longer the gradient of a single function. 
To make an analogous assumption for games the optimization literature considers the notion of \emph{strong monotonicity}.
\begin{definition}\label{def:strong_monotonicity}
An operator $F: \vomega \mapsto (F_\vtheta(\vomega),F_\vphi(\vomega)) \in \R^{d+p}$ is said to be $(\mu_\vtheta,\mu_\vphi)$-strongly monotone if for all $\vomega,  \vomega' \in \R^{p+d}$ we have
\begin{equation} \label{eq:weighted_norm}
   \Omega((\vtheta,\vphi),(\vtheta',\vphi')) :=
   \mu_\vtheta \|\vtheta -\vtheta'\|^2 + \mu_\vphi \|\vphi -\vphi'\|^2
   \leq 
   (F(\vomega)- F(\vomega'))^\top (\vomega - \vomega') \,,  \notag
\end{equation}
where we write $\,\vomega := (\vtheta,\vphi) \in \R^{d+p}$. A \emph{monotone operator} is a $(0,0)$-strongly monotone operator.
\end{definition}
This definition is a generalization of strong convexity for operators: if $f$ is $\mu$-strongly convex, then $\nabla f$ is a $\mu$-monotone operator. Another assumption is the $\gamma$ regularity assumption,
\begin{definition}\label{def:gamma_regularity}
An operator $F: \vomega \mapsto (F_\vtheta(\vomega),F_\vphi(\vomega)) \in \R^{d+p}$ is said to be $(\gamma_\vtheta,\gamma_\phi)$-regular if,
\begin{equation} \label{eq:gamma_regularity2}
   \gamma_\vtheta^2\|\vtheta -\vtheta'\|^2 + \gamma_\vphi^2 \|\vphi -\vphi'\|^2
   \leq 
   \|F(\vomega) -F(\vomega')\|^2\,,\quad  \forall\, \vomega,  \vomega' \in  \R^{p+d}\,.
\end{equation}
\end{definition}
 Note that an \emph{operator} is always $(0,0)$-regular.
This assumption originally introduced by~\citet{tsengLinearConvergenceIterative1995} has been recently used~\citep{azizian2019tight} to improve the convergence rate of extragradient. For instance for a full rank bilinear matrix problem $\gamma$ is its smallest singular value. More generally, in the case $\gamma_\vtheta = \gamma_\vphi$, the regularity constant is a lower bound on the minimal singular value of the Jacobian of $F$~\citep{azizian2019tight}.

One of our main assumptions is the cocoercivity assumption, which implies the Lipchitzness of the operator in the unconstrained case. We use the cocoercivity constant because it provides a tighter bound for general strongly monotone and Lipschitz games (see discussion following Theorem~\ref{thm:extra-svrg}).
\begin{definition}
 An operator $F: \vomega \mapsto (F_\vtheta(\vomega),F_\vphi(\vomega))\in \R^{d+p}$ is said to be $(\ell_\vtheta,\ell_\vphi)$-cocoercive, if for all $\vomega,  \vomega' \in \Omega$ we have
\begin{equation} \label{eq:cocoercive}
 \|F(\vomega) -F(\vomega')\|^2 \leq 
  \ell_\vtheta(F_\vtheta(\vomega)- F_\vtheta(\vomega'))^\top (\vtheta - \vtheta') + \ell_\vphi(F_\vphi(\vomega)- F_\vphi(\vomega'))^\top (\vphi - \vphi')  \,.
\end{equation} 
\end{definition}
Note that for a $L$-Lipschitz and $\mu$-strongly monotone operator, we have $\ell \in [L, L^2/\mu]$~\citep{facchineiFiniteDimensionalVariationalInequalities2003}. For instance, when $F$ is the gradient of a convex function, we have $\ell = L$. More generally, when $F(\vomega)= (\nabla f(\vtheta) + M\vphi, \nabla g(\vphi) - M^\top \vtheta)$, where $f$ and $g$ are $\mu$-strongly convex and $L$ smooth we have that $\gamma = \sigma_{\min}(M)$ and $\|M\|^2 = O(\mu L)$ is a sufficient condition for $\ell = O(L)$ (see \S\ref{sec:definitions_and_lemmas}). 
Under this assumption on each cost function of the game operator, we can define a cocoercivity constant adapted to the non-uniform sampling scheme of our stochastic algorithm:
\begin{equation}\label{eq:sampling}
  \bar \ell (\pi)^2 := \frac{1}{n}\sum_{i=1}^n \frac{1}{n \pi_i}\ell_i^2.
\end{equation}
The standard \emph{uniform sampling scheme} corresponds to $\pi_i:= \frac{1}{n}$ and the optimal \emph{non-uniform} sampling scheme corresponds to $\tilde \pi_i:= \frac{\ell_i}{\sum_{i=1}^n \ell_i}$. By Jensen's inequality, we have: $\bar \ell(\tilde \pi) \leq \bar \ell(\pi) \leq \max_i \ell_i$.

For our main result, we make strong convexity, cocoercivity and regularity assumptions.

\begin{assumption}\label{assump:SVRG}For $1\leq i\leq n$, the gradients $\nabla_\vtheta \LL_i^{G}$ and $\nabla_\vphi \LL_i^{D}$ are respectively $\ell_i^{\vtheta}$ and $\ell_i^{\vphi}$-cocoercive and $(\gamma^{\vtheta}_i,\gamma_i^{\vphi})$-regular. The operator~\eqref{eq:game_operator} is $(\mu_\vtheta,\mu_\vphi)$-strongly monotone.
\end{assumption}

We now present our convergence result for SVRE with non-uniform sampling (to make our constants comparable to those of~\citet{palaniappan2016stochastic}), but note that we have used uniform sampling in all our experiments (for simplicity).
\begin{theorem}\label{thm:extra-svrg} Under Assumption~\ref{assump:SVRG}, after $t$ iterations, the iterate $\vomega_t:=(\vtheta_t,\vphi_t)$ computed by SVRE (Alg.~\ref{alg:svre}) with step-size $\eta_\vtheta\leq ({40 \bar \ell_\vtheta})^{-1}$ and $\eta_\vphi\leq ({40 \bar \ell_\vphi})^{-1}$ and sampling scheme $(\tilde \pi_\vtheta,\tilde \pi_\vphi)$ verifies:
\begin{equation} \notag
  \E[\|\vomega_t - \vomega^*\|_2^2] \leq  \left(1- \min \left\{\frac{\eta_\vtheta\mu_\vtheta}{4} + \frac{11\eta_\vtheta^2\bar\gamma_\vtheta^2}{25},\frac{\eta_\vphi\mu_\vphi}{4} + \frac{11\eta_\vphi^2\bar\gamma_\vphi^2}{25}, \frac{2}{5n}\right\} \right)^t
  \E[\|\vomega_0 - \vomega^*\|_2^2]  \,,
\end{equation}
where $\bar \ell_\vtheta(\pi_\vtheta)$ and $\bar \ell_\vphi(\pi_\vphi)$ are defined in~\eqref{eq:sampling}. Particularly, for $\eta_\vtheta =\frac{1}{40 \bar \ell_\vtheta}$ and $\eta_\vphi =\frac{1}{40 \bar \ell_\vphi}$ we get  
\begin{equation} \notag
  \E[\|\vomega_t - \vomega^*\|_2^2] \leq  \left(1- \min \left\{\frac{1}{80}\Big(\frac{\mu_\vtheta}{2  \bar\ell_\vtheta} + \frac{\bar \gamma_\vtheta^2}{25\bar \ell_\vtheta^2}\Big),\frac{1}{80}\Big(\frac{\mu_\vphi}{2\bar \ell_\vphi} + \frac{\bar\gamma_\vphi^2}{25\bar \ell_\vphi^2}\Big), \frac{2}{5n}\right\} \right)^t
  \E[\|\vomega_0 - \vomega^*\|_2^2]  \,.
\end{equation}
\end{theorem}
We prove this theorem in \S~\ref{sub:proof_of_theorem_3}. We can notice that the respective \emph{condition numbers} of $\LL^{G}$ and $\LL^{D}$ defined as $\kappa_\vtheta:=\tfrac{\mu_\vtheta}{\bar \ell_{\vtheta}}+\tfrac{\bar \gamma_\vtheta^2}{\bar \ell^2_{\vtheta}}$ and $\kappa_\vphi := \tfrac{\mu_{\vphi}}{\bar \ell_{\vphi}} + \tfrac{\bar \gamma_\vphi^2}{\bar \ell_{\vphi}^2}$ appear in our convergence rate. The cocoercivity constant $\ell$ belongs to $[L,L^2/\mu]$, thus our rate may be significantly faster\footnote{Particularly, when $F$ is the gradient of a convex function (or close to it) we have $\ell\approx L$ and thus our rate recovers the standard $\ln(1/\epsilon)L/\mu$, improving over the accelerated algorithm of~\citet{palaniappan2016stochastic}. More generally, under the assumptions of Proposition~\ref{prop:coco}, we also recover $\ln(1/\epsilon)L/\mu$.}  than the convergence rate of the (non-accelerated) algorithm of~\citet{palaniappan2016stochastic} that depends on the product $\tfrac{\mu_\vtheta}{\bar L_{\vtheta}}\tfrac{\mu_\vphi}{\bar L_{\vphi}}$. 
They avoid a dependence on the maximum of the condition numbers squared, $\max\{\kappa_\vphi^2,\kappa_\vtheta^2\}$, by using the weighted Euclidean norm~$\Omega(\vtheta,\vphi)$ defined in~\eqref{eq:weighted_norm} and rescaling the functions $\LL^{G}$ and $\LL^{D}$ with their strong-monotonicity constant. However, this rescaling trick suffers from two issues:
\begin{enumerate*}[series = tobecont, itemjoin = \quad, label=(\roman*)]
\item we do not know in practice a good estimate of the strong monotonicity constant, which was \emph{not} the case in~\citet{palaniappan2016stochastic}'s application; and 
\item the algorithm does not adapt to local strong-monotonicity. This property is important in non-convex optimization since we want the algorithm to exploit the (potential) local stability properties of a stationary point.
\end{enumerate*} 

\subsection{Motivating example} 
\label{sub:motivating_example}
  
The example~\eqref{eq:problem_eg_fails} for $\epsilon=0$ seems to be challenging in the stochastic setting since all the standard methods and even the stochastic extragradient method fails to find its Nash equilibrium (note that this example is \emph{not} strongly monotone). We set $n = d = 100$, and draw $[\mA_{i}]_{kl} =\delta_{kli} \text{ and } [\vb_i]_k, [\vc_i]_k \sim \mathcal{N}(0,1/d) \,,\; 1\leq k,l\leq d$, where $\delta_{kli} = 1$ if $k=l=i$ and $0$ otherwise. Our optimization problem is: 
\begin{equation}\label{eq:bilin_exp} 
  \min_{\vtheta \in \R^d} \max_{\vphi \in \R^d}  \frac{1}{n}\sum_{i=1}^n (\vtheta^\top \vb_i+\vtheta^\top \mA_{i} \vphi + \vc_i^\top \vphi).
\end{equation}
We compare variants of the following algorithms (with uniform sampling and average our results over 5 different seeds): 
\begin{enumerate*}[series = tobecont, itemjoin = \quad, label=(\roman*)]
\item AltSGD: the standard method to train GANs--stochastic gradient with alternating updates of each player. \item SVRE: Alg.~\ref{alg:svre}.
\end{enumerate*} 
The AVG prefix correspond to the \emph{uniform average} of the iterates, $\bar \vomega := \frac{1}{t}\sum_{s=0}^{t-1} \vomega_s$. We observe in Fig.~\ref{fig:blinear} that AVG-SVRE converges sublinearly (whereas AVG-AltSGD fails to converge). 

This motivates a new variant of SVRE based on the idea that even if the averaged iterate converges, we do not compute the gradient at that point and thus we do not benefit from the fact that this iterate is closer to the optimums~(see \S~\ref{sub:why_is_convergence_of_last_iterate_preferable_}). Thus the idea is to occasionally restart the algorithm, i.e., consider the averaged iterate as the new starting point of our algorithm and compute the gradient at that point. Restart goes well with SVRE as we already occasionally stop the inner loop to recompute $\vmu^\mathcal{S}$, at which point we decide (with a probability $p$ to be fixed) whether or not to restart the algorithm by taking the snapshot at point $\bar \vomega_t$ instead of $\vomega_t$.
This variant of SVRE is described in Alg.~\ref{alg:svre_restart} in \S~\ref{sec:restarted_svre} and the variant combining VRAd in \S~\ref{sub:practical_aspect}.

In Fig.~\ref{fig:blinear} we observe that the only method that converges is SVRE and its variants. We do not provide convergence guarantees for Alg.~\ref{alg:svre_restart} and leave its analysis for future work. However, it is interesting that, to our knowledge, this algorithm is the only stochastic algorithm (excluding batch extragradient as it is not stochastic) that converge for~\eqref{eq:problem_eg_fails}. Note that we tried all the algorithms presented in Fig.~3 from~\citet{gidel2019variational} on this \emph{unconstrained} problem and that all of them diverge.

\section{GAN Experiments}\label{sec:experiments}
In this section, we investigate the empirical performance of SVRE for \emph{GAN training}. Note, however, that our theoretical analysis does not hold for games with non-convex objectives such as GANs.

\textbf{Datasets. \,\,}\label{par-datasets}
We used the following datasets:
\begin{enumerate*}[series = tobecont, itemjoin = \quad, label=(\roman*)]
\item \textbf{MNIST}~\citep{mnist},   
\item \textbf{CIFAR-10} \citep[\S3]{cifar10}, 
\item \textbf{SVHN}~\citep{svhn}, and
\item \textbf{ImageNet}  ILSVRC 2012~\citep{imagenet},
\end{enumerate*} using 
$28\!\times\!28$, 
$3 \!\times\! 32 \!\times\! 32$, 
$3 \!\times\! 32 \!\times\! 32$, and 
$3 \!\times\! 64 \!\times\! 64$ resolution, respectively.

\textbf{Metrics. \,\,}\label{par:metrics}
We used the \textbf{Inception score} (IS,~\citealp{salimans2016improved}) and the \textbf{Fr\'echet Inception distance} (FID,~\citealp{heusel_gans_2017}) as performance metrics for image synthesis. To gain insights if SVRE indeed reduces the variance of the gradient estimates, we used the \textbf{second moment estimate--SME} (uncentered variance), computed with an exponentially moving average. See \S~\ref{sec:app-metrics} for details.

\textbf{DNN architectures. \,\,}\label{par:arch-brief}  
For experiments on \textbf{MNIST}, we used the DCGAN architectures~\citep{radford2016unsupervised}, described in \S~\ref{app:mnist_arch}.
For real-world datasets, we used two architectures (see \S~\ref{app:arch} for details and \S~\ref{sec:arch_motivation} for motivation):
\begin{enumerate*}[series = tobecont, itemjoin = \quad, label=(\roman*)]
\item SAGAN~\citep{sagan}, and
\item ResNet, replicating the setup of~\citet{miyato2018spectral},
\end{enumerate*}
described in detail in \S~\ref{sec:shallow_sagan} and ~\ref{sec:deeper_resnet_arch}, respectively.
For clarity, we refer the former as \textit{shallow}, and the latter as \textit{deep} architectures.

\paragraph{Optimization methods.}\label{par:optim}     
We conduct experiments using the following optimization methods for GANs:
\begin{enumerate*}[series = tobecont, itemjoin = \quad, label=(\roman*)]
\item \textbf{BatchE:} full--batch extragradient,
\item \textbf{SG:} stochastic gradient (alternating GAN), and
\item \textbf{SE:} stochastic extragradient, and
\item \textbf{SVRE:} stochastic variance reduced extragradient.
\end{enumerate*}
These can be combined with adaptive learning rate methods such as \textit{Adam} or with parameter averaging, hereafter denoted as  \textbf{--A} and \textbf{AVG--}, respectively.
In \S~\ref{sub:practical_aspect}, we present a variant of Adam adapted to variance reduced algorithms, that is referred to as \textbf{--VRAd}. 
When using the SE--A baseline and \textit{deep} architectures, the convergence rapidly fails at some point of training (cf.~\S~\ref{sec:results_deep_arch}). 
This motivates experiments where we start from a stored checkpoint taken \emph{before} the baseline diverged, and \emph{continue training with SVRE}. We denote these experiments with \textbf{WS--SVRE} (warm-start SVRE).

\begin{figure}[tb]
    \begin{subfigure}[t]{0.329\linewidth}
        \centering
        \includegraphics[width=\linewidth]{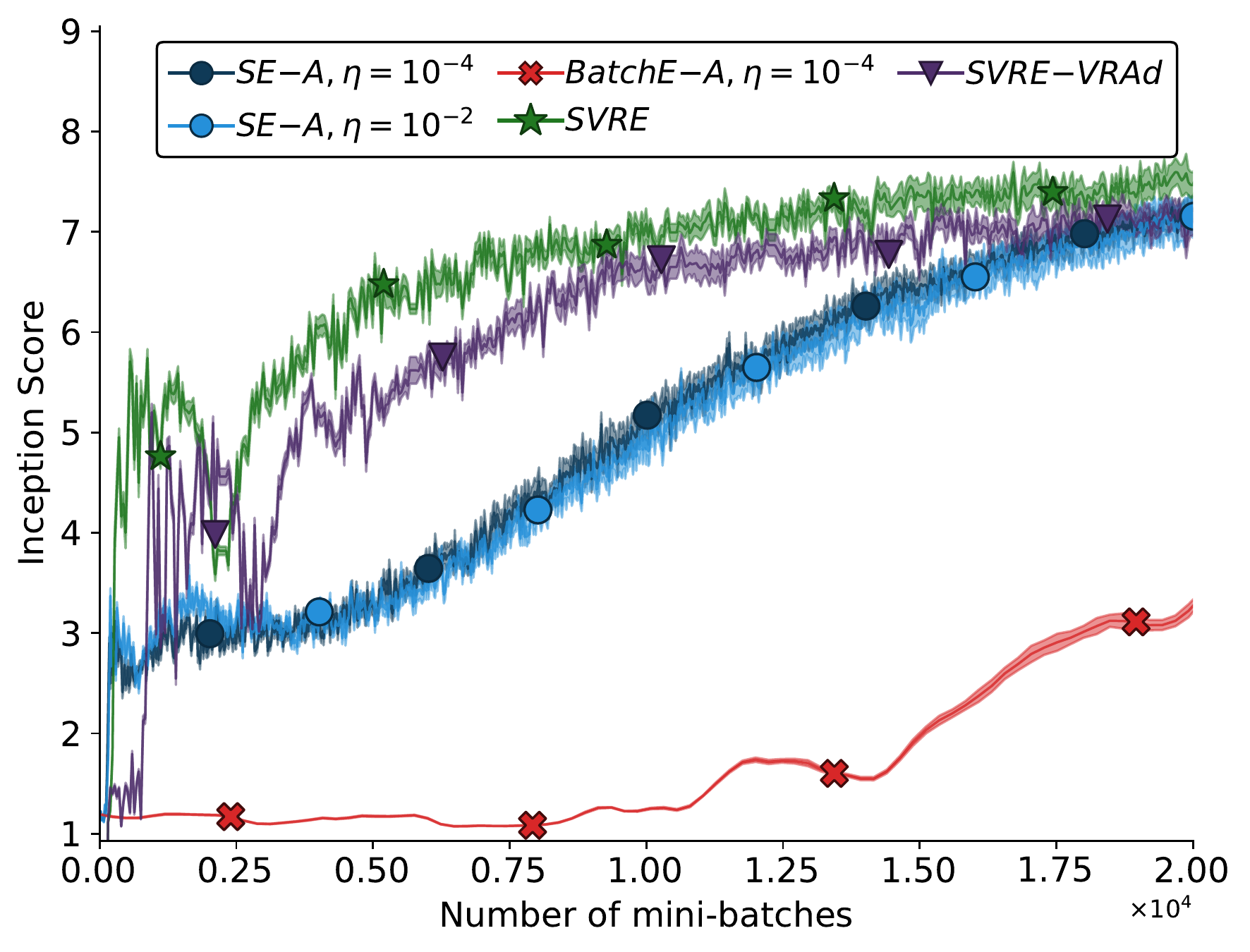}
        \caption{IS (higher is better), \textbf{MNIST}}\label{subfig-mnist_is}
    \end{subfigure}
   \begin{subfigure}[t]{0.329\linewidth}
        \centering
        \includegraphics[width=\linewidth]{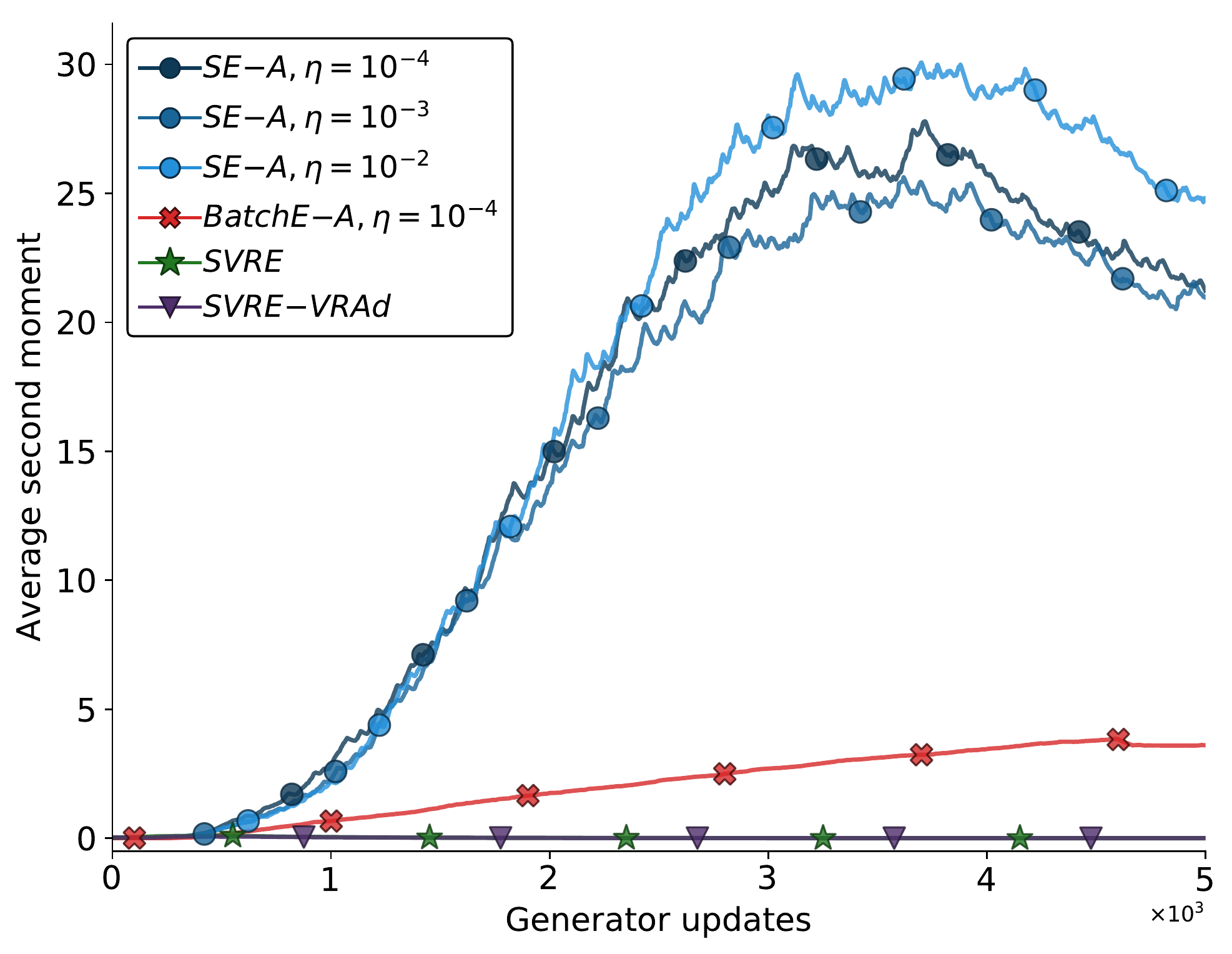}
        \caption{Generator--SME, \textbf{MNIST}}\label{subfig-var_g}
    \end{subfigure}
    \begin{subfigure}[t]{0.329\linewidth}
        \centering
        \includegraphics[width=\linewidth]{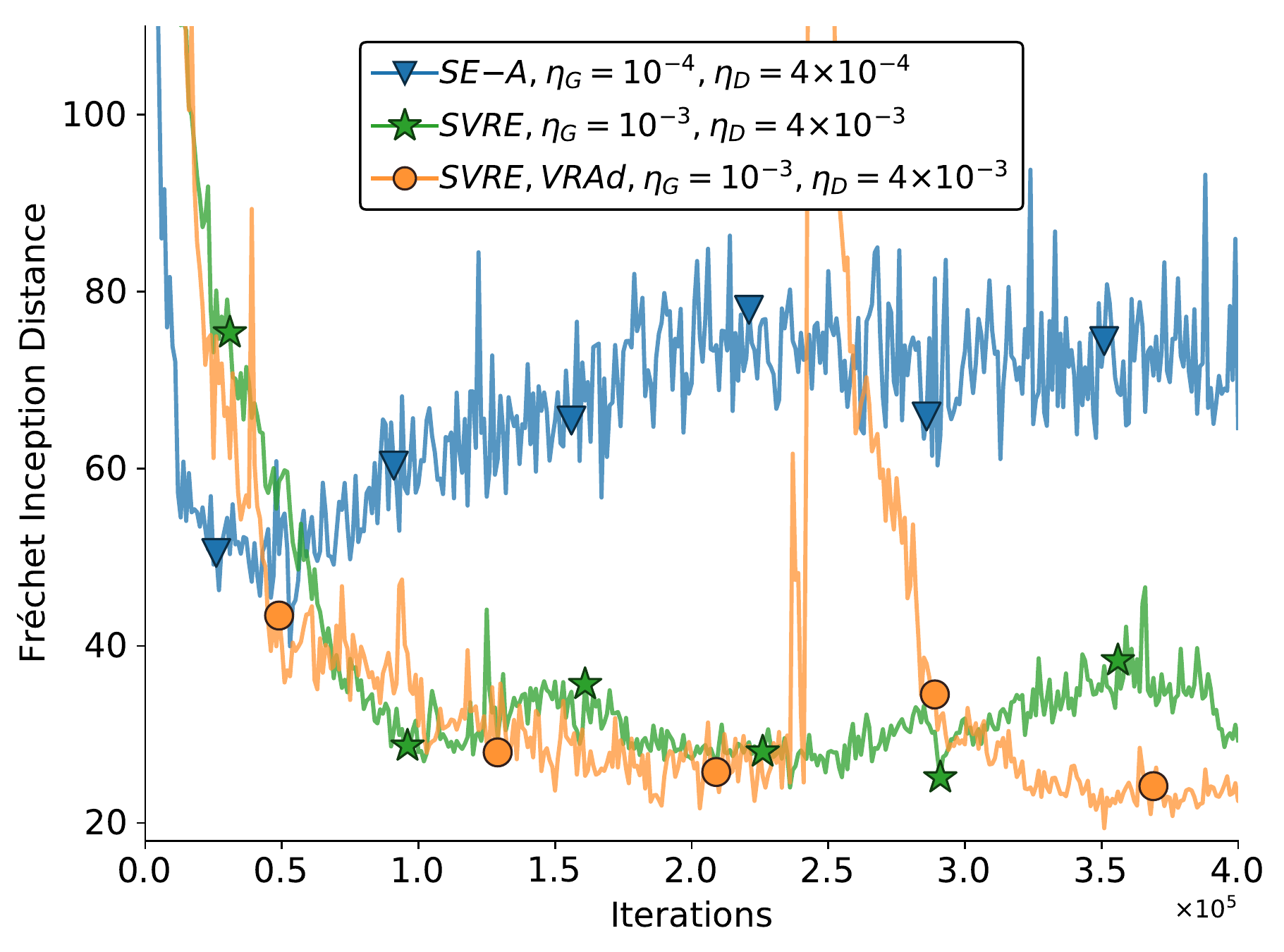}
        \caption{FID (lower is better), \textbf{SVHN}}\label{subfig-fid_svhn}
    \end{subfigure}
    \caption{\textbf{Figures~\subref{subfig-mnist_is} \& \subref{subfig-var_g}.} Stochastic, full-batch and variance reduced extragradient optimization on \textbf{MNIST}.
  We used $\eta=10^{-2}$ for SVRE.      \textit{SE--A} with $\eta=10^{-3}$ achieves similar IS performances as $\eta=10^{-2}$ and $\eta=10^{-4}$, omitted from Fig.~\subref{subfig-mnist_is} for clarity.
   \textbf{Figure~\subref{subfig-fid_svhn}.}
   FID on \textbf{SVHN}, using \textit{shallow} architectures.
See \S~\ref{par:optim} and \S~\ref{sec:impl-details} for naming of methods and details on the implementation, respectively. 
 }
    \label{fig:extra_mnist}
\end{figure}

\begin{figure} 
	\begin{minipage}
		{\linewidth}
		{}\vspace{-6mm}
		\begin{minipage}{.5 \linewidth}
			\begin{figure}[H]
				\includegraphics[width=.95\linewidth]{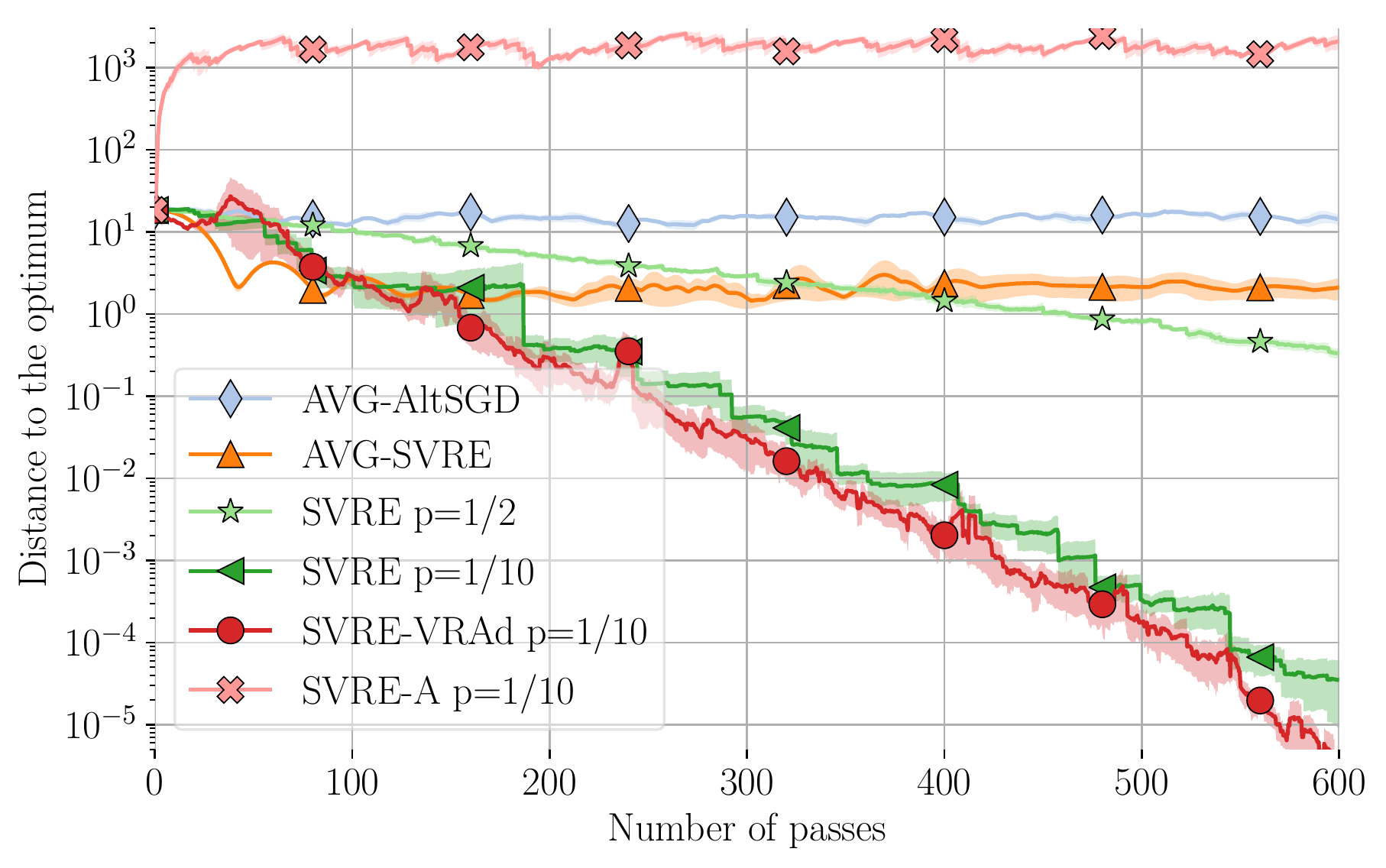}
				\caption{Distance to the optimum of~\eqref{eq:bilin_exp}, see \S~\ref{sub:motivating_example} for the experimental setup. 
																		}\label{fig:blinear}
			\end{figure}
		\end{minipage}
				\hfill
		\begin{minipage}{.48\linewidth}
			\begin{table}[H]
																				\begin{tabular}{@{}r@{\hskip .5em}c@{\hskip .4em}c@{\hskip .4em}c@{\hskip .4em}c@{}}\toprule 
					& SG-A  & SE-A & SVRE  & WS-SVRE  \\ \midrule    
					\\					CIFAR-10  &  $21.70$ 	&  $18.65$  & $23.56$ 				& $\mathbf{16.77}$ \\
					SVHN        &  $5.66$ 	& $5.14$  	& $\mathbf{4.81}$   & $\mathbf{4.88}$  \\
					\bottomrule
				\end{tabular}
				\caption{Best obtained FID scores for the different optimization methods using the \textit{deep} architectures (see Table~\ref{tab:resnet_arch}, \S~\ref{sec:deeper_resnet_arch}).
										WS--SVRE starts from the best obtained scores of SE--A.
					See \S~\ref{sec:impl-details} and \S~\ref{sec:additional-experiments} for implementation details and additional results, respectively. 
									}\label{tab:res_summary} 
			\end{table}
		\end{minipage}
	\end{minipage}
\end{figure}

\subsection{Results}\label{sec:mnist}

\textbf{Comparison on MNIST. \,\,} The \textbf{MNIST} common benchmark allowed for comparison with full-batch extragradient, as it is feasible to compute.
Fig.~\ref{fig:extra_mnist} depicts the IS metric while using either a stochastic, full-batch or variance reduced version of extragradient (see details of SVRE-GAN in \S~\ref{sub:extrasvrg_gan}). We always combine the stochastic baseline (SE) with \textit{Adam}, as  proposed by~\citet{gidel2019variational}.  In terms of number of parameter updates, SVRE performs similarly to BatchE--A (see Fig.~\ref{subfig-mnist_is_param_updates}, \S~\ref{sec:additional-experiments}).
Note that the latter requires significantly more computation: Fig.~\ref{subfig-mnist_is} depicts the IS metric using the number of mini-batch computations as x-axis (a surrogate for the wall-clock time, see below).
We observe that, as SE--A has slower per-iteration convergence rate, SVRE converges faster on this dataset. At the end of training, all methods reach similar performances (IS is above $8.5$, see Table~\ref{tab:res_summary_shallow}, \S~\ref{sec:additional-experiments}).

\textbf{Computational cost. \,\,}
The relative cost of one pass over the dataset for SVRE versus vanilla SGD is a factor of $5$: the full batch gradient is computed (on average) after one pass over the dataset, giving a slowdown of $2$; the factor $5$ takes into account the extra stochastic gradient computations for the variance reduction, as well as the extrapolation step overhead. However, as SVRE provides less noisy gradient, it may converge faster per iteration, compensating the extra per-update cost. 
Note that many computations can be done in parallel. In Fig.~\ref{subfig-mnist_is}, the x-axis uses an implementation-independent surrogate for wall-clock time that counts the number of mini-batch gradient computations. Note that some training methods for GANs require multiple discriminator updates per generator update, and we observed that to stabilize our baseline when using the \textit{deep} architectures it was required to use $1{:}5$ update ratio of $G{:}D$ (cf. \S~\ref{sec:results_deep_arch}), whereas for SVRE we used ratio of $1{:}1$ (Tab.~\ref{tab:res_summary} lists the results). 
\textbf{Second moment estimate and Adam. \,\,}
Fig.~\ref{subfig-var_g} depicts the averaged second-moment estimate for parameters of the Generator, where we observe that SVRE effectively reduces it over the iterations. The reduction of these values may be the reason why Adam combined with SVRE performs poorly (as these values appear in the denominator, see \S~\ref{sub:practical_aspect}).
To our knowledge, SVRE is the first optimization method with a constant step size that has worked empirically for GANs on non-trivial datasets.

\textbf{Comparison on real-world datasets. \,\,}
In Fig.~\ref{subfig-fid_svhn}, we compare SVRE with the SE--A baseline on \textbf{SVHN}, using \textit{shallow} architectures. 
We observe that although SE--A in some experiments obtains better performances in the early iterations, SVRE allows for obtaining improved \emph{final} performances. 
Tab.~\ref{tab:res_summary} summarizes the results on \textbf{CIFAR-10} and \textbf{SVHN} with \textit{deep} architectures.
We observe that, with deeper architectures, SE--A is notably more unstable, as training collapsed in $100$\% of the experiments. To obtain satisfying results for SE--A, we used various techniques such as a schedule of the learning rate and different update ratios (see \S~\ref{sec:results_deep_arch}). On the other hand, SVRE \emph{did not collapse in any of the experiments} but took longer time to converge compared to SE--A.
Interestingly, although WS--SVRE starts from an iterate point after which the baseline diverges, it continues to improve the obtained FID score and does not diverge.
See \S~\ref{sec:additional-experiments} for additional experiments.

\section{Related work}
\label{sub:related_work}

Surprisingly, there exist only a few works on variance reduction methods for monotone operators, namely from \cite{palaniappan2016stochastic} and \cite{davis2016smart}. 
The latter requires a co-coercivity assumption on the operator and thus only convex optimization is considered. 
Our work provides a new way to use variance reduction for monotone operators, using the extragradient method~\citep{korpelevich1976extragradient}. 
Recently, \citet{iusem2017extragradient} proposed an extragradient method with variance reduction for an \emph{infinite sum} of operators. 
The authors use mini-batches of growing size in order to reduce the variance of their algorithm and to converge with a constant step-size. However, this approach is prohibitively expensive in our application. Moreover, \citeauthor{iusem2017extragradient} are not using the SAGA/SVRG style of updates exploiting the finite sum formulation, leading to sublinear convergence rate, while our method benefits from a linear convergence rate exploiting the finite sum assumption.

\citet{daskalakis2017training} proposed a method called Optimistic-Adam inspired by game theory. This method is closely related to extragradient, with slightly different update scheme. More recently, \citet{gidel2019variational} proposed to use extragradient to train GANs, introducing a method called ExtraAdam. This method outperformed Optimistic-Adam when trained on CIFAR-10. Our work is also an attempt to find principled ways to train GANs. Considering that the game aspect is better handled by the extragradient method, we focus on the optimization issues arising from the noise in the training procedure, a disregarded potential issue in GAN training.

In the context of deep learning, despite some very interesting theoretical results on non-convex minimization~\citep{reddi2016stochastic,allen2016variance}, the effectiveness of variance reduced methods is still an open question, and a recent technical report by~\citet{defazio2018ineffectiveness} provides negative empirical results on the variance reduction aspect. 
In addition, two recent large scale studies showed that increased batch size has:
\begin{enumerate*}[series = tobecont, itemjoin = \quad, label=(\roman*)]
\item only marginal impact on \emph{single objective training}~\citep{shallue2018measuring} and
\item a surprisingly large performance improvement on \emph{GAN training}~\citep{brock2018large}.
\end{enumerate*}
In our work, we are able to show positive results for variance reduction in a real-world deep learning setting. This unexpected difference seems to confirm the remarkable discrepancy, that remains poorly understood, between multi-objective optimization and standard minimization.

\section{Discussion}

Motivated by a simple bilinear game optimization problem where stochasticity provably breaks the convergence of previous stochastic methods, we proposed the novel SVRE algorithm that combines  SVRG with the extragradient method for optimizing games. On the theory side, SVRE improves upon the previous best results for strongly-convex games, whereas empirically, it is the only method that converges for our stochastic bilinear game counter-example. 

We empirically observed that SVRE for GAN training obtained convergence speed similar to Batch-Extragradient on MNIST, while the latter is computationally infeasible for large datasets.
For shallow architectures, SVRE matched or improved over baselines on all four datasets. 
Our experiments with deeper architectures show that SVRE is notably more stable with respect to hyperparameter choice. Moreover, while its stochastic counterpart diverged in all our experiments, SVRE did not. 
However, we observed that SVRE took more iterations to converge when using deeper architectures, though notably, we were using constant step-sizes, unlike the baselines which required Adam. As adaptive step-sizes often provide significant improvements, developing such an appropriate version for SVRE is a promising direction for future work. In the meantime, the stability of SVRE suggests a practical use case for GANs as warm-starting it just before the baseline diverges, and running it for further improvements, as demonstrated with the WS--SVRE method in our experiments.

\section*{Acknowledgements}
This research was partially supported by the Canada CIFAR AI Chair Program, the Canada Excellence Research Chair in “Data Science for Realtime Decision-making”, by the NSERC Discovery Grant RGPIN-2017-06936, by the Hasler Foundation through the
MEMUDE project, and by a Google Focused Research Award. Authors would like to thank Compute Canada for providing the GPUs used for this research. TC would like to thank Sebastian Stich and Martin Jaggi, and GG and TC would like to thank Hugo Berard for helpful discussions.

\bibliography{svre}

\begin{thebibliography}{45}
\providecommand{\natexlab}[1]{#1}
\providecommand{\url}[1]{\texttt{#1}}
\expandafter\ifx\csname urlstyle\endcsname\relax
  \providecommand{\doi}[1]{doi: #1}\else
  \providecommand{\doi}{doi: \begingroup \urlstyle{rm}\Url}\fi

\bibitem[Allen-Zhu and Hazan(2016)]{allen2016variance}
Z.~Allen-Zhu and E.~Hazan.
\newblock Variance reduction for faster non-convex optimization.
\newblock In \emph{ICML}, 2016.

\bibitem[Azizian et~al.(2019)Azizian, Mitliagkas, Lacoste-Julien, and
  Gidel]{azizian2019tight}
W.~Azizian, I.~Mitliagkas, S.~Lacoste-Julien, and G.~Gidel.
\newblock A tight and unified analysis of extragradient for a whole spectrum of
  differentiable games.
\newblock \emph{arXiv preprint arXiv:1906.05945}, 2019.

\bibitem[Bottou(2010)]{bottou2010SGD}
L.~Bottou.
\newblock Large-scale machine learning with stochastic gradient descent.
\newblock In \emph{COMPSTAT}, 2010.

\bibitem[Boyd and Vandenberghe(2004)]{boyd2004convex}
S.~Boyd and L.~Vandenberghe.
\newblock \emph{Convex optimization}.
\newblock Cambridge university press, 2004.

\bibitem[Brock et~al.(2019)Brock, Donahue, and Simonyan]{brock2018large}
A.~Brock, J.~Donahue, and K.~Simonyan.
\newblock Large scale {GAN} training for high fidelity natural image synthesis.
\newblock In \emph{ICLR}, 2019.

\bibitem[Daskalakis et~al.(2018)Daskalakis, Ilyas, Syrgkanis, and
  Zeng]{daskalakis2017training}
C.~Daskalakis, A.~Ilyas, V.~Syrgkanis, and H.~Zeng.
\newblock Training {GANs} with optimism.
\newblock In \emph{ICLR}, 2018.

\bibitem[Davis(2016)]{davis2016smart}
D.~Davis.
\newblock Smart: The stochastic monotone aggregated root-finding algorithm.
\newblock \emph{arXiv:1601.00698}, 2016.

\bibitem[Defazio and Bottou(2018)]{defazio2018ineffectiveness}
A.~Defazio and L.~Bottou.
\newblock On the ineffectiveness of variance reduced optimization for deep
  learning.
\newblock \emph{arXiv:1812.04529}, 2018.

\bibitem[Defazio et~al.(2014)Defazio, Bach, and
  Lacoste-Julien]{defazio2014saga}
A.~Defazio, F.~Bach, and S.~Lacoste-Julien.
\newblock Saga: A fast incremental gradient method with support for
  non-strongly convex composite objectives.
\newblock In \emph{NIPS}, 2014.

\bibitem[Facchinei and
  Pang(2003)]{facchineiFiniteDimensionalVariationalInequalities2003}
F.~Facchinei and J.-S. Pang.
\newblock \emph{Finite-{{Dimensional Variational Inequalities}} and
  {{Complementarity Problems Vol I}}}.
\newblock Springer {{Series}} in {{Operations Research}} and {{Financial
  Engineering}}, {{Finite}}-{{Dimensional Variational Inequalities}} and
  {{Complementarity Problems}}. {Springer-Verlag}, 2003.

\bibitem[Gidel et~al.(2019{\natexlab{a}})Gidel, Berard, Vincent, and
  Lacoste-Julien]{gidel2019variational}
G.~Gidel, H.~Berard, P.~Vincent, and S.~Lacoste-Julien.
\newblock A variational inequality perspective on generative adversarial nets.
\newblock In \emph{ICLR}, 2019{\natexlab{a}}.

\bibitem[Gidel et~al.(2019{\natexlab{b}})Gidel, Hemmat, Pezeshki, Priol, Huang,
  Lacoste-Julien, and Mitliagkas]{gidel19-momentum}
G.~Gidel, R.~A. Hemmat, M.~Pezeshki, R.~L. Priol, G.~Huang, S.~Lacoste-Julien,
  and I.~Mitliagkas.
\newblock Negative momentum for improved game dynamics.
\newblock In \emph{AISTATS}, 2019{\natexlab{b}}.

\bibitem[Glorot and Bengio(2010)]{glorot2010}
X.~Glorot and Y.~Bengio.
\newblock Understanding the difficulty of training deep feedforward neural
  networks.
\newblock In \emph{AISTATS}, 2010.

\bibitem[Goodfellow et~al.(2014)Goodfellow, Pouget-Abadie, Mirza, Xu,
  Warde-Farley, Ozair, Courville, and Bengio]{goodfellow2014generative}
I.~Goodfellow, J.~Pouget-Abadie, M.~Mirza, B.~Xu, D.~Warde-Farley, S.~Ozair,
  A.~Courville, and Y.~Bengio.
\newblock Generative adversarial nets.
\newblock In \emph{NIPS}, 2014.

\bibitem[Harker and Pang(1990)]{harker1990finite}
P.~T. Harker and J.-S. Pang.
\newblock Finite-dimensional variational inequality and nonlinear
  complementarity problems: a survey of theory, algorithms and applications.
\newblock \emph{Mathematical programming}, 1990.

\bibitem[He et~al.(2015)He, Zhang, Ren, and Sun]{resnet}
K.~He, X.~Zhang, S.~Ren, and J.~Sun.
\newblock Deep residual learning for image recognition.
\newblock \emph{arXiv:1512.03385}, 2015.

\bibitem[Heusel et~al.(2017)Heusel, Ramsauer, Unterthiner, Nessler, and
  Hochreiter]{heusel_gans_2017}
M.~Heusel, H.~Ramsauer, T.~Unterthiner, B.~Nessler, and S.~Hochreiter.
\newblock {GANs} trained by a two time-scale update rule converge to a local
  nash equilibrium.
\newblock In \emph{NIPS}, 2017.

\bibitem[Hofmann et~al.(2015)Hofmann, Lucchi, Lacoste-Julien, and
  McWilliams]{hofmann2015variance}
T.~Hofmann, A.~Lucchi, S.~Lacoste-Julien, and B.~McWilliams.
\newblock Variance reduced stochastic gradient descent with neighbors.
\newblock In \emph{NIPS}, 2015.

\bibitem[Ioffe and Szegedy(2015)]{bnorm}
S.~Ioffe and C.~Szegedy.
\newblock Batch normalization: Accelerating deep network training by reducing
  internal covariate shift.
\newblock In \emph{ICML}, 2015.

\bibitem[Iusem et~al.(2017)Iusem, Jofr{\'e}, Oliveira, and
  Thompson]{iusem2017extragradient}
A.~Iusem, A.~Jofr{\'e}, R.~I. Oliveira, and P.~Thompson.
\newblock Extragradient method with variance reduction for stochastic
  variational inequalities.
\newblock \emph{SIAM Journal on Optimization}, 2017.

\bibitem[Johnson and Zhang(2013)]{johnson2013accelerating}
R.~Johnson and T.~Zhang.
\newblock Accelerating stochastic gradient descent using predictive variance
  reduction.
\newblock In \emph{NIPS}, 2013.

\bibitem[Juditsky et~al.(2011)Juditsky, Nemirovski, and
  Tauvel]{juditsky2011solving}
A.~Juditsky, A.~Nemirovski, and C.~Tauvel.
\newblock Solving variational inequalities with stochastic mirror-prox
  algorithm.
\newblock \emph{Stochastic Systems}, 2011.

\bibitem[Kingma and Ba(2015)]{kingma2014adam}
D.~P. Kingma and J.~Ba.
\newblock Adam: A method for stochastic optimization.
\newblock In \emph{ICLR}, 2015.

\bibitem[Korpelevich(1976)]{korpelevich1976extragradient}
G.~Korpelevich.
\newblock The extragradient method for finding saddle points and other
  problems.
\newblock \emph{Matecon}, 1976.

\bibitem[Krizhevsky(2009)]{cifar10}
A.~Krizhevsky.
\newblock {Learning Multiple Layers of Features from Tiny Images}.
\newblock Master's thesis, 2009.

\bibitem[Leblond et~al.(2018)Leblond, Pederegosa, and
  Lacoste-Julien]{leblond2018improved}
R.~Leblond, F.~Pederegosa, and S.~Lacoste-Julien.
\newblock Improved asynchronous parallel optimization analysis for stochastic
  incremental methods.
\newblock \emph{JMLR}, 19\penalty0 (81):\penalty0 1--68, 2018.

\bibitem[Lecun and Cortes()]{mnist}
Y.~Lecun and C.~Cortes.
\newblock The {MNIST} database of handwritten digits.
\newblock URL \url{http://yann.lecun.com/exdb/mnist/}.

\bibitem[{Lim} and {Ye}(2017)]{lim2017geometricGan}
J.~H. {Lim} and J.~C. {Ye}.
\newblock {Geometric GAN}.
\newblock \emph{arXiv:1705.02894}, 2017.

\bibitem[Mescheder et~al.(2017)Mescheder, Nowozin, and
  Geiger]{mescheder_numerics_2017}
L.~Mescheder, S.~Nowozin, and A.~Geiger.
\newblock The numerics of {GANs}.
\newblock In \emph{NIPS}, 2017.

\bibitem[Miyato et~al.(2018)Miyato, Kataoka, Koyama, and
  Yoshida]{miyato2018spectral}
T.~Miyato, T.~Kataoka, M.~Koyama, and Y.~Yoshida.
\newblock Spectral normalization for generative adversarial networks.
\newblock In \emph{ICLR}, 2018.

\bibitem[Netzer et~al.(2011)Netzer, Wang, Coates, Bissacco, Wu, and
  Y.~Ng]{svhn}
Y.~Netzer, T.~Wang, A.~Coates, A.~Bissacco, B.~Wu, and A.~Y.~Ng.
\newblock Reading digits in natural images with unsupervised feature learning.
\newblock 2011.
\newblock URL \url{http://ufldl.stanford.edu/housenumbers/}.

\bibitem[Palaniappan and Bach(2016)]{palaniappan2016stochastic}
B.~Palaniappan and F.~Bach.
\newblock Stochastic variance reduction methods for saddle-point problems.
\newblock In \emph{NIPS}, 2016.

\bibitem[Radford et~al.(2016)Radford, Metz, and
  Chintala]{radford2016unsupervised}
A.~Radford, L.~Metz, and S.~Chintala.
\newblock Unsupervised representation learning with deep convolutional
  generative adversarial networks.
\newblock In \emph{ICLR}, 2016.

\bibitem[Reddi et~al.(2016)Reddi, Hefny, Sra, Poczos, and
  Smola]{reddi2016stochastic}
S.~J. Reddi, A.~Hefny, S.~Sra, B.~Poczos, and A.~Smola.
\newblock Stochastic variance reduction for nonconvex optimization.
\newblock In \emph{ICML}, 2016.

\bibitem[Robbins and Monro(1951)]{robbins1951stochastic}
H.~Robbins and S.~Monro.
\newblock A stochastic approximation method.
\newblock \emph{The Annals of Mathematical Statistics}, 1951.

\bibitem[Russakovsky et~al.(2015)Russakovsky, Deng, Su, Krause, Satheesh, Ma,
  Huang, Karpathy, Khosla, Bernstein, Berg, and Fei-Fei]{imagenet}
O.~Russakovsky, J.~Deng, H.~Su, J.~Krause, S.~Satheesh, S.~Ma, Z.~Huang,
  A.~Karpathy, A.~Khosla, M.~Bernstein, A.~C. Berg, and L.~Fei-Fei.
\newblock {ImageNet Large Scale Visual Recognition Challenge}.
\newblock \emph{IJCV}, 115\penalty0 (3):\penalty0 211--252, 2015.

\bibitem[Salimans et~al.(2016)Salimans, Goodfellow, Zaremba, Cheung, Radford,
  and Chen]{salimans2016improved}
T.~Salimans, I.~Goodfellow, W.~Zaremba, V.~Cheung, A.~Radford, and X.~Chen.
\newblock Improved techniques for training {GAN}s.
\newblock In \emph{NIPS}, 2016.

\bibitem[Schaul et~al.(2013)Schaul, Zhang, and LeCun]{schaul2013no}
T.~Schaul, S.~Zhang, and Y.~LeCun.
\newblock No more pesky learning rates.
\newblock In \emph{ICML}, 2013.

\bibitem[Schmidt et~al.(2017)Schmidt, Le~Roux, and Bach]{schmidt2017minimizing}
M.~Schmidt, N.~Le~Roux, and F.~Bach.
\newblock Minimizing finite sums with the stochastic average gradient.
\newblock \emph{Mathematical Programming}, 2017.

\bibitem[Shallue et~al.(2018)Shallue, Lee, Antognini, Sohl-Dickstein, Frostig,
  and Dahl]{shallue2018measuring}
C.~J. Shallue, J.~Lee, J.~Antognini, J.~Sohl-Dickstein, R.~Frostig, and G.~E.
  Dahl.
\newblock Measuring the effects of data parallelism on neural network training.
\newblock \emph{arXiv:1811.03600}, 2018.

\bibitem[Szegedy et~al.(2015)Szegedy, Vanhoucke, Ioffe, Shlens, and
  Wojna]{inceptionmodel}
C.~Szegedy, V.~Vanhoucke, S.~Ioffe, J.~Shlens, and Z.~Wojna.
\newblock Rethinking the inception architecture for computer vision.
\newblock \emph{arXiv:1512.00567}, 2015.

\bibitem[Tseng(1995)]{tsengLinearConvergenceIterative1995}
P.~Tseng.
\newblock On linear convergence of iterative methods for the variational
  inequality problem.
\newblock \emph{Journal of Computational and Applied Mathematics}, 1995.

\bibitem[Wilson et~al.(2017)Wilson, Roelofs, Stern, Srebro, and
  Recht]{wilson2017marginal}
A.~C. Wilson, R.~Roelofs, M.~Stern, N.~Srebro, and B.~Recht.
\newblock The marginal value of adaptive gradient methods in machine learning.
\newblock In \emph{NIPS}, 2017.

\bibitem[Xiao and Zhang(2014)]{xiao2014SVRG_NUS}
L.~Xiao and T.~Zhang.
\newblock A proximal stochastic gradient method with progressive variance
  reduction.
\newblock \emph{SIAM Journal on Optimization}, 24\penalty0 (4):\penalty0
  2057--2075, 2014.

\bibitem[{Zhang} et~al.(2018){Zhang}, {Goodfellow}, {Metaxas}, and
  {Odena}]{sagan}
H.~{Zhang}, I.~{Goodfellow}, D.~{Metaxas}, and A.~{Odena}.
\newblock {Self-Attention Generative Adversarial Networks}.
\newblock \emph{arXiv:1805.08318}, 2018.

\end{thebibliography}
\bibliographystyle{abbrvnat}

\newpage 

\onecolumn
\appendix

\newgeometry{}
\section{Noise in games}\label{sec:illustration}

\subsection{Why is convergence of the last iterate preferable?}
\label{sub:why_is_convergence_of_last_iterate_preferable_}

In light of Theorem~\ref{thm:diverge}, the behavior of the iterates on the unconstrained version of~\eqref{eq:problem_eg_fails} ($\epsilon =0$):
\begin{equation}\label{eq:problem_eg_fails_bounded}
  \min_{\vtheta \in \Theta} \max_{\vphi \in \Phi}  \frac{1}{n}\sum_{i=1}^n \vtheta^\top \mA_{i} \vphi
  \quad
  \text{where} \quad [\mA_{i}]_{kl} = 1 \;  \text{if} \; k=l=i \; \text{and} \; 0 \; \text{otherwise}.
\end{equation}
where $\Theta$ and $\Phi$ are compact and convex sets,
 is the following: they will diverge until they reach the boundary of $\Theta$ and $\Phi$ and then they will start to turn around the Nash equilibrium of~\eqref{eq:problem_eg_fails_bounded} lying on these boundaries. Using convexity properties, we can then show that the averaged iterates will converge to the Nash equilibrium of the problem. However, with an arbitrary large domain, this convergence rate may be arbitrary slow (since it depends on the diameter of the domain).

Moreover, this behavior might be even more problematic in a non-convex framework because even if by chance we initialize close to the Nash equilibrium, we would get away from it and we cannot rely on convexity to expect the average of the iterates to converge. 

Consequently, we would like optimization algorithms generating iterates that \emph{stay close to the Nash equilibrium}.

\section{Definitions and Lemmas} 
\label{sec:definitions_and_lemmas}

\subsection{Smoothness and Monotonicity of the operator} 
\label{sub:smoothness_and_monotonicity_of_the_operator}

Another important property used is the Lipschitzness of an operator. 
\begin{definition}
A mapping $F: \RR^p \to \RR^d$ is said to be $L$-Lipschitz if,
\begin{equation}
  \|F(\vomega) - F(\vomega')\|_2\leq L\|\vomega-\vomega'\|_2 \,, \quad \forall \vomega,\vomega' \in \Omega\,.
\end{equation}
\end{definition}

\begin{definition}
A differentiable function $f : \Omega \to \RR$ is said to be \emph{$\mu$-strongly convex} if 
\begin{equation}
  f(\vomega) \geq f(\vomega') + \nabla f(\vomega')^\top(\vomega-\vomega') + \frac{\mu}{2}\|\vomega-\vomega'\|_2^2\, \quad \forall \vomega,\vomega' \in \Omega \,.
\end{equation}
\end{definition}
\begin{definition}\label{def:convex-concave} A function $(\vtheta,\vphi) \mapsto \LL(\vtheta,\vphi)$ is said convex-concave if $\LL(\cdot,\vphi)$ is convex for all $\vphi \in \Phi$ and $\LL(\vtheta,\cdot)$ is concave for all $\vtheta \in \Theta$. An $\LL$ is said to be $\mu$-strongly convex concave if $(\vtheta,\vphi) \mapsto \LL(\vtheta,\vphi) - \frac{\mu}{2}\|\vtheta\|_2^2 + \frac{\mu}{2}\|\vphi\|_2^2$ is convex concave.
\end{definition}

\begin{definition}\label{def:strong_monotone_app}For $\mu_\vtheta,\mu_\vphi >0$, an operator $F: \vomega \mapsto (F_\vtheta(\vomega),F_\vphi(\vomega)) \in \R^{d+p}$ is said to be $(\mu_\vtheta,\mu_\vphi)$-strongly monotone if $\,\forall\, \vomega,  \vomega' \in \Omega \subset \R^{p+d}$,
\begin{equation} \notag
  (F(\vomega)- F(\vomega'))^\top (\vomega - \vomega') \geq \mu_\vtheta \|\vtheta -\vtheta'\|^2 + \mu_\vphi \|\vphi -\vphi'\|^2\,.
\end{equation}
where we noted $\,\vomega := (\vtheta,\vphi) \in \R^{d+p}$.
\end{definition}

\begin{definition}
 An operator $F: (\vomega),\in \R^{d}$ is said to be $\ell$-cocoercive, if for all $\vomega,  \vomega' \in \Omega$ we have
\begin{equation} \label{eq:weighted_norm}
 \|F(\vomega) -F(\vomega')\|^2 \leq 
  \ell(F(\vomega)- F(\vomega'))^\top (\vomega - \vomega') \,.
\end{equation} 
\end{definition}

\begin{proposition}[Folklore]
A $L$-Lipschitz and $\mu$-strongly monotone operator is $L^2/\mu$-cocoercive
\end{proposition}
\begin{proof} By applying lipschitzness and strong monotonicity,
\begin{equation}
    \|F(\vomega) -F(\vomega')\|^2 \leq  L^2 \|\vomega -\vomega'\|^2 \leq L^2/\mu (F(\vomega) -F(\vomega'))^\top(\vomega -\vomega') 
\end{equation}
\end{proof}

\begin{proposition}\label{prop:coco}
If $F(\vomega)= (\nabla f(\vtheta) + M\vphi, \nabla g(\vphi) - M^\top \vtheta)$, where $f$ and $g$ are $\mu$-strongly convex and $L$ smooth, then $\|M\|^2 = O(\mu L)$ is a sufficient condition for $F$ to be $\ell$-cocoercive with $\ell = O(L)$
\end{proposition}
\begin{proof}
We rewrite $F$ as the sum of the gradient of convex Lipschitz function $F_{grad}$ and a $L$-Lipschitz and $\mu$-strongly monotone operator $F_{mon}$:
\begin{equation}
    F_{grad}(\vomega) := (\nabla f(\vtheta)- \mu \vtheta , \nabla g(\vphi)-\mu \vphi)
    \quad \text{and} \quad F_{mon}:(M\vphi +\mu \vtheta,- M^\top \vtheta +\mu \vphi)
\end{equation}
Then
\begin{align}
    \|F(\vomega) -F(\vomega')\|^2 
    &\leq 2  \|F_{grad}(\vomega) -F_{grad}(\vomega')\|^2
    + 2 \|F_{mon}(\vomega) -F_{mon}(\vomega')\|^2\\
    & \leq 2 (L + \mu) (F_{grad}(\vomega) -F_{grad}(\vomega'))^\top(\vomega -\vomega') \\
    & \quad + 2 (\|M\|+\mu)^2/\mu (F_{mon}(\vomega) -F_{mon}(\vomega'))^\top(\vomega -\vomega') \\
     & = O(L) (F_{grad}(\vomega) -F_{grad}(\vomega'))^\top(\vomega -\vomega')  \\
     & \quad + O(L) (F_{mon}(\vomega) -F_{mon}(\vomega'))^\top(\vomega -\vomega') \\
     & = O(L) (F(\vomega) -F(\vomega'))^\top(\vomega -\vomega')  
\end{align}
where for the second inequality we used that a $(L+\mu)$-Lipschitz convex function is $(L+\mu)$-cocoercive and Proposition~\ref{prop:coco}.
\end{proof}

\section{Proof of Theorems} 
\label{sec:proof_of_theorems}

\subsection{Proof of Theorem~\ref{thm:diverge}}
\label{sub:proof_of_theorem_thm:diverges}

\proof
We consider the following stochastic optimization problem,
\begin{equation}\label{eq:problem_eg_fails}
   \frac{1}{n}  \sum_{i=1}^n \frac{\epsilon}{2} \theta_i^2 +\bm{\theta}^\top \bm{A}_{i} \bm{\varphi} - \frac{\epsilon}{2}\varphi_i^2  
   = \frac{1}{n}  \sum_{i=1}^n \frac{\epsilon}{2} \|\bm{A}_{i}\bm\theta\|^2 +\bm{\theta}^\top \bm{A}_{i} \bm{\varphi} - \frac{\epsilon}{2}\|\bm{A}_{i}\bm\varphi\|^2  
\end{equation}
where $[\mA_{i}]_{kl} = 1$ if $k=l=i$ and $0$ otherwise. Note that $(\mA_{i})^\top = \mA_{i}$ for $1\leq i \leq n$.
Let us consider the extragradient method where to compute an unbiased estimator of the gradients at $(\vtheta,\vphi)$ we sample $i \in \{1,\ldots,n\}$ and use $[\mA_{i} \vtheta ,\, \mA_{i}\vphi]$ as estimator of the vector flow.

In this proof we note, $\mA_{I} := \sum_{i\in I} \mA_{i}$ and $\vtheta^{(I)}$ the vector such that $[\vtheta^{(I)}]_i= [\vtheta]_i$ if $i \in I$ and $0$ otherwise.  Note that $\mA_{I} \vtheta = \vtheta^{(I)}$ and that $\mA_{I} \mA_{J} =\mA_{I \cap J}$.

Thus the extragradient update rule can be noted as 
\begin{equation}
\left\{
\begin{aligned}
  &\vtheta_{t+1} = (1-\eta \mA_{I} \epsilon) \vtheta_t - \eta \mA_{I} ( (1-\eta \mA_{J} \epsilon)\vphi_t + \eta \mA_{J} \vtheta_t)  \\
  &\vphi_{t+1} = (1-\eta \mA_{I} \epsilon)\vphi_t +\eta \mA_{I} ((1-\eta\mA_{J}\epsilon)\vtheta_t - \eta \mA_{J}\vphi_t)
\end{aligned}
\right.
\end{equation}
\\where $I$ is the mini-batch sampled (without replacement) for the update and $J$ the mini-batch sampled (without replacement) for the extrapolation.

We can thus notice that, when $I \cap J = \emptyset$, we have 
\begin{equation}\label{eq:proof_grad}
  \left\{
  \begin{aligned}
  \vtheta_{t+1} &= \vtheta_t - \eta\epsilon \vtheta^{(I)}_t  - \eta  \vphi^{(I)}_t  \\
  \vphi_{t+1} &= \vphi_t - \eta\epsilon  \vphi^{(I)}_t  +\eta \vtheta^{(I)}_t  \,,
  \end{aligned}
\right.
\end{equation}
and otherwise,
\begin{equation}\label{eq:proof_extra_grad}
   \left\{
  \begin{aligned}
  \vtheta_{t+1} &= \vtheta_t - \eta\epsilon \vtheta^{(I)}_t  - \eta \vphi^{(I)}_t  - \eta^2 (\vtheta^{(I \cap J)}_t - \epsilon \vphi^{(I \cap J)}_t) \\
  \vphi_{t+1} &= \vphi_t - \eta\epsilon  \vphi^{(I)}_t + \eta \vtheta^{(I)}_t - \eta^2(\vphi^{(I \cap J)}_t + \epsilon \vtheta^{(I \cap J)}_t)  \,.
  \end{aligned}
\right.
\end{equation}
The intuition is that, on one hand, when $I \cap J = \emptyset$ (which happens with high probability when $|I|<<n$, e.g., when $|I|=1$, $\mathbb{P}(I \cap J = \emptyset) = 1 - 1/n$), the algorithm performs an update that get away from the Nash equilibrium when $2\epsilon\geq\eta$:
\begin{equation}
\eqref{eq:proof_grad}
\; 
\Rightarrow
\;
N_{t+1}= N_t +(\eta^2\epsilon^2 +\eta^2 - 2\eta \epsilon) N_t^{(I)} \,,
\end{equation}
where $N_t := \|\vtheta_{t}\|^2 + \|\vphi_{t}\|^2$ and $N_t^{(I)} := \|\vtheta_{t}^{(I)}\|^2 + \|\vphi_{t}^{(I)}\|^2$. On the other hand,
The updates that provide improvement only happen when $I\cap J$ is large (which happen with low probability, e.g., when $|I|=1$, $\mathbb{P}(I \cap J \neq \emptyset) =1/n$):
\begin{equation}\label{eq:proof_extra_grad_2}
  \eqref{eq:proof_extra_grad}
\;
\Rightarrow
\;
N_{t+1}= N_t -  N_t^{(I)}(2\eta \epsilon -  \eta^2(1+\epsilon^2)) - N_t^{(I\cap J)}(  2\eta^2 - \eta^4(1+\epsilon^2)
\end{equation}
Conditioning on $\vtheta_t$ and $\vphi_t$, we get that 
\begin{equation}
  \E[N_t^{(I\cap J)}|\vtheta_t,\vphi_t] = \sum_{i=1}^n \mathbb{P}(i \in I\cap J)([\vtheta_t]_i^2 + [\vphi_t]^2_i) 
  \quad \text{and} \quad  \mathbb{P}(i \in I\cap J) =\mathbb{P}(i \in I) \mathbb{P}(i \in J)=  \frac{|I|^2}{n^2} \,.
\end{equation}
Leading to,
\begin{equation}
  \E[N_t^{(I\cap J)}|\vtheta_t,\vphi_t] = \frac{|I|^2}{n^2} \sum_{i=1}^n([\vtheta_t]_i^2 + [\vphi_t]^2_i) =  \frac{|I|^2}{n^2} N_t
  \quad \text{and} \quad  \E[N_t^{(I)}|\vtheta_t,\vphi_t] = \frac{|I|}{n}N_t \,.
\end{equation}
Plugging these expectations in~\eqref{eq:proof_extra_grad_2}, we get that,
\begin{align}
 \mathbb E[N_{t+1}] 
 &= \left(1 -  \tfrac{|I|}{n}(2\eta \epsilon -  \eta^2(1+\epsilon^2)) - \tfrac{|I|^2}{n^2}(  2\eta^2 - \eta^4(1+\epsilon^2))\right) \mathbb E[N_t] \,.
\end{align}
Consequently for $\eta<\epsilon$ we get,
\begin{equation}
  \E[N_{t+1}] \geq \left(1 - 2\eta^2 \frac{|I|^2}{n^2} +  \eta^2\frac{|I|}{n}\right) \E[N_t] \,.
\end{equation}

To sum-up, if $|I|$ is not large enough (more precisely if $2|I|\leq n$), we have the geometric divergence of the quantity $\E[N_t] := \E[\|\vtheta_{t}\|^2 + \|\vphi_{t}\|^2]$ for any $\eta \geq \epsilon$.
\endproof

\subsection{Proof of Theorem~\ref{thm:extra-svrg}} 
\label{sub:proof_of_theorem_3}

\paragraph{Setting of the Proof.} 
\label{par:setting_of_the_proof_}
We will prove a slightly more general result than Theorem~\ref{thm:extra-svrg}. We will work in the context of monotone operator. Let us consider the \emph{general} extrapolation update rule,
\begin{equation}
\label{eq:app_extragradient}
\left\{
\begin{aligned}
  \text{Extrapolation:} \quad
  &\vomega_{t+\frac{1}{2} } = \vomega_t - \eta_t \vg_t \\
  \text{Update:} \quad 
  &\vomega_{t+1} = \vomega_t - \eta_t \vg_{t+1/2} \,,
\end{aligned}
\right.
\end{equation}
where $\vg_t$ depends on $\vomega_t$ and $\vg_{t+1/2}$ depends on $\vomega_{t+1/2}$. For instance, $\vg_t$ can either be $F(\vomega_t)$, $F_{i_t}(\vomega_t)$ or the SVRG estimate defined in~\eqref{def_g_i}.

This update rule generalizes~\eqref{eq:extragradient} for 2-player games~\eqref{eq:two_player_games} and ExtraSVRG (Alg.~\ref{alg:svrg_gan}).

Let us first state a lemma standard in convex analysis (see for instance~\citep{boyd2004convex}),
\begin{lemma}\label{lemma:ineg} Let $\vomega \in \Omega$ and $\vomega^+ \defas P_\Omega(\vomega + \uu)$ 
then for all $\vomega' \in \Omega$ we have,
\begin{equation}
  \|\vomega^+-\vomega'\|_2^2 \leq \|\vomega-\vomega'\|_2^2 + 2 \uu^\top (\vomega^+-\vomega') - \|\vomega^+ -\vomega\|_2^2\,.
\end{equation}
\end{lemma}
\proof[\textbf{Proof of Lemma~\ref{lemma:ineg}}] We start by simply developing,
\begin{align*}
  \|\vomega^+-\vomega'\|_2^2 = \|(\vomega^+- \vomega) + (\vomega -\vomega')\|_2^2
  &= \|\vomega-\vomega'\|_2^2 + 2 (\vomega^+-\vomega)^\top (\vomega - \vomega') + \|\vomega^+ -\vomega\|_2^2 \\
  &= \|\vomega-\vomega'\|_2^2 + 2 (\vomega^+-\vomega)^\top (\vomega^+ - \vomega') - \|\vomega^+ -\vomega\|_2^2 \, .
\end{align*}
Then since $\vomega^+$ is the projection onto the convex set $\Omega$ of $\vomega + \uu$ we have that $(\vomega^+ - (\vomega + \uu))^\top(\vomega^+-\vomega') \leq 0 \,, \; \forall\, \vomega'\in \Omega$, leading to the result of the Lemma. \endproof

\begin{lemma}\label{lemma:strong_monotone} If $F$ is $(\mu_\vtheta,\mu_\vphi)$-strongly monotone for any $\vomega, \vomega', \vomega'' \in \Omega$ we have,
\begin{equation}
  \mu_\vtheta \left( \|\vtheta-\vtheta''\|_2^2 - 2\|\vtheta'-\vtheta\|_2^2 \right) + \mu_\vphi \left( \|\vphi-\vphi''\|_2^2 - 2\|\vphi'-\vphi\|_2^2 \right) \leq 2(F(\vomega')- F(\vomega''))^\top (\vomega'-\vomega'')
  \,, 
\end{equation}
where we noted $\vomega := (\vtheta,\vphi)$.
\end{lemma}
\proof 
By $(\mu_\vtheta,\mu_\vphi)$-strong monotonicity,
\begin{align}
 2\mu_\vtheta  \|\vtheta'-\vtheta''\|_2^2 + 2\mu_\vphi  \|\vphi'-\vphi''\|_2^2
 &\leq  2(F(\vomega'')-F(\vomega''))^\top (\vomega'-\vomega'')
\end{align}
and then we use the inequality $2\|\va'-\va''\|_2^2 \geq  \|\va-\va''\|_2^2 - 2\|\va'-\va\|_2^2 $ to get the result claimed. 
\endproof

Using this update rule we can thus deduce the following lemma, 
the derivation of this lemma is very similar from the derivation of~\citet[Lemma~12.1.10]{harker1990finite}.
\begin{lemma}\label{lemma:4} Considering the update rule~\eqref{eq:app_extragradient}, we have for any $\vomega \in \Omega$ and any $t \geq 0$,
\begin{equation}
  2 \eta_t \vg_{t+1/2}^\top(\vomega_{t+1/2}-\vomega)
   \leq \|\vomega_t - \vomega\|_2^2 - 
  \|\vomega_{t+1}-\vomega\|_2^2  - \|\vomega_{t+1/2} -\vomega_t\|_2^2
  + \eta_t^2 \|\vg_{t}-\vg_{t+1/2}\|_2^2 \, .
\end{equation} 
\end{lemma}
\proof 
By applying Lem.~\ref{lemma:ineg} for $(\vomega,\uu,\vomega^+,\vomega') = (\vomega_t,-\eta_t \vg_{t+1/2},\vomega_{t+1},\vomega)$ and $(\vomega,\uu,\vomega^+,\vomega') = (\vomega_{t},-\eta_t \vg_{t},\vomega_{t+1/2},\vomega_{t+1})$, we get,
\begin{equation}\label{eq:proof_alg_v1.2_1}
  \|\vomega_{t+1}-\vomega\|_2^2 
  \leq \|\vomega_t - \vomega\|_2^2 - 2 \eta_t \vg_{t+1/2}^\top(\vomega_{t+1}-\vomega) - \|\vomega_{t+1}-\vomega_t\|_2^2 \,,
\end{equation}
and
\begin{equation} \label{eq:proof_alg_v1.2_2}
  \|\vomega_{t+1/2} -\vomega_{t+1}\|_2^2 \leq \|\vomega_t-\vomega_{t+1}\|_2^2 - 2 \eta_t \vg_{t}^\top(\vomega_{t+1/2}-\vomega_{t+1}) - \|\vomega_{t+1/2} -\vomega_t\|_2^2\,.
\end{equation}
Summing \eqref{eq:proof_alg_v1.2_1} and~\eqref{eq:proof_alg_v1.2_2} we get,
\begin{align}
  \|\vomega_{t+1}-\vomega\|_2^2 
  &\leq \|\vomega_t - \vomega\|_2^2 - 2 \eta_t \vg_{t+1/2}^\top(\vomega_{t+1}-\vomega)\\
  &\quad - 2 \eta_t \vg_{t}^\top(\vomega_{t+1/2}-\vomega_{t+1})  - \|\vomega_{t+1/2} -\vomega_t\|_2^2 - \|\vomega_{t+1/2} -\vomega_{t+1}\|_2^2 \\
  & =  \|\vomega_t - \vomega\|_2^2 - 2 \eta_t \vg_{t+1/2}^\top(\vomega_{t+1/2}-\vomega)  - \|\vomega_{t+1/2} -\vomega_t\|_2^2 - \|\vomega_{t+1/2} -\vomega_{t+1}\|_2^2 \notag\\
  &\quad - 2 \eta_t (\vg_{t} -\vg_{t+1/2})^\top(\vomega_{t+1/2}-\vomega_{t+1}) \, .
\end{align}
Then, we can use Young's inequality $-2a^\top b \leq \|a\|_2^2 + \|b\|_2^2$ to get,
\begin{align}
  \|\vomega_{t+1}-\vomega\|_2^2 
  & \leq \|\vomega_t - \vomega\|_2^2 - 2 \eta_t \vg_{t+1/2}^\top(\vomega_{t+1/2}-\vomega) +  \eta_t^2 \|\vg_{t}-\vg_{t+1/2}\|_2^2 \notag\\
  & \quad + \|\vomega_{t+1/2}-\vomega_{t+1}\|_2^2
  - \|\vomega_{t+1/2} -\vomega_t\|_2^2 - \|\vomega_{t+1/2} -\vomega_{t+1}\|_2^2 \\
  & = \|\vomega_t - \vomega\|_2^2 - 2 \eta_t \vg_{t+1/2}^\top(\vomega_{t+1/2}-\vomega)+  \eta_t^2 \|\vg_{t}-\vg_{t+1/2}\|_2^2 - \|\vomega_{t+1/2} -\vomega_t\|_2^2   \,.
\end{align}
\endproof Note that if we would have set $\vg_t = \bm{0}$ and $\vg_{t+1/2}$ any estimate of the gradient at $\vomega_t$ we recover the standard lemma for gradient method.

Let us consider \emph{unbiased} estimates of the gradient,
\begin{equation}\label{def_g_i}
  \vg_i(\vomega) := \frac{1}{n\pi_i} \left(    F_{i}(\vomega) - \valpha_{i}\right) + \bar \valpha \,,
\end{equation}
where $\bar \valpha := \frac{1}{n} \sum_{j=1}^n \valpha_j$, 
the index $i$ are (potentially) non-uniformly sampled from $\{1,\ldots,n\}$ with replacement  according to $\vpi$ and $F(\vomega) := \frac{1}{n} \sum_{j=1}^n F_i(\vomega)$. Hence we have that $\E[\vg_i(\vomega)] = F(\vomega)$, where the expectation is taken with respect to the index $i$ sampled from $\vpi$.

We will consider a class of algorithm called \emph{uniform memorization algorithms} first introduced by~\citep{hofmann2015variance}. This class of algorithms describes a large subset of variance reduced algorithms taking advantage of the finite sum formulation such as SAGA~\citep{defazio2014saga}, SVRG~\citep{johnson2013accelerating} or $q$-SAGA and $\mathcal{N}$-SAGA~\citep{hofmann2015variance}. In this work, we will use a slightly more general definition of such algorithm in order to be able to handle extrapolation steps:
\begin{definition}[Extension of~\citep{hofmann2015variance}]\label{def:uniform_mem_alg}
A uniform $q$-memorization algorithm evolves iterates $(\vomega_t)$ according to~\eqref{eq:app_extragradient}, with $\vg_t$ defined in~\eqref{def_g_i} and selecting in each iteration $t$ a random index set $J_t$ of memory locations to update according to,
\begin{equation}
  \valpha_k^{(0)}:= F_k(\vomega_0) \;,\quad
    \valpha_k^{(t+1/2)} := \valpha_k^{(t)},\; \forall k \in \{1,\ldots,n\}
  \quad \text{and} \quad
     \valpha_k^{(t+1)} := \left\{ 
  \begin{aligned}
  F_k(\vomega_{t}) \quad &\text{if} \;\; k \in J_t \\
  \valpha_k^{(t)} \quad &\text{otherwise}.
  \end{aligned}
\right.
\end{equation}
such that any $k$ has the same probability $q/n$ to be updated, i.e., $P\{k\} = \sum_{J_t, k \in J_t} P(J_t) = q/n$, $\forall k \in \{1,\ldots,n\}.$
\end{definition}
In the case of SVRG, either $J_t = \emptyset$ or $J_t = \{1,\ldots,n\}$ (when we update the snapshot).

We have the following lemmas,
\begin{lemma}\label{lemma:g_t_g_t12}
For any  $t\geq 0$, if we consider a $q$-memorization algorithm we have
\begin{equation}
  \E[\|\vg_t - \vg_{t+1/2}\|^2 ] \leq 10\E[\|\tfrac{1}{n \pi_i}(F_i(\vomega^*) - \valpha^{(t)}_i) \|^2]+ 10\E[\|\tfrac{1}{n \pi_i}(F_i(\vomega^*) - F_i(\vomega_t) )\|^2] + 5  \bar L^2\E[\|\vomega_{t} -\vomega_{t+1/2}\|^2]\,. \notag
   \end{equation}
\end{lemma}
\proof 
We use an extended version of Young's inequality: $\|\sum_{i=1}^k\va_i \|^2 \leq k \sum_{i=1}^k\|\va_i\|^2$,
\begin{align*}
  \|\sum_{i=1}^k\va_i \|^2 
  & = \sum_{i,j=1}^k\va_i^\top\va_j   \\
  & \leq \frac{1}{2}\sum_{i,j=1}^k\|\va_i\|^2 + \|\va_j\|^2 \\
  & = k \sum_{i=1}^k\|\va_i\|^2 \,,
\end{align*}
where we used that $2 \va^\top \vb \leq + \|\va\|^2 + \|\vb\|^2$.
We combine Young's inequality with the definition of $q$-memorization algorithm: $\vg_t = \frac{1}{n\pi_i}(F_i(\vomega_t) - \bar \valpha^{(t)}_{i})$ and $\vg_{t+1/2} =\frac{1}{n\pi_j}( F_j(\vomega_{t+1/2}) - \bar \valpha^{(t)}_{j} )$ to get (we omit the $t$ subscript for $i$ and $j$ and we note $\bar  \valpha^{(t)}_i := \valpha^{(t)}_i -  n\pi_i \bar \valpha^{(t)}$), 
\begin{align*}
  \|\vg_t - \vg_{t+1/2}\|^2
  & = \|\tfrac{1}{n \pi_i}(F_i(\vomega_t) -\bar  \valpha^{(t)}_i) - \tfrac{1}{n \pi_j}(F_j(\vomega_{t+1/2}) - \bar\valpha^{(t)}_j)\|^2 \\
  & = \|\tfrac{1}{n \pi_i}(F_i(\vomega_t) - \bar \valpha^{(t)}_i) +  \tfrac{1}{n \pi_j}(F_j(\vomega_{t}) - F_j(\vomega_{t+1/2})) +  \tfrac{1}{n \pi_j}(\bar \valpha^{(t)}_j - F_j(\vomega_{t}))\|^2 \\
\\\\   &\leq 5\E[\|\tfrac{1}{n \pi_i}(F_i(\vomega^*) - \bar \valpha^{(t)}_i)) \|^2]+ 5\E[\|\tfrac{1}{n \pi_j}(\bar \valpha^{(t)}_j - F_j(\vomega^*))\|^2] \\
  &  \quad + 5\E[\|\tfrac{1}{n \pi_i}(F_i(\vomega^*) - F_i(\vomega_t)) \|^2]+ 5\E[\|\tfrac{1}{n \pi_j}(F_j(\vomega^*) - F_j(\vomega_{t}))\|^2] \\
  & \quad + 5  \E[\|\tfrac{1}{n \pi_j}(F_j(\vomega_{t}) - F_j(\vomega_{t+1/2}))\|^2]
\end{align*}
Notice that since $i_t$ and $j_t$ are independently sampled from the same distribution we have
\begin{equation}
  \E[\tfrac{1}{n^2 \pi_{j_t}^2}\|F_{j_t}(\vomega^*) - \valpha^{(t)}_{j_t}\|^2] = \E[\tfrac{1}{n^2 \pi_{i_t}^2}\|F_{i_t}(\vomega^*) - \valpha^{(t)}_{i_t}\|^2]\,.
\end{equation}
Note that we have (using that $\E[F_i(\vomega^*)]=0$ and $\E[\valpha^{(t)}_i] = \bar \valpha^{(t)}$), 
\begin{equation}
\E[\|\tfrac{1}{n \pi_i}(F_i(\vomega^*) - \bar \valpha^{(t)}_i) \|^2]  = 
\E[\|\tfrac{1}{n \pi_i}(F_i(\vomega^*) - \valpha^{(t)}_i) \|^2] - \|\bar \valpha^{(t)}\|^2 \leq 
\E[\|\tfrac{1}{n \pi_i}(F_i(\vomega^*) - \valpha^{(t)}_i) \|^2]
\end{equation}
By assuming that each $F_i$ is $L_i$-Lipschitz we get,
\begin{align}
  \E[\tfrac{1}{n^2 \pi_{j_t}^2}\|F_j(\vomega_{t}) - F_j(\vomega_{t+1/2})\|^2] 
  &=\frac{1}{n^2}\sum_{j=1}^n \frac{1}{\pi_j} \E[\|F_j(\vomega_{t}) - F_j(\vomega_{t+1/2})\|^2]\\ 
  & \leq  \frac{1}{n^2}\sum_{j=1}^n \frac{L_j^2}{\pi_j}\E[\|\vomega_{t} - \vomega_{t+1/2}\|^2] \\
  &= \bar L^2 \E[\|\vomega_{t} - \vomega_{t+1/2}\|^2 \,,
\end{align}
where $\bar L^2 := \frac{1}{n^2}\sum_{i=1}^n \frac{L_i^2}{\pi_j}$. Note that $\vomega_t$ and $\vomega_{t+1/2}$ do not depend on $j_t$ (which is the index sampled for the update step), that is not the case for $i$ (the index for the extrapolation step) since $\vomega_{t+1/2}$ is the result of the extrapolation.
 \endproof

This lemma make appear the quantity $\E[\|\tfrac{1}{n \pi_i}(F_i(\vomega^*) - \bar \valpha^{(t)}_i) \|^2]$ that we need to bound. In order to do that we prove the following lemma,
\begin{lemma}\label{lemma:geom_Liap}
Let $(\valpha_j^{(t)})$ be updated according to the rules of a $q$-uniform memorization algorithm~ (Def.~\ref{def:uniform_mem_alg}). Let us note $H_t := \tfrac{1}{n}\sum_{i=1}^n \tfrac{1}{n\pi_i}\|F_i(\vomega^*) - \valpha_i^{(t)}\|^2$.
For any $t \in \NN$, 
\begin{equation}
  \E[H_{t+1}] =  \frac{q}{n}\E[\|\tfrac{1}{n\pi_{i_t}}(F_{i_t}(\vomega_t)- F_{i_t}(\vomega^*))\|^2] +\frac{n-q}{n} \E[H_t] \,.
   \end{equation}
\proof We will use the definition of $q$-uniform memorization algorithms (saying that $\valpha_{i}$ is updated at time $t+1$ with probability $q/n$). We call this event "$i$ updated",
\begin{align*}
  \E[H_{t+1}] 
  &:= \E[\frac{1}{n}\sum_{i =1}^n \tfrac{1}{n\pi_i}\|\valpha_i^{(t+1)}- F_i(\vomega^*) \|^2] \\
  & = \frac{1}{n}\E[\sum_{i \text{ updated}}\tfrac{1}{n\pi_i}\|\valpha_i^{(t+1)}- F_i(\vomega^*) \|^2 
      + \sum_{i \text{ not updated}} \tfrac{1}{n\pi_i}\|\valpha_i^{(t+1)}- F_i(\vomega^*) \|^2] \\
  & =\frac{1}{n} \E[\sum_{i \text{ updated}}\tfrac{1}{n\pi_i}\|F_i(\vomega_t)- F_i(\vomega^*) \|^2 
      + \sum_{i \text{ not updated}} \tfrac{1}{n\pi_i}\|\valpha_i^{(t)}- F_i(\vomega^*) \|^2] \\
  & =\frac{1}{n} \sum_{i=1}^n  \mathbf{P}(i \text{ updated})\tfrac{1}{n\pi_i}\E[\|F_i(\vomega_t)- F_i(\vomega^*) \|^2 
      + \frac{1}{n} \sum_{i=1}^n  \mathbf{P}(i \text{ not updated}) \tfrac{1}{n\pi_i}\E[\|\valpha_i^{(t)}- F_i(\vomega^*) \|^2] \\
\\\\& =  \frac{q}{n}\E[\|\tfrac{1}{n\pi_{i_t}}(F_{i_t}(\vomega_t)- F_{i_t}(\vomega^*)) \|^2 ] +\frac{n-q}{n} \E[H_t]
\end{align*}
\endproof 
\end{lemma}

Using all these lemmas we can prove our theorem.

\begin{reptheorem}{thm:extra-svrg} Under Assumption~\ref{assump:SVRG}, after $t$ iterations, the iterate $\vomega_t$ computed by a $q$-memorization algorithm with step-sizes $(\eta_\theta,\eta_\phi) \leq \big( (40 \bar \ell_\vtheta)^{-1},({40\bar \ell_\vphi})^{-1} \big)$ verifies: 
\begin{equation}
  \E[\|\vomega_t - \vomega^*\|_2^2] \leq  \left(1-\min \left\{\Big(\frac{\eta \mu}{4} + \frac{11\eta^2\bar\gamma^2}{25}\Big), \frac{2q}{5n}\right\} \right)^t
  \E[\|\vomega_0 - \vomega^*\|_2^2]  \,.
\end{equation}
\end{reptheorem}
\proof In this proof we will consider a constant step-size $\eta_t = (\eta_\vtheta,\eta_\phi)$. For simplicity of notations we will consider the notation,
 \begin{align*}
 &\bar L^2 \| \vomega \|^2 :=  \bar L_\vtheta^2\|\vtheta\|^2 + \bar L_\vphi^2 \|\vphi\|^2 \, 
 ,\quad \eta^2 \| \vomega \|^2 :=  \eta_\vtheta^2\|\vtheta\|^2 + \eta_\vphi^2 \|\vphi\|^2 \,
 ,\quad \mu \| \vomega \|^2 :=  \mu_\vtheta^2\|\vtheta\|^2 + \mu_\vphi^2 \|\vphi\|^2 \\
 &\qquad \qquad \quad  \eta \mu = (\eta_\vtheta \mu_\vtheta ,\eta_\vphi \mu_\vphi) \,, \quad \sigma \bar L^2 = (\sigma_\vtheta \bar L_\vtheta^2,\sigma_\vphi \bar L_\vphi^2)  \quad \text{and} \quad \eta^2 \bar L^2 = (\eta_\vtheta^2 \bar L_\vtheta^2,\eta_\vphi^2 \bar L_\vphi^2) \,.
 \end{align*}

We start by recalling Lemma~\ref{lemma:4},
\begin{equation}\label{eq:proof_thm}
\|\vomega_{t+1}-\vomega^*\|_2^2 \leq  \|\vomega_t - \vomega^*\|_2^2 - 2\eta\vg_{t+1/2}^\top(\vomega_{t+1/2} - \vomega^*)
  - (1 - 2 \eta \mu) \|\vomega_{t+1/2} -\vomega_t\|_2^2
  + \eta^2 \|\vg_{t}-\vg_{t+1/2}\|_2^2 \, .
\end{equation}
We can then take the expectation and plug-in the expression of $\E[\|\vg_{t}-\vg_{t+1/2}\|_2^2]$ from Lemma~\ref{lemma:g_t_g_t12},
\begin{align*}
\E[\|\vomega_{t+1}-\vomega^*\|_2^2] 
 &\leq  \E[\|\vomega_t - \vomega^*\|_2^2] - 2\eta \E[F(\vomega_{t+1/2})^\top(\vomega_{t+1/2} - \vomega^*)]
  - (1 - 2 \eta \mu - 5 \eta^2 \bar L^2) ]\E[\|\vomega_{t+1/2} -\vomega_t\|_2^2]\\
 &\quad  + \eta^2 (10\E[\|\tfrac{1}{n \pi_i}(F_i(\vomega^*) - \valpha^{(t)}_i) \|^2]+ 10\E[\|\tfrac{1}{n \pi_i}(F_i(\vomega^*) - F_i(\vomega_t) )\|^2]) \, .
\end{align*}
Let us define $\LL_t := \E[\|\vomega_t - \vomega^*\|_2^2] + \sigma \E[H_t]$, where $H_t := \tfrac{1}{n}\sum_{i=1}^n \frac{1}{n\pi_i} \|F_i(\vomega^*) - \valpha_i^{(t)}\|^2 $. We can combine~\eqref{eq:proof_thm} with Lemma~\ref{lemma:geom_Liap} multiplied by a constant $\sigma>0$ that we will set later to get
\begin{align*}
  \LL_{t+1} 
  &= \E[\|\vomega_{t+1} - \vomega^*\|_2^2] + \sigma \E[H_{t+1}] \\
  &\leq  \E[\|\vomega_t - \vomega^*\|_2^2] - 2\eta \E[F(\vomega_{t+1/2})^\top(\vomega_{t+1/2} - \vomega^*)]
  - (1 - 2 \eta \mu - 5 \eta^2 \bar L^2) \E[\|\vomega_{t+1/2} -\vomega_t\|_2^2] \\
  & \quad  + (\tfrac{\sigma q}{n} + 10 \eta^2)\E[\|\tfrac{1}{n \pi_i}(F_i(\vomega^*) - F_i(\vomega_t) ) \|^2] + (\tfrac{10\eta^2}{\sigma} + \tfrac{n-q}{n} ) \sigma \E[H_t]\,.
 \end{align*}
 Since $i_t$ and $j_t$ are independently drawn from the same distribution, we have, $\E[\|\tfrac{1}{n \pi_i}(F_i(\vomega^*) - F_i(\vomega_{t}) ) \|^2] = \E[\|\tfrac{1}{n \pi_j}(F_j(\vomega^*) - F_j(\vomega_{t}) ) \|^2]$ and thus,
 \begin{align*}
    \LL_{t+1} 
  &\leq  \E[\|\vomega_t - \vomega^*\|_2^2] - 2\eta \E[F(\vomega_{t+1/2})^\top(\vomega_{t+1/2} - \vomega^*)]
  - (1 - 2 \eta \mu - 5 \eta^2 \bar L^2) \|\vomega_{t+1/2} -\vomega_t\|_2^2 \\
  & \quad  + (\tfrac{\sigma q}{n} + 10 \eta^2)\E[\|\tfrac{1}{n \pi_j}(F_j(\vomega^*) - F_j(\vomega_t) ) \|^2 + (\tfrac{10\eta^2}{\sigma} + \tfrac{n-q}{n} ) \sigma \E[H_t]\\  
  &\leq  \E[\|\vomega_t - \vomega^*\|_2^2] 
  - (1 - 2 \eta \mu - 5 \eta^2 \bar L^2 - 2(\tfrac{\sigma q}{n} + 10 \eta^2) \bar L^2) \|\vomega_{t+1/2} -\vomega_t\|_2^2 \\
  & \quad  - 2\eta \E[F(\vomega_{t+1/2})^\top(\vomega_{t+1/2} - \vomega^*)]+ 2(\tfrac{\sigma q}{n} + 10 \eta^2)\E[\|\tfrac{1}{n \pi_j}(F_j(\vomega^*) - F_j(\vomega_{t+1/2}) ) \|^2 \\
  &\quad+ (\tfrac{10\eta^2}{\sigma} + \tfrac{n-q}{n} ) \sigma \E[H_t] \\
 & \leq 
  \E[\|\vomega_t - \vomega^*\|_2^2]
  - (1 - 2 \eta \mu - 5 \eta^2 \bar L^2 - 2(\tfrac{\sigma q}{n} + 10 \eta^2) \bar L^2) \|\vomega_{t+1/2} -\vomega_t\|_2^2  \\
  & \quad  - 2\eta \E[F(\vomega_{t+1/2})^\top(\vomega_{t+1/2} - \vomega^*)]  + (\tfrac{10\eta^2}{\sigma} + \tfrac{n-q}{n} ) \sigma \E[H_t]   \\
  &\quad+ 2(\tfrac{\sigma q}{n} + 10 \eta^2)\E[\tfrac{\ell_j}{n^2 \pi_j^2}(F_j(\vomega^*) - F_j(\vomega_{t+1/2}))^\top(\vomega^*-\vomega_{t+1/2})]
\end{align*}
where for the second inequality we used Young's inequality and the Lipchitzness of $F_j$
and for the last one we used the co-coercivity of $F_j$:
\begin{equation}
    \|F_j(\vomega)-F_j(\vomega')\|^2\leq \ell_i (F_j(\vomega')-F_j(\vomega))^\top(\vomega'-\vomega)\,.
\end{equation}
Thus using $\pi_j = \frac{\ell_j}{\sum_j \ell_j}$, we get
\begin{align*}
      \LL_{t+1}
 &\leq 
  \E[\|\vomega_t - \vomega^*\|_2^2] 
  - (1 - 2 \eta \mu - 5 \eta^2 \bar L^2 - 2(\tfrac{\sigma q}{n} + 10 \eta^2) \bar L^2) \|\vomega_{t+1/2} -\vomega_t\|_2^2  \\
  & \quad - 2\eta \E[F(\vomega_{t+1/2})^\top(\vomega_{t+1/2} - \vomega^*)]  + 2(\tfrac{10\eta^2}{\sigma} + \tfrac{n-q}{n} ) \sigma \E[H_t] \\
  &\quad + 2\bar \ell (\tfrac{\sigma q}{n} + 10 \eta^2)\E[\tfrac{1}{n \pi_j}(F_j(\vomega^*) - F_j(\vomega_{t+1/2}))^\top(\vomega^*-\vomega_{t+1/2})] \\
  &= \E[\|\vomega_t - \vomega^*\|_2^2] 
 - (1 - 2 \eta \mu - 5 \eta^2 \bar L^2 - 2(\tfrac{\sigma q}{n} + 10 \eta^2) \bar L^2) \|\vomega_{t+1/2} -\vomega_t\|_2^2  \\
  & \quad - 2\eta \E[F(\vomega_{t+1/2})^\top(\vomega_{t+1/2} - \vomega^*)]  + 2\bar \ell (\tfrac{\sigma q}{n} + 10 \eta^2)\E[F(\vomega_{t+1/2})^\top(\vomega_{t+1/2}-\vomega^*)] \\
  &\quad + (\tfrac{10\eta^2}{\sigma}+ \tfrac{n-q}{n} ) \sigma \E[H_t]
\end{align*}
where $\bar \ell := \frac{1}{n} \sum_i \ell_i$. Now we can set $\frac{20\eta^2}{\sigma} = \frac{q}{n}$ to get
\begin{align*}
      \LL_{t+1}
 &\leq 
 \E[\|\vomega_t - \vomega^*\|_2^2]
  - (1 - 2 \eta \mu - 65\eta^2 \bar L^2) \|\vomega_{t+1/2} -\vomega_t\|_2^2  \\
  & \quad -\eta( 2 -  60\bar \ell \eta)\E[F(\vomega_{t+1/2})^\top(\vomega_{t+1/2}-\vomega^*)] + (1- \tfrac{q}{2n} ) \sigma \E[H_t]\,.
\end{align*} By using the strong convexity of $F$ and Young's inequality we have that
\begin{equation}
    F(\vomega_{t+1/2})^\top(\vomega_{t+1/2}-\vomega^*) \geq \mu \|\vomega_{t+1/2}-\vomega^*\|^2 
    \geq \tfrac{\mu}{2}\|\vomega_{t}-\vomega^*\|^2 - \mu\|\vomega_{t+1/2}-\vomega_t\|^2 \,.
\end{equation}
Finally with $\eta \leq \frac{1}{40 \bar\ell}$ (note that we always have $\bar \ell \geq \bar L \geq \mu$ because $\ell_i \geq L_i$) we get
\begin{align*}
\LL_{t+1} &\leq 
 \E[\|\vomega_t - \vomega^*\|_2^2] -\frac{\eta\mu}{4}\E[\|\vomega_{t}-\vomega^*\|^2]
  - \frac{89}{100} \E[\|\vomega_{t+1/2} -\vomega_t\|_2^2]   + (1- \tfrac{q}{2n} ) \sigma \E[H_t]\,.
\end{align*}
We finally use the projection-type error bound $\|F_i(\vomega_t)-F_i(\vomega^*)\|^2 \geq \gamma_i^2\|\vomega_t - \vomega^*\|^2$ the same way as~\citep{azizian2019tight} to get,
\begin{align*}
   \|\vomega_{t+1/2} -\vomega_t\|_2^2 
   &= \eta^2 \|\frac{1}{n\pi_i}(F_i(\vomega_t) - \bar \valpha_i^{(t)})\|^2 \\
   & \geq  \frac{\eta^2}{2}\|\frac{1}{n\pi_i}(F_i(\vomega_t) - F_i(\vomega^*))\|^2 -  \eta^2 \| \frac{1}{n\pi_i}(F_i(\vomega^*)-  \bar \valpha_i^{(t)})\|^2 \\
   &\geq  \frac{\gamma_i^2\eta^2}{2}\|\frac{1}{n\pi_i}(\vomega_t - \vomega^*)\|^2 -  \eta^2 \| \frac{1}{n\pi_i}(F_i(\vomega^*)-  \bar \valpha_i^{(t)})\|^2 
   \,.
\end{align*}
Thus we have that, 
\begin{align*}
\LL_{t+1} &\leq (1-\tfrac{\eta \mu}{4})
 \E[\|\vomega_t - \vomega^*\|_2^2]
  - \frac{11 \bar \gamma^2 \eta^2}{25} \E[\|\vomega_{t} -\vomega^*\|_2^2]   + (1- \tfrac{q}{2n} + \frac{9q}{100n} ) \sigma \E[H_t]\,,
\end{align*}
where $\bar \gamma^2 := \frac{1}{n} \sum_{k=1}^n \frac{\gamma_i^2}{n\pi_i}$.
We can thus conclude the proof using the strong convexity of $F$,
\begin{align*}
\LL_{t+1} &\leq \left(1- \min \left\{\Big(\frac{\eta \mu}{4} + \frac{11\eta^2\bar\gamma^2}{25}\Big), \frac{2q}{5n}\right\} \right)\LL_t \,.
\end{align*}

\section{Details on the SVRE--GAN Algorithm}\label{sec:app-algo}
\label{sec:implementation}
\subsection{Practical Aspect}
\label{sub:practical_aspect}

\paragraph{Noise dataset.} \label{par:noise_dataset_}
Variance reduction is usually performed on finite sum dataset. However, the noise dataset in GANs (sampling from the noise variable $z$ for the generator $G$) is in practice considered as an infinite dataset. We considered several ways to cope with this: 
\begin{itemize}
  \item Infinitely taking new samples from a predefined latent distribution $p_g$. In this case, from a theoretical point of view, in terms of using finite sum formulation, there is no convergence guarantee for SVRE even in the strongly convex case. Moreover, the estimators~\eqref{eq:D_svrg_dir} and~\eqref{eq:G_svrg_dir} are biased estimator of the gradient (as $\vmu_D$ and $\vmu_G$ do not estimate the full expectation but a finite sum).
  \item Sampling a different noise dataset at each epoch, i.e. considering a different finite sum at each epoch. In that case, we are performing a variance reduction of this finite sum over the epoch. 
  \item Fix a finite sum noise dataset for the entire training. 
\end{itemize}
In practice, we did not notice any notable difference between the three alternatives.

\paragraph{Adaptive methods.} Particular choices such as the optimization method (\textit{e.g.}  Adam~\citep{kingma2014adam}), learning rates, and normalization, have been established in practice as almost \textit{prerequisite} for convergence\footnote{For instance, \citet{daskalakis2017training,gidel2019variational} plugged Adam into their principled method to get better results.}, in contrast to supervised classification problems where they have been shown to only provide a marginal value~\citep{wilson2017marginal}. 
To our knowledge, SVRE is the only method that works with a constant step size for GANs on non-trivial datasets. This combined with the fact that recent works empirically tune the first moment controlling hyperparameter to $0$ ($\beta_1$, see below) and the variance reduction (VR) one ($\beta_2$, see below) to a non-zero value, sheds light on the reason behind the success of Adam on GANs.

However, combining SVRE with adaptive step size scheme on GANs remains an open problem. We first briefly describe the update rule of Adam, and then we propose a new adaptation of it that is more suitable for VR methods, which we refer to as variance reduced Adam (VRAd).

\paragraph{Adam.}\label{par:adam}
Adam stores an exponentially decaying average of both past gradients $m_t$ and squared gradients $v_t$, for each parameter of the model:
\begin{align}
  m_t = \beta_1 m_{t-1} + (1-\beta_1)g_t \label{eq:adam_beta1} \\ 
  v_t = \beta_2 v_{t-1} + (1-\beta_2)g_t^2\,, \label{eq:adam_beta2}
\end{align}
where $\beta_1, \beta_2 \in [0,1]$, $m_0=0$,  $v_0=0$, and $t=1,\dots T$ denotes the iteration. $m_t$ and $v_t$ are respectively the estimates of the first and the second moments of the stochastic gradient. To compensate the bias toward $0$ due to initialization, \citet{kingma2014adam} propose to use bias-corrected estimates of these first two moments:
\begin{align}
  \hat{m}_t = \frac{m_t}{1-\beta_1^t}\\
  \hat{v}_t = \frac{v_t}{1-\beta_2^t}.
\end{align}
The Adam update rule can be described as:
\begin{equation}
  \vomega_{t+1} = \vomega_t - \eta \frac{\hat{\vm}_t}{\sqrt{\hat{\vv}_t}+\epsilon} .
\end{equation}

Adam can be understood as an approximate gradient method with a diagonal step size of $\eta_{Adam}:=  \frac{\eta}{\sqrt{\vv_t}+\epsilon}$. Since VR methods aim to provide a vanishing $\vv_t$, they lead to a too large step-size $\eta_{Adam}$ of $\frac{\eta}{\eps}$. This could indicate that the update rule of Adam may not be a well-suited method to combine with VR methods. 

\paragraph{VRAd.}
This motivates the introduction of a new Adam-inspired variant of adaptive step sizes that maintain a reasonable size even when $\vv_t$ vanishes,
\begin{equation}\label{eq:VRAd}
  \vomega_{t+1} = \vomega_t - \eta \frac{ |\hat{\vm}_t|}{\sqrt{\hat{\vv}_t}+\epsilon} \hat{\vm}_t \tag{VRAd} \,.
\end{equation}
This adaptive variant of Adam is motivated by the step size $\eta^* = \eta \frac{\vm_t^2}{\vv_t}$ derived by~\citet{schaul2013no}. \eqref{eq:VRAd} is simply the square-root of $\eta^*$ in order to stick with Adam's scaling of $\vv_t$.

\subsection{SVRE-GAN}
\label{sub:extrasvrg_gan} 

\begin{algorithm}[t]
	\caption{Pseudocode for SVRE-GAN.}
	\label{alg:svrg_gan}
	\begin{algorithmic}[1]
		\STATE {\bfseries Input:} 
		dataset $\mathcal{D}$,
		noise dataset $\mathcal{Z}$ ($|\mathcal{Z}|=|\mathcal{D}|=n$), stopping iteration $T$,
		learning rates $\eta_D, \eta_G$,
		generator loss $ \LL^{G}$,
		discriminator loss $ \LL^{D}$, mini-batch size B.
		\STATE {\bfseries Initialize:} $D$, 
		$G$
				
		\FOR{$e=0$ {\bfseries to} $T{-}1$}
		\STATE $D^{\mathcal{S}} = D \,$ and $\,\vmu_{D} = \frac{1}{n} \sum_{i=1}^{n} \sum_{j=1}^n \nabla_{D}  \LL^{D}(G^{\mathcal{S}},D^{\mathcal{S}},\mathcal D_j, \mathcal Z_i)$    \label{line:mu_D}
		\STATE $G^{\mathcal{S}} = G\,$ and $\,\vmu_{G}= \frac{1}{n} \sum_{i=1}^{n} \nabla_{G}  \LL^{G}(  G^{\mathcal{S}},D^{\mathcal{S}}, \mathcal Z_i)$
		\STATE $N \sim \geom\big(B/n\big)$  \hfill (length of the epoch)   
		\FOR{$i=0$ {\bfseries to} $N{-}1$ } \label{alg:stochastic_gan}
		\STATE \textbf{Sample} mini-batches $(n_d,n_z)$; do \textbf{extrapolation:}
		\STATE $\tilde D = D - \eta_D \vd_{D}(G,D,G^{\mathcal{S}},D^{\mathcal{S}},n_z)$   
		\hfill $\triangleright$~\eqref{eq:D_svrg_dir}
		\STATE $\tilde G = G - \eta_G \vd_{G}(G,D,G^{\mathcal{S}},D^{\mathcal{S}},n_d,n_z)$ 
		\hfill $\triangleright$~\eqref{eq:G_svrg_dir}
		\STATE \textbf{Sample} new mini-batches $(n_d,n_z)$; do \textbf{update:}
		\STATE $D = D - \eta_D \vd_{D}(\tilde G,\tilde D,G^{\mathcal{S}},D^{\mathcal{S}},n_z)$      
		\hfill $\triangleright$~\eqref{eq:D_svrg_dir}
		\STATE $G = G - \eta_G \vd_{G}(\tilde G,\tilde D,G^{\mathcal{S}},D^{\mathcal{S}},n_d,n_z)$  
		\hfill $\triangleright$~\eqref{eq:G_svrg_dir}
		\ENDFOR
		\ENDFOR   
		\STATE {\bfseries Output:} $G, D$   
	\end{algorithmic}
\end{algorithm}

In order to cope with the issues introduced by the stochastic game formulation of the GAN models, we proposed the SVRE algorithm Alg.~\ref{alg:svre} which combines SVRG and extragradient method. We refer to the method of applying SVRE to train GANs as the \textit{SVRE-GAN} method, and we describe it in detail in Alg.~\ref{alg:svrg_gan} (generalizing it with mini-batching, but using uniform probabilities). Assuming that we have $\mathcal{D}[n_d]$ and $\mathcal{Z}[n_z]$, respectively two mini-batches of size $B$ of the true dataset and the noise dataset, we compute $\nabla_{D} \LL^D(G, D,\mathcal{D}[n_d], \mathcal{Z}[n_z])$ and $ \nabla_{G} \LL^G(G, D, \mathcal{Z}[n_z])$ the respective mini-batches gradient of the discriminator and the generator:
\begin{align}
  &\nabla_{D} \LL^D(G, D,\mathcal{D}[n_d], \mathcal{Z}[n_z]) := \frac{1}{|n_z|} \frac{1}{|n_d|} \sum_{i \in n_z} \sum_{j \in n_d}  \nabla_D \LL^D(G, D,\mathcal{D}_j, \mathcal{Z}_i) \\
 &\nabla_{G} \LL^G(G, D, \mathcal{Z}[n_z]) := \frac{1}{|n_z|} \sum_{i \in n_z} \nabla_{G} \LL^G(G, D, \mathcal{Z}_i) \,,
\end{align}
where $\mathcal{Z}_i$ and $\mathcal{D}_j$ are respectively the $i^{th}$ example of the noise dataset and the $j^{th}$ of the true dataset. Note that $n_z$ and $n_d$ are lists and thus that we allow repetitions in the summations over $n_z$ and $n_d$.  
The variance reduced gradient of the SVRG method are thus given by: 
\begin{align}
 \vd_{D}(G,D,G^{\mathcal{S}},D^{\mathcal{S}})
 & := \vmu_{D} +
    \nabla_{D} \LL^D(G, D,  
        \mathcal{D}[n_d], \mathcal{Z}[n_z]) - \nabla_{D} \LL^D(G^{\mathcal{S}},  
 D^{\mathcal{S}},\mathcal{D}[n_d], \mathcal{Z}[n_z]) \!\!  \label{eq:D_svrg_dir}\\
 \vd_{G}(G,D,G^{\mathcal{S}},D^{\mathcal{S}})
  &:= \vmu_{G} + \nabla_{G} \LL^G(G, D, \mathcal{Z}[n_z]) -\nabla_{} \LL^G(G^{\mathcal{S}}, D^{\mathcal{S}},\mathcal{Z}[n_z]) \label{eq:G_svrg_dir} \,,
\end{align}
where $G^{\mathcal{S}}$ and $D^{\mathcal{S}}$ are the snapshots and $\vmu_{D}$ and $\vmu_{G}$ their respective gradients.

Alg.~\ref{alg:svrg_gan} summarizes the SVRG optimization extended to GAN.
To obtain that $ \mathop{\mathbb{E}} \big[  
\nabla_{\boldsymbol{\Theta^{\mathcal{S}}}} \LL(\vtheta^{\mathcal{S}},  \vphi^{\mathcal{S}}, \cdot) - \boldsymbol{\mu} \big]$ vanishes, when updating $\boldsymbol{\theta}$ and $\boldsymbol{\varphi}$ where the expectation is over samples of $\mathcal{D}$ and $\mathcal{Z}$ respectively, we use the snapshot networks $\vtheta^{\mathcal{S}}$ and $\vphi^{\mathcal{S}}$ for the second term in lines $9,10,12$ and $13$.
Moreover, the noise dataset  $\mathcal{Z} \sim p_z$, where $|\mathcal{Z}| = |\mathcal{D}| =n$, is fixed. 
Empirically we observe that directly sampling from $p_z$ (contrary to fixing the noise dataset and re-sampling it with frequency $m$) does not impact the performance, as $|\mathcal{Z}|$ is usually high. 

Note that the double sum in Line~$4$ can be written as two sums because of the separability of the expectations in typical GAN objectives. Thus the time complexity for calculating $\mu^D$ is still $O(n)$ and not $O(n^2)$ which would be prohibitively expensive.

\section{Restarted SVRE} 
\label{sec:restarted_svre}
\begin{algorithm}[t]
	\caption{Pseudocode for Restarted SVRE.}
	\label{alg:svre_restart}
	\begin{algorithmic}[1]
		\STATE {\bfseries Input:} 
		Stopping time $T$,
		learning rates $\eta_\vtheta, \eta_\vphi$,
		both players' losses $ \LL^{G}$ and $ \LL^{D}$,
		probability of restart $p$.
		\STATE {\bfseries Initialize:} $\vphi$, $\vtheta$, $t=0$ \hfill $\triangleright$ $t$ is for the online average computation.
		\FOR{$e=0$ {\bfseries to} $T{-}1$}
		
		\STATE Draw $\texttt{restart} \sim \mathrm{B}(p)$. \hfill $\triangleright$ Check if we restart the algorithm.
		\IF{\texttt{restart} \textbf{and} $e >0$}
		\STATE $\vphi \leftarrow \bar \vphi$, \; $\vtheta \leftarrow \bar \vtheta$ and $t=1$ 
		\ENDIF 
		\STATE $\vphi^{\mathcal{S}} \leftarrow \vphi \,$ and $\,\vmu_{\vphi}^\mathcal{S} \leftarrow \frac{1}{|\mathcal{Z}|} \sum_{i=1}^{n} \nabla_{\vphi}  \LL^{D}_i(\vtheta^{\mathcal{S}},\vphi^{\mathcal{S}})$ 
		\STATE $\vtheta^{\mathcal{S}} \leftarrow \vtheta\,$ and $\,\vmu_{\vtheta}^\mathcal{S} \leftarrow \frac{1}{|\mathcal{\vphi}|} \sum_{i=1}^{n} \nabla_{\vtheta}  \LL^{G}_i(  \vtheta^{\mathcal{S}},\vphi^{\mathcal{S}})$
		\STATE $N \sim \geom\big(1/n\big)$  \hfill $\triangleright$ Length of the epoch.
		\FOR{$i =0$ {\bfseries to} $N{-}1$ } \label{alg:stochastic_gan_restart}
		\STATE \textbf{Sample} $i_\vtheta \sim \pi_\vtheta$, $i_\vphi \sim \pi_\vphi$, do \textbf{extrapolation:}
		\STATE $\tilde \vphi \leftarrow \vphi - \eta_\vtheta \vd_{\vphi}(\vtheta,\vphi,\vtheta^{\mathcal{S}},\vphi^{\mathcal{S}})$   
		\,,\; $\tilde \vtheta \leftarrow \vtheta - \eta_\vphi \vd_{\vtheta}(\vtheta,\vphi,\vtheta^{\mathcal{S}},\vphi^{\mathcal{S}})$ 
		\hfill $\triangleright$~\eqref{eq:D_svrg_dir} and~\eqref{eq:G_svrg_dir}
		\STATE\textbf{Sample} $i_\vtheta \sim \pi_\vtheta$, $i_\vphi \sim \pi_\vphi$, do \textbf{update:}
		\STATE $\vphi \leftarrow \vphi - \eta_\vtheta \vd_{\vphi}(\tilde \vtheta,\tilde \vphi,\vtheta^{\mathcal{S}},\vphi^{\mathcal{S}})$      
		\,,\; $\vtheta \leftarrow \vtheta - \eta_\vphi \vd_{\vtheta}(\tilde \vtheta,\tilde \vphi,\vtheta^{\mathcal{S}},\vphi^{\mathcal{S}})$  
		\hfill $\triangleright$~\eqref{eq:D_svrg_dir} and~\eqref{eq:G_svrg_dir} \\[1mm]
		\STATE $\bar \vtheta \leftarrow \frac{t}{t+1}\bar \vtheta + \frac{1}{t+1} \vtheta$ and $\bar \vphi \leftarrow \frac{t}{t+1}\bar \vphi + \frac{1}{t+1} \vphi $ 
		\hfill $\triangleright$ Online computation of the average.
		\STATE $t \leftarrow t+1$  \hfill $\triangleright$ Increment $t$ for the online average computation.
		\ENDFOR
		\ENDFOR   
		\STATE {\bfseries Output:} $\vtheta, \vphi$   
	\end{algorithmic}
\end{algorithm}

Alg.~\ref{alg:svre_restart} describes the restarted version of SVRE presented in \S~\ref{sub:motivating_example}.
With a probability $p$ (fixed) before the computation of $\vmu_{\vphi}^\mathcal{S} $ and $\vmu_{\vtheta}^\mathcal{S} $, we decide whether to restart SVRE (by using the averaged iterate as the new starting point--Alg.~\ref{alg:svre_restart}, Line $6$--$\bar \vomega_t$) or computing the batch snapshot at a point $\vomega_t$.

\section{Details on the  implementation}\label{sec:impl-details}
For our experiments, we used the PyTorch\footnote{\url{https://pytorch.org/}} deep learning framework, whereas for computing the FID and IS metrics, we used the provided implementations in Tensorflow\footnote{\url{https://www.tensorflow.org/}}.
\subsection{Metrics}\label{sec:app-metrics}
We provide more details about the metrics enumerated in \S~\ref{sec:experiments}.
Both FID and IS use:
\begin{enumerate*}[series = tobecont, itemjoin = \quad, label=(\roman*)]
\item the \textit{Inception v3 network}~\citep{inceptionmodel} that has been trained on the ImageNet dataset consisting of ${\sim}1$ million RGB images of $1000$ classes, $C=1000$. 
\item a sample of $m$ generated images $ x \sim p_g$, where usually $m=50000$.
\end{enumerate*}

\subsubsection{Inception Score}\label{sec:is}
Given an image $x$, IS uses the softmax output of the Inception network $p(y|x)$ which represents the probability that $x$ is of class $c_i, i \in 1 \dots C$, i.e., $p(y|x) \in [0,1]^C$.
It then computes the marginal class distribution $p(y)=\int_x p(y|x)p_g(x)$.
IS measures the Kullback--Leibler divergence $\mathbb{D}_{KL}$ between the predicted conditional label distribution $p(y|x)$  and the marginal class distribution $p(y)$. 
More precisely, it is computed as follows:
\begin{align}
IS(G) = \exp\big( \E_{x \sim p_g} [ \mathbb{D}_{KL}( p(y|x) || p(y) )  ] \big)
= \exp\big(   
\frac{1}{m} \sum_{i=1}^m \sum_{c=1}^{C} p(y_c|x_i) \log{\frac{p(y_c|x_i)}{p(y_c)}}
\big).
\end{align}

It aims at estimating 
\begin{enumerate*}[series = tobecont, itemjoin = \quad, label=(\roman*)]
\item if the samples look realistic i.e., $p(y|x)$ should have low entropy, and 
\item if the samples are diverse (from different ImageNet classes) i.e., $p(y)$ should have high entropy.
\end{enumerate*}
As these are combined using the Kullback--Leibler divergence, the higher the score is, the better the performance. Note that the range of IS scores at convergence varies across datasets, as the Inception network is pretrained on the ImageNet classes. For example, we obtain low IS values on the SVHN dataset as a large fraction of classes are numbers, which typically do not appear in the ImageNet dataset. Since \textbf{MNIST} has greyscale images, we used a classifier trained on this dataset and used  $m=5000$.  For the rest of the datasets, we used the original implementation\footnote{\url{https://github.com/openai/improved-gan/}} of IS in TensorFlow, and $m=50000$. 

\subsubsection{Fr\'echet Inception Distance}\label{sec:fid}
Contrary to IS, FID aims at comparing the synthetic samples $x \sim p_g$ with those of the training dataset $ x \sim p_d$ in a feature space. The samples are embedded using the first several layers of the Inception network. Assuming  $p_g$ and $p_d$ are multivariate normal distributions, it then estimates the means $\vm_g$ and $\vm_d$ and covariances $C_g$ and $C_d$, respectively for  $p_g$ and $p_d$ in that feature space. Finally, FID is computed as: 
\begin{align}
\mathbb{D}_{\text{FID}}(p_d, p_g) \approx d^2((\vm_d, C_d), (\vm_g, C_g )) =  ||\vm_d - \vm_g||_2^2 + Tr(C_d + C_g - 2(C_dC_g)^{\frac{1}{2}}), 
\end{align}
where $d^2$ denotes the Fr\'echet Distance.
Note that as this metric is a distance, the lower it is, the better the performance.
We used the original implementation of FID\footnote{\url{https://github.com/bioinf-jku/TTUR}} in Tensorflow, along with the provided statistics of the datasets.

\subsubsection{Second Moment Estimate}\label{sec:sme}
To evaluate SVRE effectively, we used the \textbf{second moment estimate} (SME, uncentered variance, see \S~\ref{par:adam}) of the gradient estimate throughout the iterations $t=1 \dots T$ per parameter, computed as:
$v_t = \gamma v_{t-1} + (1-\gamma) g_t^2$, where $g_t$ denotes the gradient estimate for the parameter and iteration $t$, and $\gamma = 0.9$. For SVRE, $g_t$ is $d_{\vphi}$ and $d_{\vtheta}$ (see  Eq.~\ref{eq:D_svrg_dir} and~\ref{eq:G_svrg_dir}) for $G$ and $D$, respectively.
We initialize $g_0=0$  and we use bias-corrected estimates: $\hat{v} = \frac{v_t}{1-\gamma^t}$.
As the second moment estimate is computed per each parameter of the model, we depict the average of these values for the parameters of $G$ and $D$ separately.

In this work, as we aim at assessing if SVRE \emph{effectively} reduces the variance of the gradient updates, we use SME in our analysis as it is computationally inexpensive and fast to compute.

\subsubsection{Entropy \& Total Variation on MNIST}\label{sec:tve_mnist}
For the experiments on \textbf{MNIST} illustrated in Fig.~\ref{subfig-mnist_is_param_updates} \&~\ref{subfig-var_g} in \S~\ref{sec:experiments}, we plot in \S~\ref{sec:additional-experiments} the \textbf{entropy} (E) of the generated samples' class distribution, as well as the \textbf{total variation} (TV) between the class distribution of the generated samples and a uniform one (both computed using a pretrained network that classifies its $10$ classes).

\subsection{Architectures \& Hyperparameters}\label{app:arch}

\paragraph{Description of the architectures.}
We  describe the models we used in the empirical evaluation of  SVRE by listing the layers they consist of, as adopted in GAN works, e.g. ~\citep{miyato2018spectral}.
With ``conv.'' we denote a convolutional layer and ``transposed conv'' a transposed convolution layer~\citep{radford2016unsupervised}.
The models use Batch Normalization~\citep{bnorm} and Spectral Normalization layers~\citep{miyato2018spectral}.

\subsubsection{Architectures for experiments on MNIST}\label{app:mnist_arch}
For experiments on the \textbf{MNIST} dataset, we used the DCGAN architectures~\citep{radford2016unsupervised}, listed in Table~\ref{tab:mnist_arch}, and the parameters of the models are initialized using PyTorch default initialization. 
We used mini-batch sizes of $50$ samples, whereas for full dataset passes we used mini-batches of $500$ samples as this reduces the wall-clock time for its computation.
For experiments on this dataset, we used the \textit{non saturating} GAN loss as proposed~\citep{goodfellow2014generative}:
\begin{align}
& \LL_D =   \E_{x\sim p_d}  \log(D(x)) + \E_{z\sim p_z}   \log(D(G(z))) \label{eq:vanilla_loss_d}\\
& \LL_G =  \E_{z\sim p_z}   \log(D(G(z))), \label{eq:nonsat_loss_g}
\end{align}
where $p_d$ and $p_z$ denote the data and the latent distributions (the latter to be predefined). 

For both the baseline and the SVRE variants we tried the following step sizes $\eta=[1\times10^{-2},$ $1\times10^{-3}, 1\times10^{-4}]$. 
We observe that SVRE can be used with larger step sizes.
In Table~\ref{tab:res_summary_shallow}, we used $\eta=1\times10^{-4}$ and $\eta=1\times10^{-2}$ for SE--A and SVRE(--VRAd), respectively.

\begin{table}\centering
\begin{minipage}[b]{0.49\hsize}\centering
\begin{tabular}{@{}c@{}}\toprule
\textbf{Generator}\\\toprule
\textit{Input:} $z \in \mathds{R}^{128} \sim \mathcal{N}(0, I) $ \\  \hdashline 
transposed conv. (ker: $3{\times}3$, $128 \rightarrow 512$; stride: $1$) \\
 Batch Normalization \\ 
 ReLU  \\
 transposed conv. (ker: $4{\times}4$, $512 \rightarrow 256$, stride: $2$)\\
 Batch Normalization \\ 
 ReLU  \\
 transposed conv. (ker: $4{\times}4$, $256 \rightarrow 128$, stride: $2$) \\
 Batch Normalization \\ 
 ReLU  \\
 transposed conv. (ker: $4{\times}4$, $128 \rightarrow 1$, stride: $2$, pad: 1) \\
 $Tanh(\cdot)$\\ 
\bottomrule 
\end{tabular}
\end{minipage}
\hfill 
\begin{minipage}[b]{0.49\hsize}\centering
\begin{tabular}{@{}c@{}}\toprule
\textbf{Discriminator}\\\toprule
\textit{Input:} $x \in \mathds{R}^{1{\times}28{\times}28} $ \\  \hdashline 
 conv. (ker: $4{\times}4$, $1 \rightarrow 64$; stride: $2$; pad:1) \\
 LeakyReLU (negative slope: $0.2$) \\ 
 conv. (ker: $4{\times}4$, $64 \rightarrow 128$; stride: $2$; pad:1) \\
 Batch Normalization \\
 LeakyReLU (negative slope: $0.2$) \\
 conv. (ker: $4{\times}4$, $128 \rightarrow 256$; stride: $2$; pad:1) \\
 Batch Normalization \\
 LeakyReLU (negative slope: $0.2$) \\
 conv. (ker: $3{\times}3$, $256 \rightarrow 1$; stride: $1$) \\
 $Sigmoid(\cdot)$ \\
\bottomrule
\end{tabular}
\end{minipage}
\caption{DCGAN architectures~\citep{radford2016unsupervised} used for experiments on \textbf{MNIST}.
We use \textit{ker} and \textit{pad} to denote \textit{kernel} and \textit{padding} for the (transposed) convolution layers, respectively. 
With $h{\times}w$ we denote the kernel size.
With $ c_{in} \rightarrow y_{out}$ we denote the number of channels of the input and output, for (transposed) convolution layers.
}\label{tab:mnist_arch}
\end{table}

\subsubsection{Choice of architectures on real-world datasets}\label{sec:arch_motivation}
We replicate the experimental setup described for \textbf{CIFAR-10} and \textbf{SVHN} in~\citep{miyato2018spectral}, described also below in \S~\ref{sec:deeper_resnet_arch}.
We observe that this experimental setup is highly sensitive to the choice of the hyperparameters (see our results in \S~\ref{sec:results_deep_arch}), making it more difficult to compare the optimization methods for a fixed hyperparameter choice. In particular, apart from the different combinations of learning rates for $G$ and $D$, for the baseline this also included experimenting with: $\beta_1$ (see \eqref{eq:adam_beta1}), a multiplicative factor of exponential learning rate decay scheduling $\gamma$, as well as different ratio of updating  $G$ and $D$ per iteration.
These observations, combined with that we had limited computational resources, motivated us to use shallower architectures, which we describe below in \S~\ref{sec:shallow_sagan}, and which use an inductive bias of so-called Self--Attention layers~\citep{sagan}.
As a reference, our SAGAN and ResNet architectures for \textbf{CIFAR-10} have approximately $35$ and $85$ layers, respectively--in total for G and D, including the non linearity and the normalization layers.
For clarity, although the deeper and the shallower architectures differ as they are based on ResNet and SAGAN, we refer these as \textit{deep} (see \S~\ref{sec:shallow_sagan}) and \textit{shallow} (see \S~\ref{sec:deeper_resnet_arch}), respectively.

\subsubsection{Shallower SAGAN architectures}\label{sec:shallow_sagan}
We used the SAGAN architectures~\citep{sagan}, as the techniques of self-attention introduced in SAGAN were used to obtain the state-of-art GAN results on ImageNet~\citep{brock2018large}. 
In summary, these architectures: 
\begin{enumerate*}[series = tobecont, itemjoin = \quad, label=(\roman*)]
\item allow for attention-driven, long-range dependency modeling, 
\item use spectral normalization~\citep{miyato2018spectral} on both $G$ and $D$ (efficiently computed with the \textit{power iteration} method); and 
\item use different learning rates for $G$ and $D$, as advocated in~\citep{heusel_gans_2017}.
\end{enumerate*}
The foremost is obtained by combining weights, or alternatively \textit{attention vectors}, with the convolutions across layers, so as to allow modeling textures that are consistent globally--for the generator, or enforcing geometric constraints on the global image structure--for the discriminator. 

We used the architectures listed in Table~\ref{tab:sagan_shallow_arch} for \textbf{CIFAR-10} and \textbf{SVHN} datasets, and the architectures described in Table~\ref{tab:sagan_shallow_arch_imgnet}  for the experiments on \textbf{ImageNet}. The models' parameters are initialized using the default initialization of PyTorch.

For experiments with SAGAN, we used the hinge version of the adversarial non-saturating loss ~\citep{lim2017geometricGan,sagan}:
\begin{align}
& \LL_D =   \E_{x\sim p_d}  \max(0, 1-D(x)) + \E_{z\sim p_z}   \max(0, 1+D(G(z)) \label{eq:hinge_loss_d}\\
& \LL_G =  - \E_{z\sim p_z}   D(G(z)). \label{eq:hinge_loss_g},
\end{align}
where consistent with the notation above, $p_d$ and $p_z$ denote the data and the latent distributions. 

\begin{table} \centering
\ras{1.3}
\begin{tabular}{@{}ccc@{}}\toprule
	\multicolumn{3}{c}{\textbf{Self--Attention Block ($d$ -- input depth) }}\\\toprule
	\multicolumn{3}{c}{\textit{Input:} $t \in \mathds{R}^{d \times H \times W} $} \\  \hdashline 
	\multicolumn{1}{c:}{\textit{i:} conv. (ker: $1{\times}1$, $d \rightarrow \lfloor d/8 \rfloor $)}&
	\multicolumn{1}{c:}{\textit{ii:} conv. (ker: $1{\times}1$, $d \rightarrow \lfloor d/8 \rfloor $)}&
	\textit{iii:} conv. (ker: $1{\times}1$, $d \rightarrow d $) \\
	\multicolumn{2}{c:}{\textit{iv:} softmax( \textit{(i)} $\otimes$  \textit{(ii)} )} & \\\hdashline
	\multicolumn{3}{c}{\textit{Output:} $\gamma \big( \textit{(iv)} \otimes \textit{(iii)} \big) + t$}\\
\bottomrule \\\end{tabular}
\caption{Layers of the self--attention block used in the SAGAN architectures (see Tables~\ref{tab:sagan_shallow_arch} and \ref{tab:sagan_shallow_arch_imgnet}), where $\otimes$ denotes matrix multiplication and $\gamma$ is a scale parameter initialized with $0$.
The columns emphasize that the execution is in parallel, more precisely, that the block input $t$ is input to the convolutional layers \textit{(i)}--\textit{(iii)}.
The shown row ordering corresponds to consecutive layers' order, \textit{e.g.} softmax is done on the product of the outputs of the \textit{(i)} and \textit{(ii)} convolutional layers.
The $1\times1$ convolutional layers have stride of $1$.
For complete details see~\cite{sagan}.}\label{tab:att_layer}
\end{table}

\begin{table}\centering
\begin{minipage}[b]{.449\hsize}\centering  \begin{tabular}{@{}c@{}}\toprule
\textbf{Generator}\\\toprule
\textit{Input:} $z \in \mathds{R}^{128} \sim \mathcal{N}(0, I) $ \\  \hline  
  transposed conv. (ker: $4{\times}4$, $128 \rightarrow 256$; stride: $1$) \\
 Spectral Normalization \\ 
 Batch Normalization \\ 
 ReLU \\ \hdashline 
  transposed conv. (ker: $4{\times}4$, $256 \rightarrow 128$, stride: $2$, pad: $1$)\\
 Spectral Normalization \\ 
 Batch Normalization \\ 
 ReLU  \\  \hdashline 
  Self--Attention Block ($128$)\\ \hdashline 
  transposed conv. (ker: $4{\times}4$, $128 \rightarrow 64$, stride: $2$, pad: $1$) \\
 Spectral Normalization \\ 
 Batch Normalization \\ 
 ReLU  \\ \hdashline 
   Self--Attention Block ($64$)\\ \hdashline 
  transposed conv. (ker: $4{\times}4$, $64 \rightarrow 3$, stride: $2$, pad: $1$) \\
 $Tanh(\cdot)$\\ 
\bottomrule \\\end{tabular}
\end{minipage}
\hfill \begin{minipage}[b]{.449\hsize}\centering  \begin{tabular}{@{}c@{}}\toprule
\textbf{Discriminator}\\\toprule
\textit{Input:} $x \in \mathds{R}^{3{\times}32{\times}32} $ \\  \hline  
  conv. (ker: $4{\times}4$, $3 \rightarrow 64$; stride: $2$; pad: $1$) \\
 Spectral Normalization \\
 LeakyReLU (negative slope: $0.1$) \\  \hdashline 
   conv. (ker: $4{\times}4$, $64 \rightarrow 128$; stride: $2$; pad: $1$) \\
 Spectral Normalization \\
 LeakyReLU (negative slope: $0.1$) \\  \hdashline 
    conv. (ker: $4{\times}4$, $128 \rightarrow 256$; stride: $2$; pad: $1$) \\
 Spectral Normalization \\
 LeakyReLU (negative slope: $0.1$) \\  \hdashline 
     Self--Attention Block ($256$)\\ \hdashline 
      conv. (ker: $4{\times}4$, $256 \rightarrow 1$; stride: $1$) \\
\bottomrule
\end{tabular}
\end{minipage}
\caption{\textit{Shallow} SAGAN architectures for experiments on \textbf{SVHN} and \textbf{CIFAR-10}, for the Generator (left) and the Discriminator (right).
The self-attention block is described in Table~\ref{tab:att_layer}.
We use the default PyTorch hyperparameters for the Batch Normalization layer.}\label{tab:sagan_shallow_arch}
\end{table}

\begin{table}\centering
\begin{minipage}[b]{.449\hsize}\centering  \begin{tabular}{@{}c@{}}\toprule
\textbf{Generator}\\\toprule
\textit{Input:} $z \in \mathds{R}^{128} \sim \mathcal{N}(0, I) $ \\  \hline 
  transposed conv. (ker: $4{\times}4$, $128 \rightarrow 512$; stride: $1$) \\
 Spectral Normalization \\ 
 Batch Normalization \\ 
 ReLU \\ \hdashline 
  transposed conv. (ker: $4{\times}4$, $512 \rightarrow 256$, stride: $2$, pad: $1$)\\
 Spectral Normalization \\ 
 Batch Normalization \\ 
 ReLU  \\  \hdashline 
  transposed conv. (ker: $4{\times}4$, $256 \rightarrow 128$, stride: $2$, pad: $1$) \\
 Spectral Normalization \\ 
 Batch Normalization \\ 
 ReLU  \\ \hdashline 
   Self--Attention Block ($128$)\\ \hdashline 
   transposed conv. (ker: $4{\times}4$, $128 \rightarrow 64$, stride: $2$, pad: $1$) \\
 Spectral Normalization \\ 
 Batch Normalization \\ 
 ReLU  \\ \hdashline 
   Self--Attention Block ($64$)\\ \hdashline  
  transposed conv. (ker: $4{\times}4$, $64 \rightarrow 3$, stride: $2$, pad: $1$) \\
 $Tanh(\cdot)$\\ 
\bottomrule \\\end{tabular}
\end{minipage}
\hfill
\begin{minipage}[b]{.449\hsize}\centering  \begin{tabular}{@{}c@{}}\toprule
\textbf{Discriminator}\\\toprule
\textit{Input:} $x \in \mathds{R}^{3{\times}64{\times}64} $ \\  \hline 
      conv. (ker: $4{\times}4$, $3 \rightarrow 64$; stride: $2$; pad: $1$) \\
   Spectral Normalization \\
   LeakyReLU (negative slope: $0.1$) \\  \hdashline 
      conv. (ker: $4{\times}4$, $64 \rightarrow 128$; stride: $2$; pad: $1$) \\
   Spectral Normalization \\
   LeakyReLU (negative slope: $0.1$) \\  \hdashline 
      conv. (ker: $4{\times}4$, $128 \rightarrow 256$; stride: $2$; pad: $1$) \\
   Spectral Normalization \\
   LeakyReLU (negative slope: $0.1$) \\  \hdashline 
      Self--Attention Block ($256$)\\ \hdashline 
      conv. (ker: $4{\times}4$, $256 \rightarrow 512$; stride: $2$; pad: $1$) \\
   Spectral Normalization \\
   LeakyReLU (negative slope: $0.1$) \\  \hdashline 
      Self--Attention Block ($512$)\\ \hdashline 
      conv. (ker: $4{\times}4$, $512 \rightarrow 1$; stride: $1$) \\
\bottomrule
\end{tabular}
\end{minipage}
\caption{\textit{Shallow} SAGAN architectures for experiments on \textbf{ImageNet}, for the Generator (left) and the Discriminator (right).
The self--attention block is described in Table~\ref{tab:att_layer}.
Relative to the architectures used for \textbf{SVHN} and \textbf{CIFAR-10} (see Table~\ref{tab:sagan_shallow_arch}), the generator has  one additional ``common'' block (conv.--norm.--ReLU), whereas the discriminator has additional ``common'' block as well as self--attention block (both of more parameters).}\label{tab:sagan_shallow_arch_imgnet}
\end{table}

For the SE--A baseline we obtained best performances when $\eta_G=1\times 10^{-4}$ and $\eta_D=4\times 10^{-4}$, for G and D, respectively. Similarly as noted for \textbf{MNIST}, using SVRE allows for using larger order of the step size on the rest of the datasets, whereas SE--A with increased step size ($\eta_G=1\times 10^{-3}$ and $\eta_D=4\times 10^{-3}$ failed to converge. In Table~\ref{tab:res_summary}, $\eta_G=1\times 10^{-3}$, $\eta_D=4\times 10^{-3}$, and $\eta_G=5\times 10^{-3}$, $\eta_D=8\times 10^{-3}, \beta_1=0.3$ for  SVRE and SVRE--VRAd, respectively. We did not use momentum for the vanilla SVRE experiments.

\subsubsection{Deeper ResNet architectures}\label{sec:deeper_resnet_arch}
We experimented with ResNet~\citep{resnet} architectures on \textbf{CIFAR-10} and \textbf{SVHN}, using the architectures listed in Table~\ref{tab:resnet_arch}, that replicate the setup described in \citep{miyato2018spectral} on \textbf{CIFAR-10}.
For experiments with ResNet, we used the hinge version of the adversarial non-saturating loss, Eq.~\ref{eq:hinge_loss_d} and~\ref{eq:hinge_loss_g}.
For this architectures, we refer the reader to \S~\ref{sec:results_deep_arch} for details on the hyperparameters, where we list the hyperparameters along with the obtained results.

\begin{table}\centering		 \begin{minipage}[b]{0.43\hsize}\centering  \begin{tabular}{@{}c@{}}\toprule
\textbf{G--ResBlock}\\\toprule
	\multicolumn{1}{l}{\textit{Bypass}:} \\
	Upsample($\times 2$) \\  \hdashline 
		\multicolumn{1}{l}{\textit{Feedforward}:} \\
	Batch Normalization \\ 
	ReLU \\
	Upsample($\times 2$) \\
	conv. (ker: $3{\times}3$, $256 \rightarrow 256 $; stride: $1$; pad: $1$) \\
    Batch Normalization \\ 
	ReLU \\
	conv. (ker: $3{\times}3$, $256 \rightarrow 256 $; stride: $1$; pad: $1$) \\
\bottomrule
\end{tabular}
\end{minipage} \hfill
\begin{minipage}[b]{0.559\hsize}\centering  \begin{tabular}{@{}c@{}}\toprule
\textbf{D--ResBlock ($\ell$--th block)}\\\toprule
	\multicolumn{1}{l}{\textit{Bypass}:} \\
    $[$AvgPool (ker:$2{\times}2$ )$]$, if $\ell = 1$ \\
	conv. (ker: $1{\times}1$, $3_{\ell=1} / 128_{\ell \neq 1} \rightarrow 128 $; stride: $1$) \\
	Spectral Normalization \\
	$[$AvgPool (ker:$2{\times}2$, stride:$2$)$]$, if $\ell \neq 1$ \\ \hdashline 
		\multicolumn{1}{l}{\textit{Feedforward}:} \\
	$[$ ReLU $]$, if $\ell \neq 1$ \\
	conv. (ker: $3{\times}3$, $3_{\ell=1} / 128_{\ell \neq 1}  \rightarrow 128 $; stride: $1$; pad: $1$) \\
	Spectral Normalization \\
	ReLU \\
	conv. (ker: $3{\times}3$, $128 \rightarrow 128 $; stride: $1$; pad: $1$) \\
	Spectral Normalization \\
	AvgPool (ker:$2{\times}2$ )\\
\bottomrule
\end{tabular}
\end{minipage}
\caption{ResNet blocks used for the ResNet architectures (see Table~\ref{tab:resnet_arch}), for the Generator (left) and the Discriminator (right). Each ResNet block contains skip connection (bypass), and a sequence of convolutional layers, normalization, and the ReLU non--linearity. 
The skip connection of the ResNet blocks for the Generator (left) upsamples the input using a factor of $2$ (we use the default PyTorch upsampling algorithm--nearest neighbor), whose output is then added to the one obtained from the ResNet block listed above.  For clarity we list the layers sequentially, however, note that the bypass layers operate in parallel with the layers denoted as ``feedforward''~\citep{resnet}. The ResNet block for the Discriminator (right) differs if it is the first block in the network (following the input to the Discriminator), $\ell = 1$, or a subsequent one, $\ell > 1$, so as to avoid performing the ReLU non--linearity immediate on the input.}
\label{tab:resblock}
\end{table}

\begin{table}	\centering
\ras{1.2}
\begin{tabular}{c@{\hskip 5em}c}\toprule
\textbf{Generator} & \textbf{Discriminator}\\\toprule
\textit{Input:} $z \in \mathds{R}^{128} \sim \mathcal{N}(0, I) $ &
\textit{Input:} $x \in \mathds{R}^{3{\times}32{\times}32} $  \\ \hdashline 
Linear($128 \rightarrow 4096$)	& D--ResBlock \\
G--ResBlock 	& D--ResBlock \\
G--ResBlock 	& D--ResBlock \\
G--ResBlock 	& D--ResBlock \\
Batch Normalization & ReLU \\
ReLU  				& AvgPool (ker:$8{\times}8$ ) \\
conv. (ker: $3{\times}3$, $256 \rightarrow 3$; stride: $1$; pad:1) & 
Linear($128 \rightarrow 1$) \\
 $Tanh(\cdot)$  & Spectral Normalization \\ 
\bottomrule \\\end{tabular}
\caption{ \textit{Deep} ResNet architectures used for experiments on \textbf{SVHN} and \textbf{CIFAR-10}, where G--ResBlock and D--ResBlock for the Generator (left) and the Discriminator (right), respectively, are described in Table~\ref{tab:resblock}. The models' parameters are initialized using the Xavier initialization~\citep{glorot2010}.}\label{tab:resnet_arch}
\end{table}

\clearpage
\section{Additional Experiments}\label{sec:additional-experiments}
\subsection{Results on MNIST}\label{sec:results_mnist}

\begin{table}\centering 
\begin{tabular}{r ccc c ccc}\toprule 
        & \multicolumn{3}{c}{ IS}   
        & \phantom{a}
        & \multicolumn{3}{c}{ FID}  \\
        \cmidrule{2-4} \cmidrule{6-8} 
        & SE--A & SVRE  & SVRE--VRAd &
        & SE--A & SVRE & SVRE--VRAd  \\ \midrule   
MNIST       &  $8.62 $ & $8.58$ &  $8.56$ &
           & $0.17$ & $0.15$ & $0.18$ \\
CIFAR-10      &  $6.61$  & $6.50$   & $\bm{6.67}$&
          & $37.20$& $39.20$ & $38.88$\\
SVHN      &  $2.83$  & $3.01$   & $\bm{3.04}$&
            & $39.95$ & $24.01$ & $\bm{19.40}$ \\
ImageNet     &  $7.22$  & $\bm{8.08}$  & $7.50$ &
          & $89.40$ & $\bm{75.60}$ & $81.24$ \\
\bottomrule
\end{tabular}
\caption{Best obtained IS and FID scores for the different optimization methods, using \textit{shallow} architectures,
for a fixed number of iterations (see \S~\ref{sec:impl-details}). The architectures for each dataset are described in: \textbf{MNIST}--Table~\ref{tab:mnist_arch}, \textbf{SVHN} and \textbf{CIFAR-10}--Table~\ref{tab:sagan_shallow_arch}, and \textbf{ImageNet}--Table~\ref{tab:sagan_shallow_arch_imgnet}.
The standard deviation of the Inception scores is around $0.1$ and is omitted.
Although the IS metric gives relatively close values on \textbf{SVHN} due to the dataset properties (see \S~\ref{sec:app-metrics}), we include it for completeness.
\label{tab:res_summary_shallow} }
\end{table}

The results in Table~\ref{tab:res_summary} on \textbf{MNIST} are obtained using $5$ runs with different seeds, and the shown performances are the averaged values. 
Each experiment was run for $100K$ iterations.
The corresponding scores with the standard deviations are as follows:
\begin{enumerate*}[series = tobecont, itemjoin = \quad, label=(\roman*)]
\item IS: $8.62{\pm}.02 $, $8.58{\pm}.08$, $8.56{\pm}.11$;
\item FID: $0.17{\pm}.03$, $0.15{\pm}.01$, $0.18{\pm}.02$; 
\end{enumerate*}
for SE--A, SVRE, and SVRE--VRAd, respectively.
On this dataset, we obtain similar final performances if run for many iterations, however SVRE converges faster (see Fig.~\ref{fig:extra_mnist}).
Fig.~\ref{fig:extra_mnist-second} illustrates additional metrics of the experiments shown in Fig.~\ref{fig:extra_mnist}.

\label{sub:mnist-experiments}
\begin{figure}[!htb]
    \centering
    \centering
     \begin{subfigure}[t]{0.48\linewidth}
        \centering
        \includegraphics[width=\linewidth]{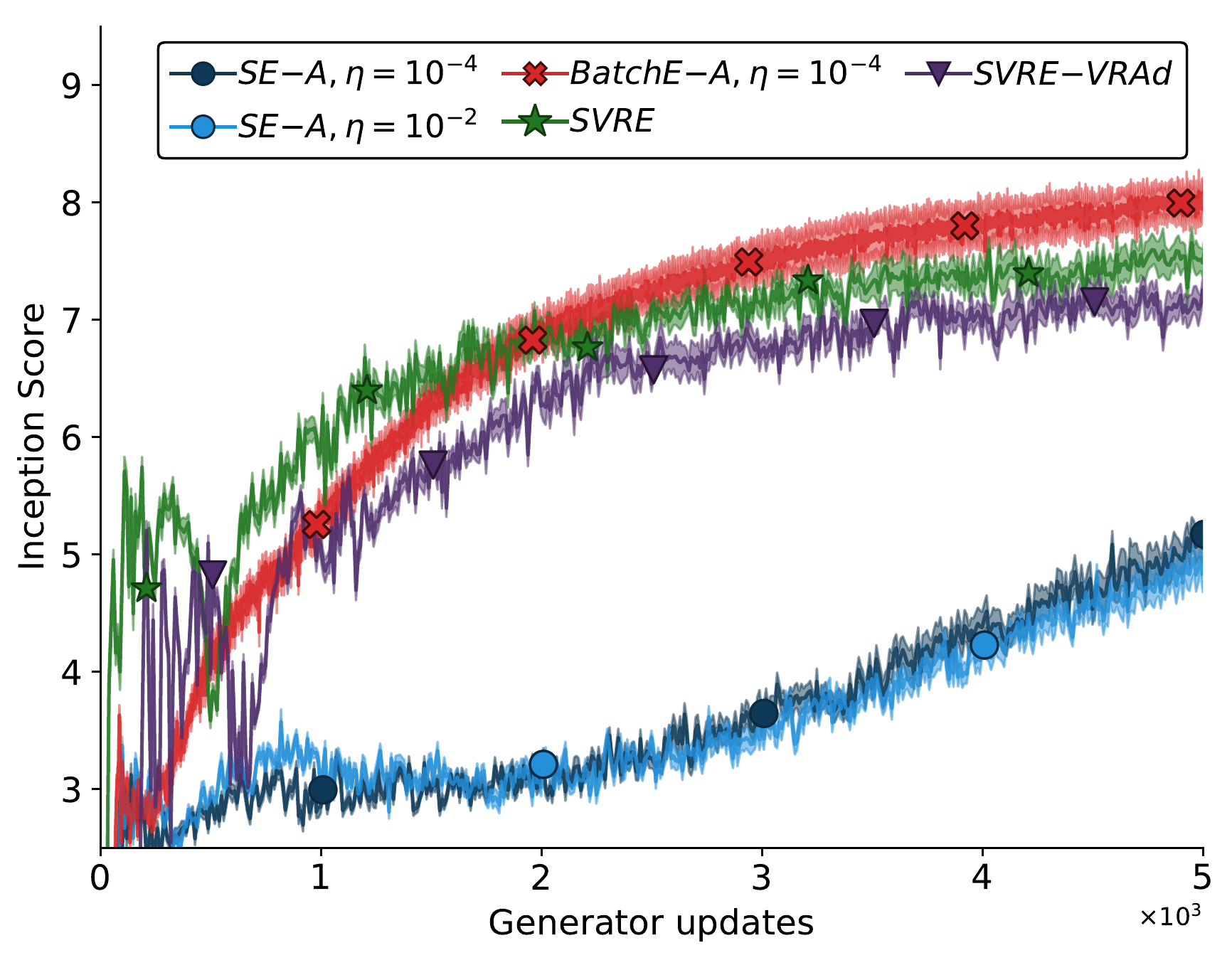}
        \caption{IS (higher is better)}\label{subfig-mnist_is_param_updates}
    \end{subfigure}
    \begin{subfigure}[t]{0.48\linewidth}
        \centering
        \includegraphics[width=\linewidth]{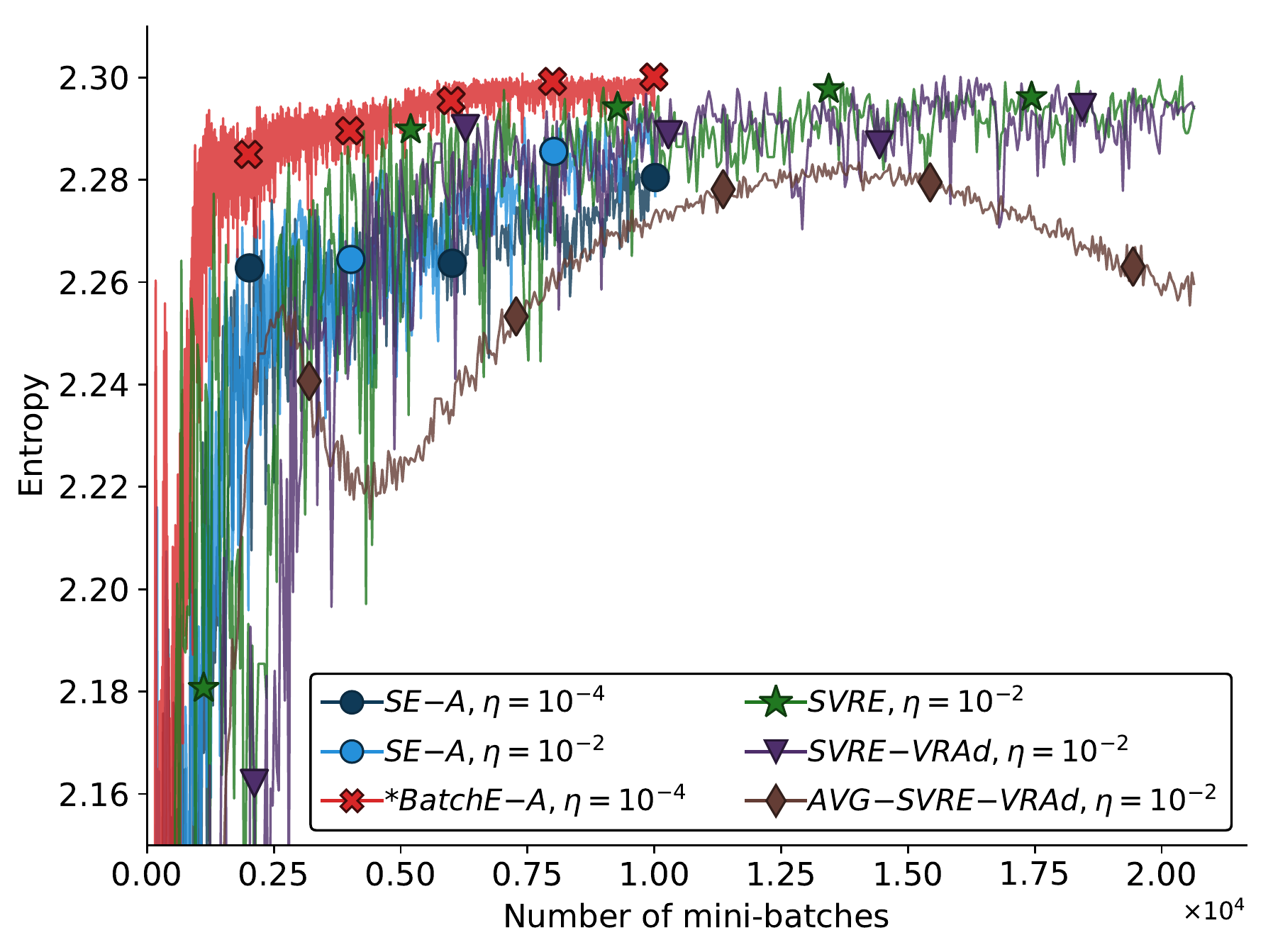}
        \caption{Entropy (higher is better) }\label{subfig-mnist_entropy}
    \end{subfigure}
    \begin{subfigure}[t]{0.48\linewidth}
        \centering
        \includegraphics[width=\linewidth]{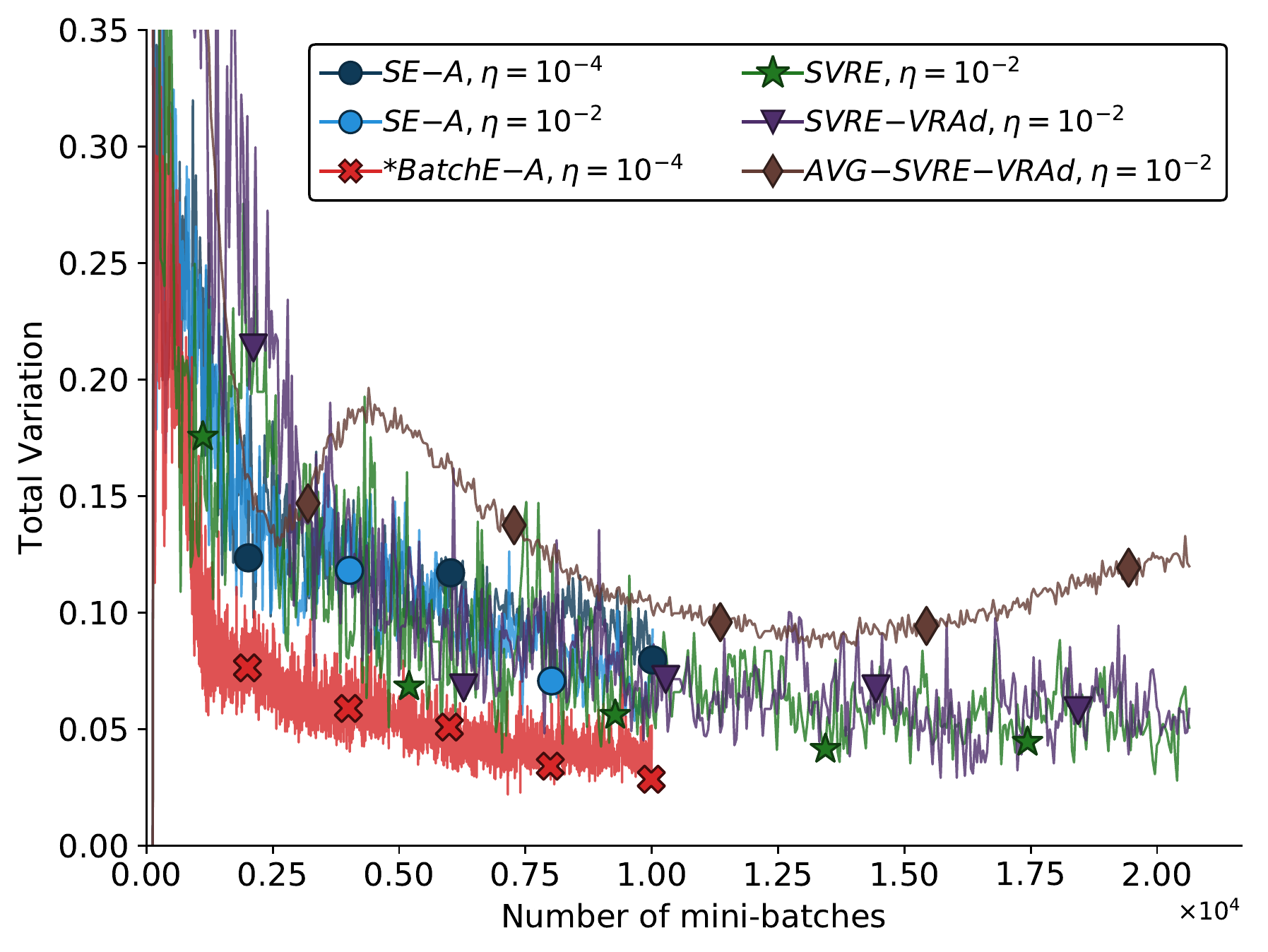}
        \caption{Total variation (lower is better)}\label{subfig-mnist_tv}
    \end{subfigure}
    \begin{subfigure}[t]{0.48\linewidth}
        \centering
        \includegraphics[width=\linewidth]{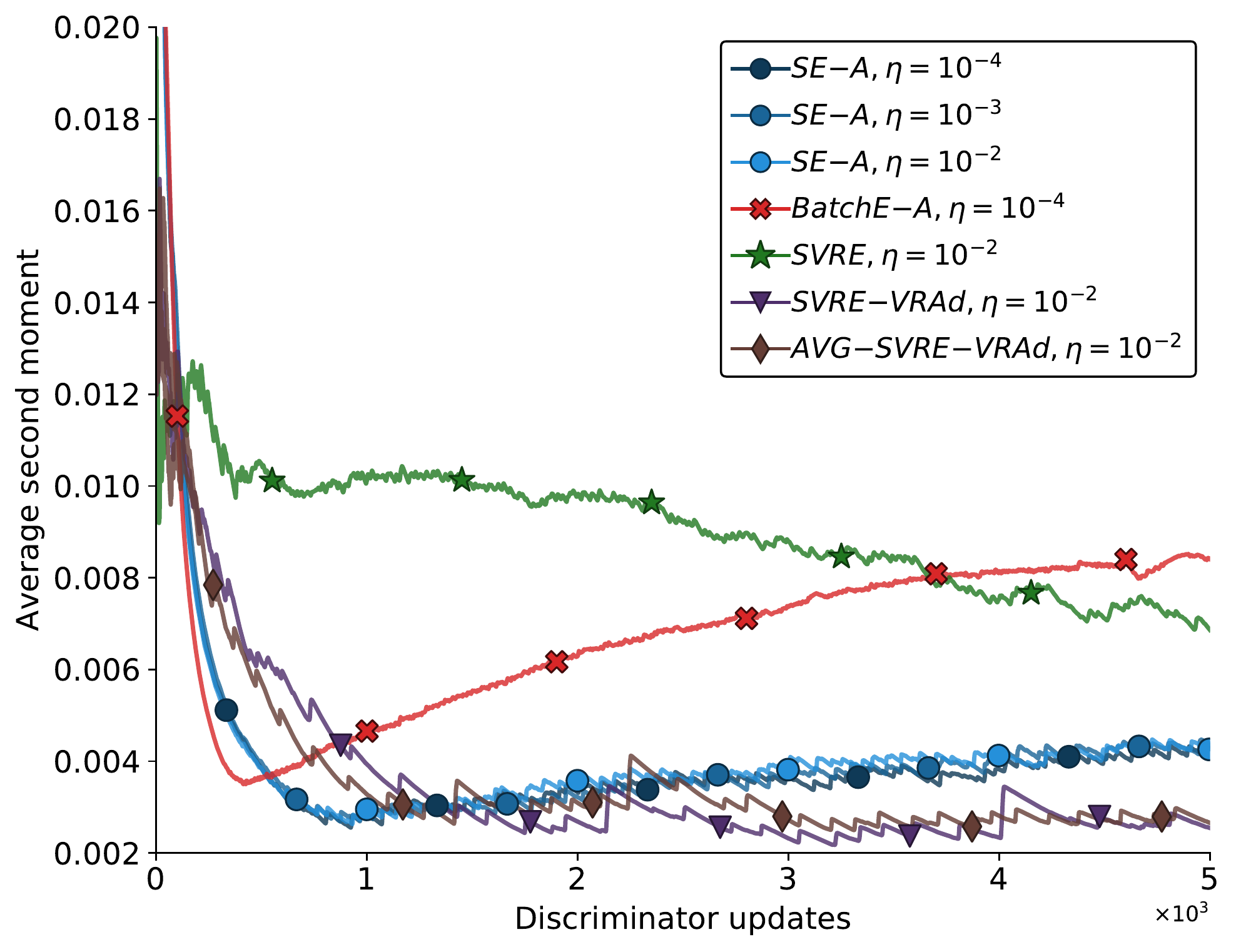}
        \caption{Discriminator}\label{subfig-var_d}
    \end{subfigure}
    \caption{ Stochastic, full-batch and variance reduced versions of the extragradient method ran on \textbf{MNIST}, see \S~\ref{sec:mnist}.
   \textit{*BatchE--A} emphasizes that this method is \textbf{not} scaled with the number of passes (x-axis). The input space is $1{\times}28{\times}28$, see \S~\ref{app:arch} for details on the implementation. }\label{fig:extra_mnist-second}
\end{figure}

\subsection{Results with shallow architectures}\label{sec:results_shallow_arch}
Fig.~\ref{fig:extra_imagenet} depicts the results on \textbf{ImageNet} using the \textit{shallow} architectures described in Table~\ref{tab:sagan_shallow_arch_imgnet}, \S~\ref{sec:shallow_sagan}.
Table~\ref{tab:res_summary_shallow} summarizes the results obtained on \textbf{SVHN}, \textbf{CIFAR-10} and \textbf{ImageNet} with these architectures.
Fig.~\ref{fig:extra_svhn_sme} depicts the SME metric (see \S~\ref{sec:sme}) for the the SE--A baseline and SVRE shown in Fig.~\ref{subfig-fid_svhn}, on \textbf{SVHN}.

\begin{figure}[!htb]
    \centering
    \begin{subfigure}[t]{0.495\linewidth}
        \centering
        \includegraphics[width=\linewidth]{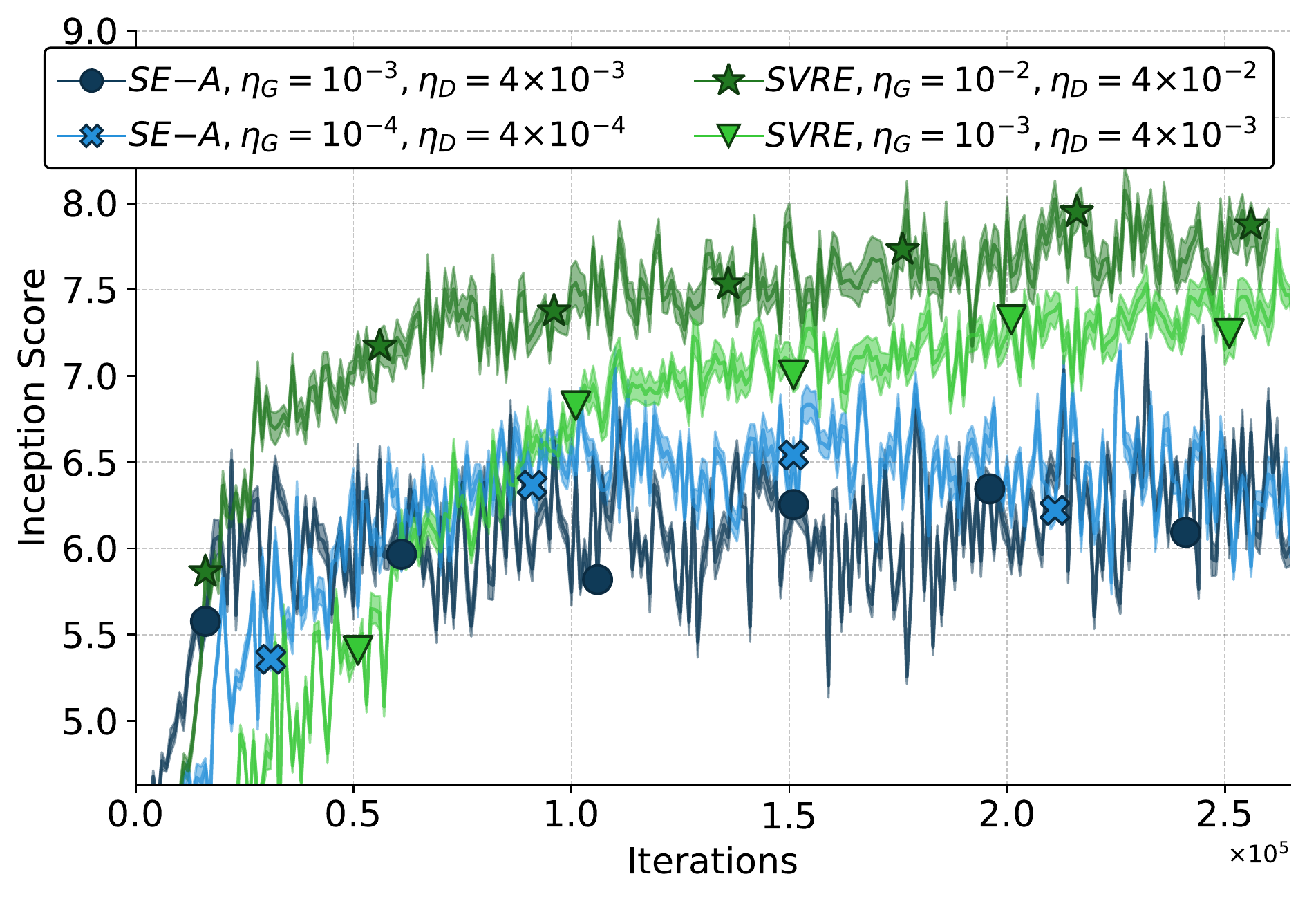}
        \caption{IS (higher is better)}
    \end{subfigure}
    \begin{subfigure}[t]{0.495\linewidth}
        \centering
        \includegraphics[width=\linewidth]{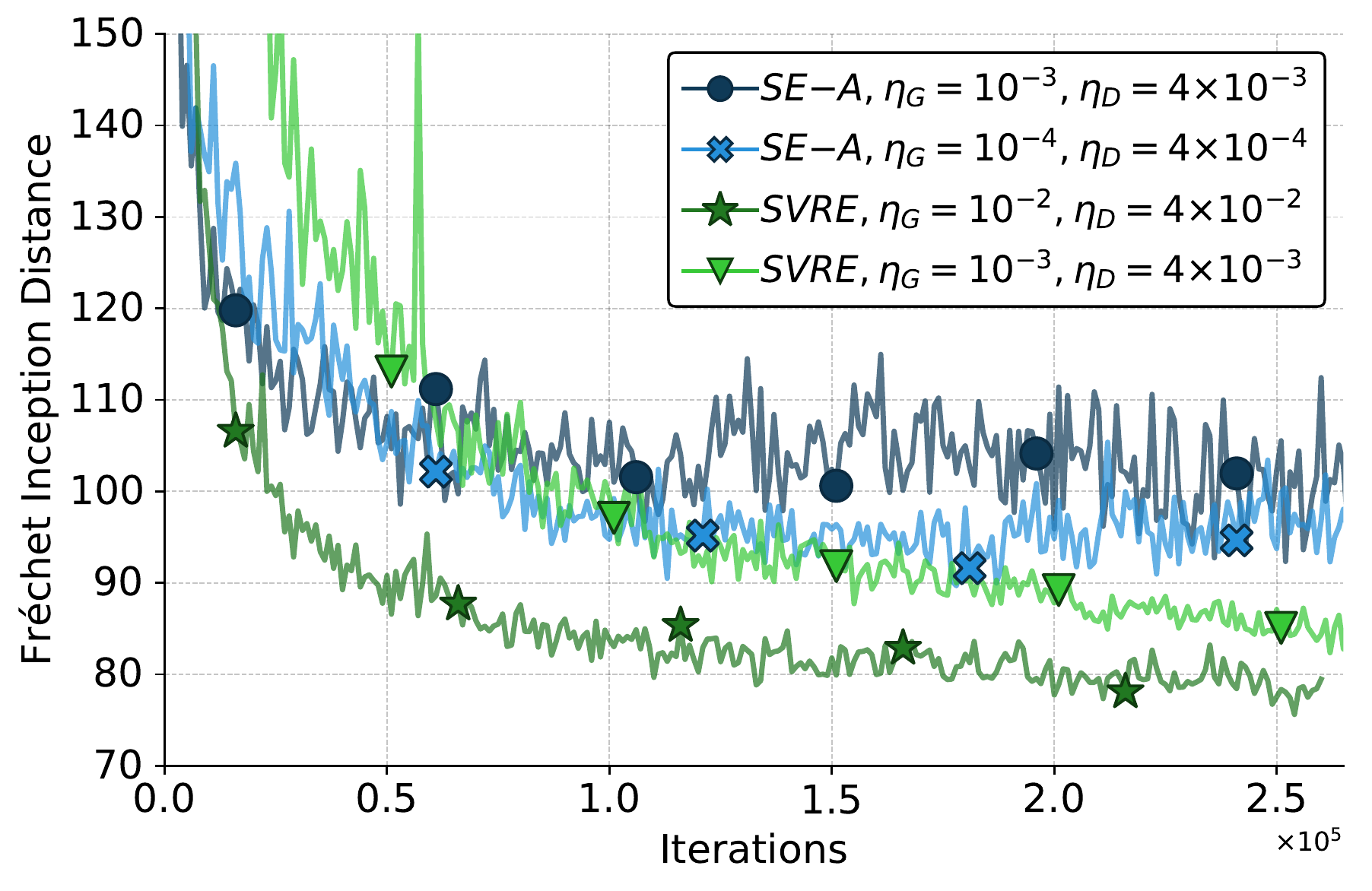}
        \caption{FID (lower is better)}
    \end{subfigure}
    \caption{Comparison between \textit{SVRE} and the \textit{SE--A} baseline on \textbf{Imagenet}, using the \textit{shallow} architectures described in Table~\ref{tab:sagan_shallow_arch_imgnet}.
    See \S~\ref{sec:app-metrics} for details on the used IS and FID metrics.}
    \label{fig:extra_imagenet}
\end{figure}

\begin{figure}[!htb]
    \centering
    \begin{subfigure}[t]{0.495\linewidth}
        \centering
        \includegraphics[width=\linewidth]{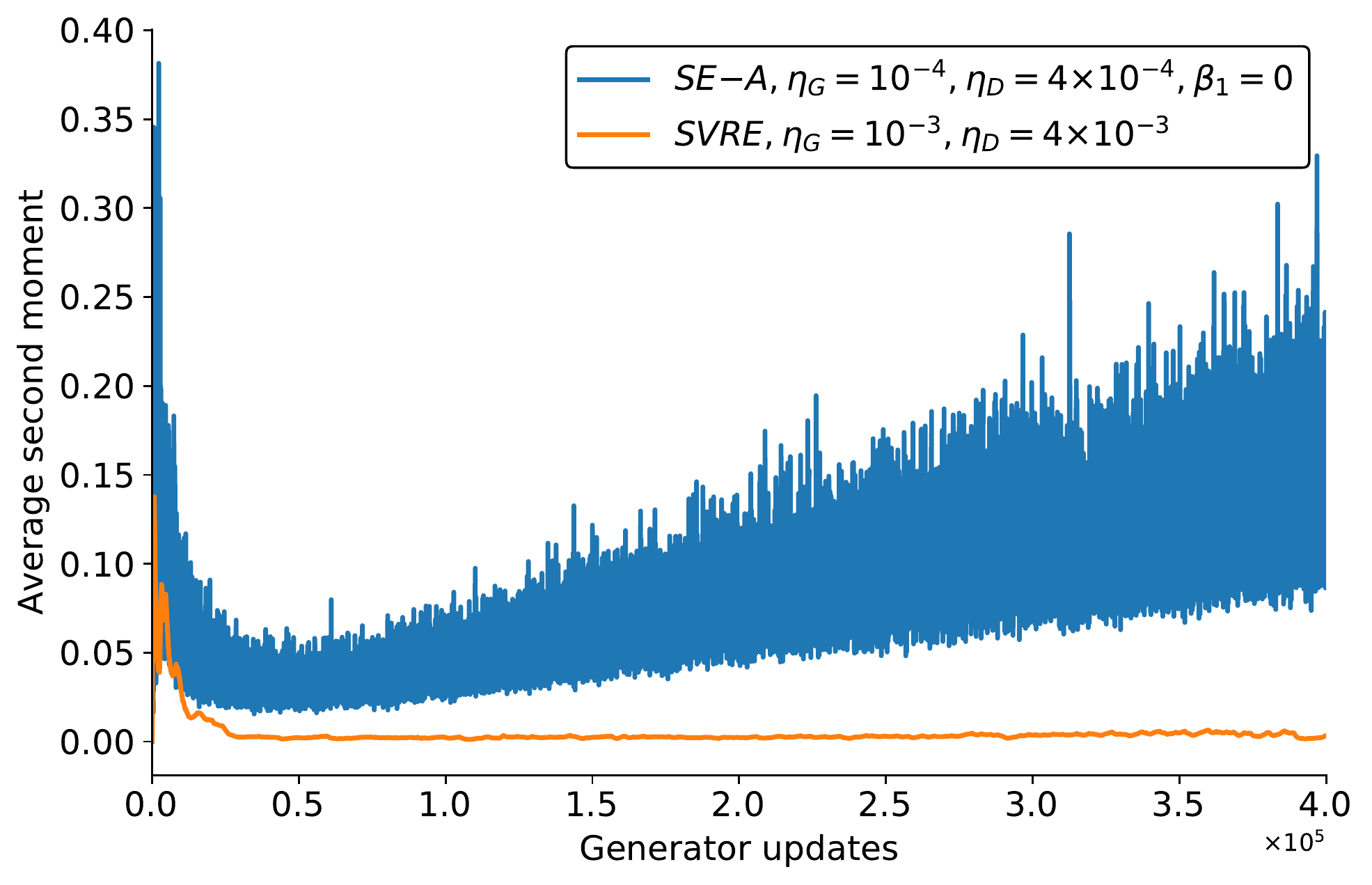}
        \caption{Generator}
    \end{subfigure}
    \begin{subfigure}[t]{0.495\linewidth}
        \centering
        \includegraphics[width=\linewidth]{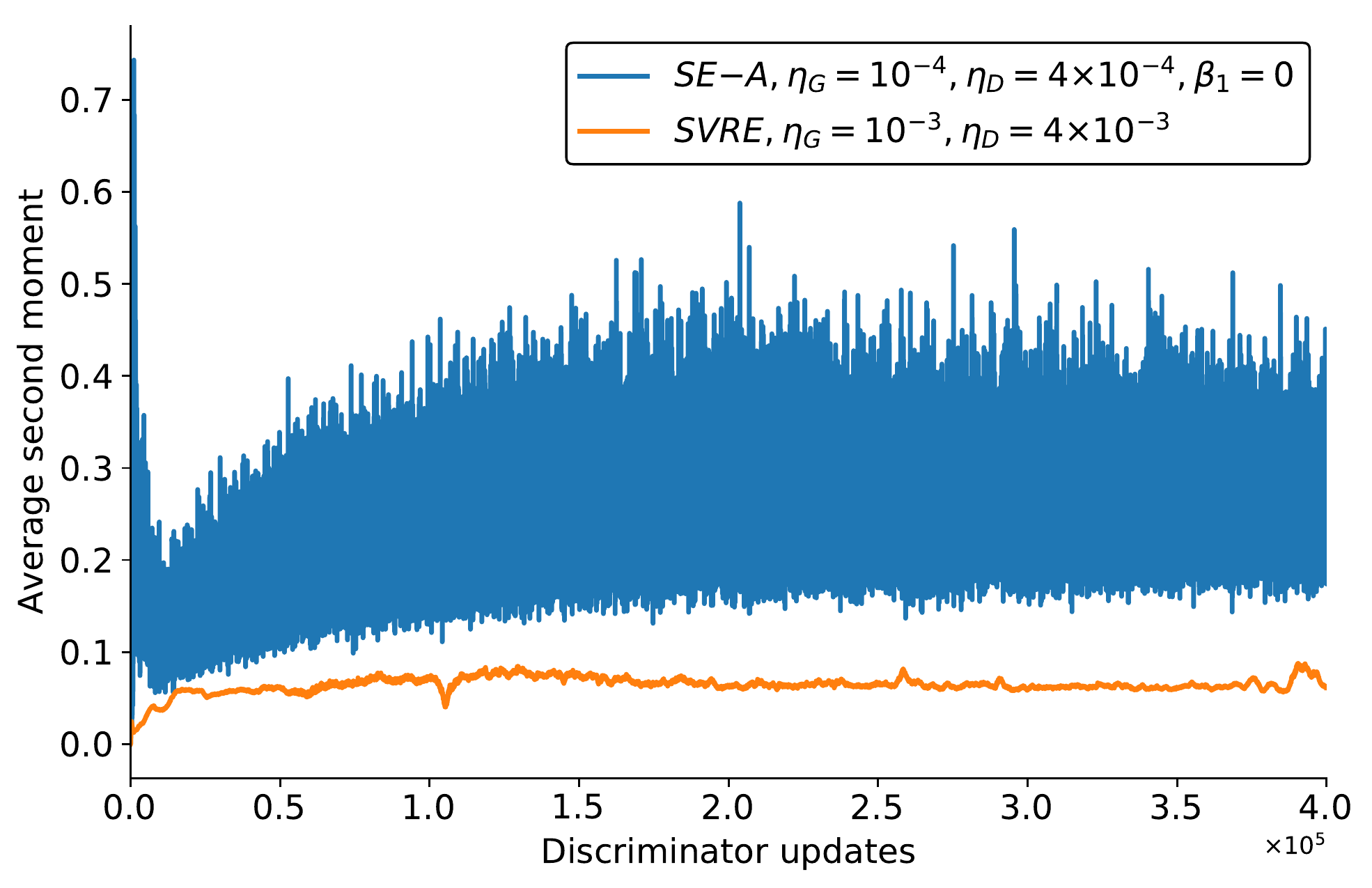}
        \caption{Discriminator}
    \end{subfigure}
    \caption{Average second moment estimate (see \S~\ref{sec:sme}) on \textbf{SVHN} for the Generator (left) and the Discriminator (right), using the \textit{shallow} architectures described in Table~\ref{tab:sagan_shallow_arch}.
    The corresponding FID scores for these experiments are shown in Fig.~\ref{subfig-fid_svhn}.}
    \label{fig:extra_svhn_sme}
\end{figure}

\clearpage
\subsection{Results with deeper architectures}\label{sec:results_deep_arch}
We observe that GAN training is more challenging when using \emph{deeper} architectures and some empirical observations differ in the two settings. 
For example, our stochastic baseline is drastically more unstable and often does not start to converge, whereas SVRE is notably \textit{stable}, but slower compared to when using shallower architectures.
In this section, all our discussions focus on \textit{deep} architectures (see \S~\ref{sec:deeper_resnet_arch}).

\paragraph{Stability: convergence of the GAN training.}
For our stochastic baselines, \emph{irrespective whether we use the extragradient or gradient method}, we observe that the convergence is notably more \textit{unstable} (see Fig.~\ref{fig:stochastic_hyperparams}) when using the \textit{deep} architectures described in \S~\ref{sec:deeper_resnet_arch}.
More precisely, either the training fails to converge or it diverges at later iterations.
When updating G and D equal number of times \textit{i.e.} using $1:1$ update ratio, using SE--A on \textbf{CIFAR-10} we obtained best FID score of $24.91$ using $\eta_G=2\times 10^{-4}$, $\eta_D=4\times 10^{-4}, \beta_1=0$, while experimenting with several combinations of $\eta_G, \eta_D, \beta_1$.
Using exponential learning rate decay with a multiplicative factor of $0.99$, improved the best FID score to $20.70$, obtained for the experiment with $\eta_G=2\times 10^{-4}$, $\eta_D=2\times 10^{-4}, \beta_1=0$.
Finally, using $1:5$ update ratio, with $\eta_G=2\times 10^{-4}$, $\eta_D=2\times 10^{-4}, \beta_1=0$ provided best FID of $18.65$ for the baseline.
Figures~\ref{fig:se-a_cif10} and~\ref{fig:sg-a_svhn} depict the hyper-parameter sensitivity of SE--A and SG--A, respectively.
The latter denotes the alternating GAN training with Adam, that is most commonly used for GAN training.

\begin{figure}[!htb]
    \centering
    \begin{subfigure}[t]{0.495\linewidth}
        \centering
        \includegraphics[width=\linewidth]{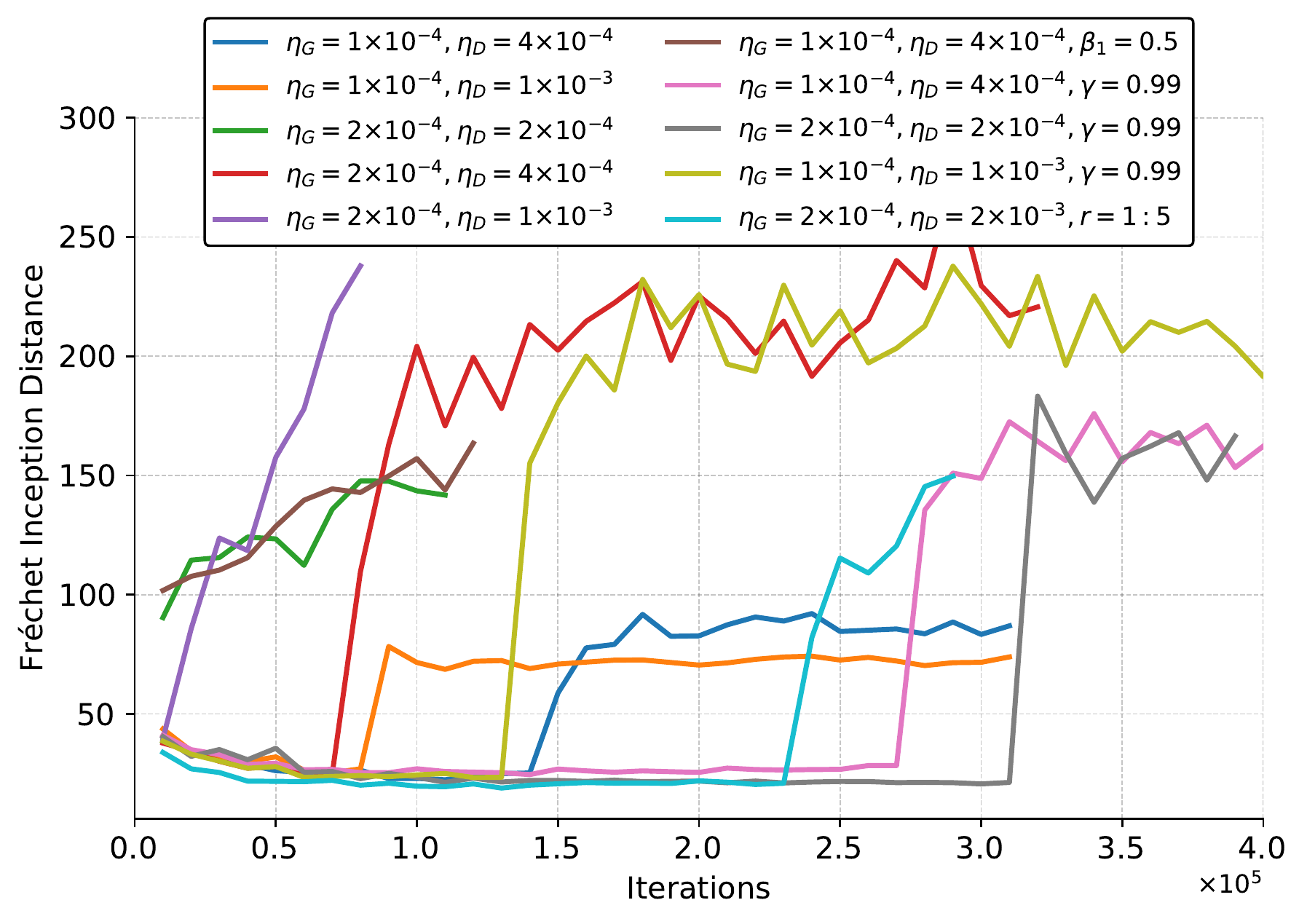}
        \caption{SE--A, \textbf{CIFAR-10}}\label{fig:se-a_cif10}
    \end{subfigure}
    \begin{subfigure}[t]{0.495\linewidth}
        \centering
        \includegraphics[width=\linewidth]{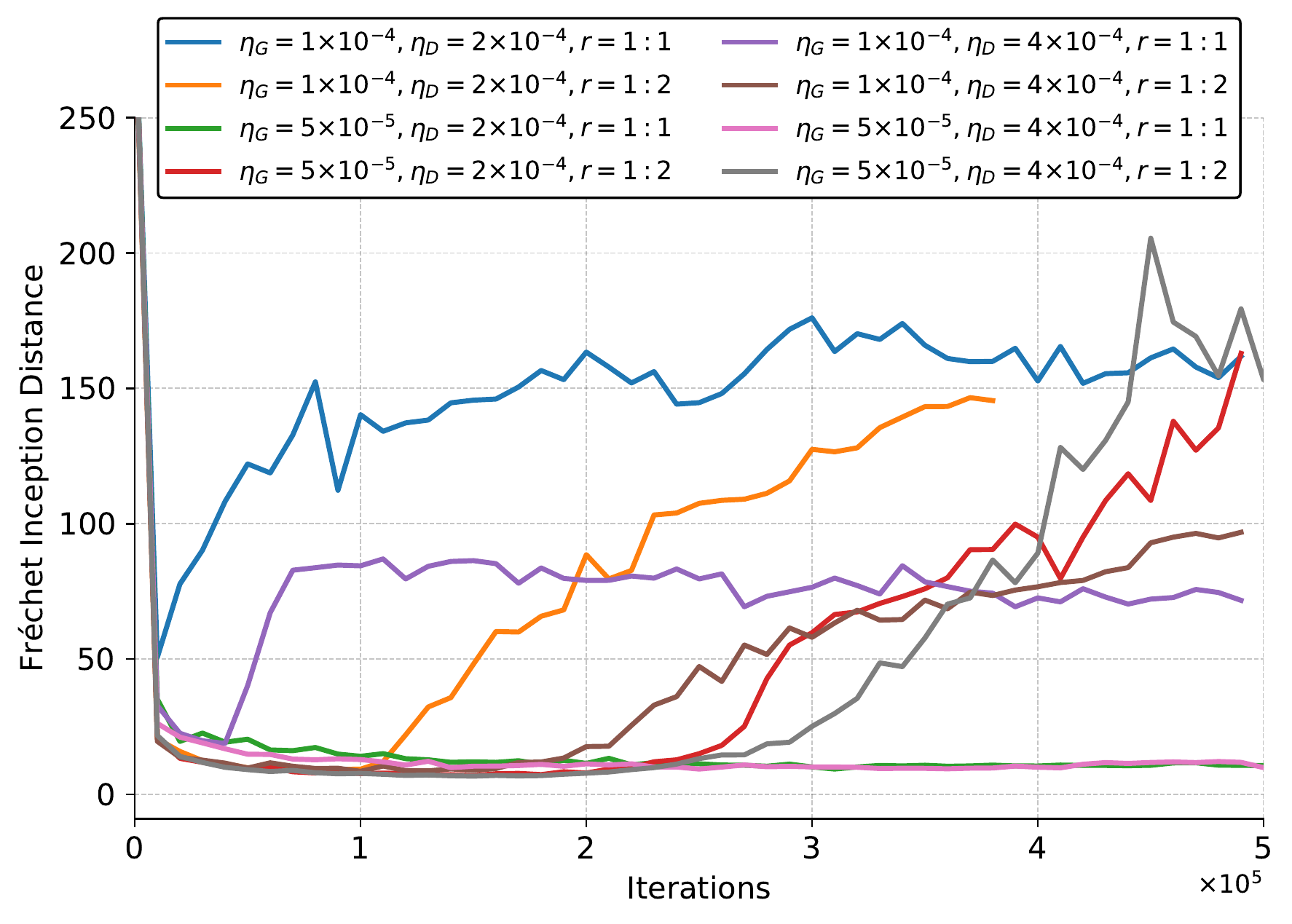}
        \caption{SG--A, \textbf{SVHN}}\label{fig:sg-a_svhn}
    \end{subfigure}
    \caption{FID scores (lower is better) with different hyperparameters for the SE--A baseline on \textbf{CIFAR10} (left) and the SG--A baseline on \textbf{SVHN} (right), using the \textit{deep} architectures described in Table~\ref{tab:resnet_arch}, \S~\ref{sec:deeper_resnet_arch}. 
    SG--A denotes the standard stochastic \textit{alternating} GAN training, with the Adam optimization method.
    Where omitted, $\beta_1=0$, see \eqref{eq:adam_beta1} where this hyperparameter is defined. 
    With $r$ we denote the update ratio of generator versus discriminator: in particular $1:5$ denotes that $D$ is updated $5$ times for each update of $G$.
    $\gamma$ denotes a multiplicative factor of exponential learning rate decay scheduling.
    In Fig.~\ref{fig:sg-a_svhn}, $\gamma=0.99$ for all the experiments.
    We observed in all our experiments that training diverged in later iterations for the stochastic baseline, when using \textit{deep} architectures.}
    \label{fig:stochastic_hyperparams}
\end{figure}

\begin{figure}[!htb]
    \centering
    \includegraphics[width=.9\linewidth]{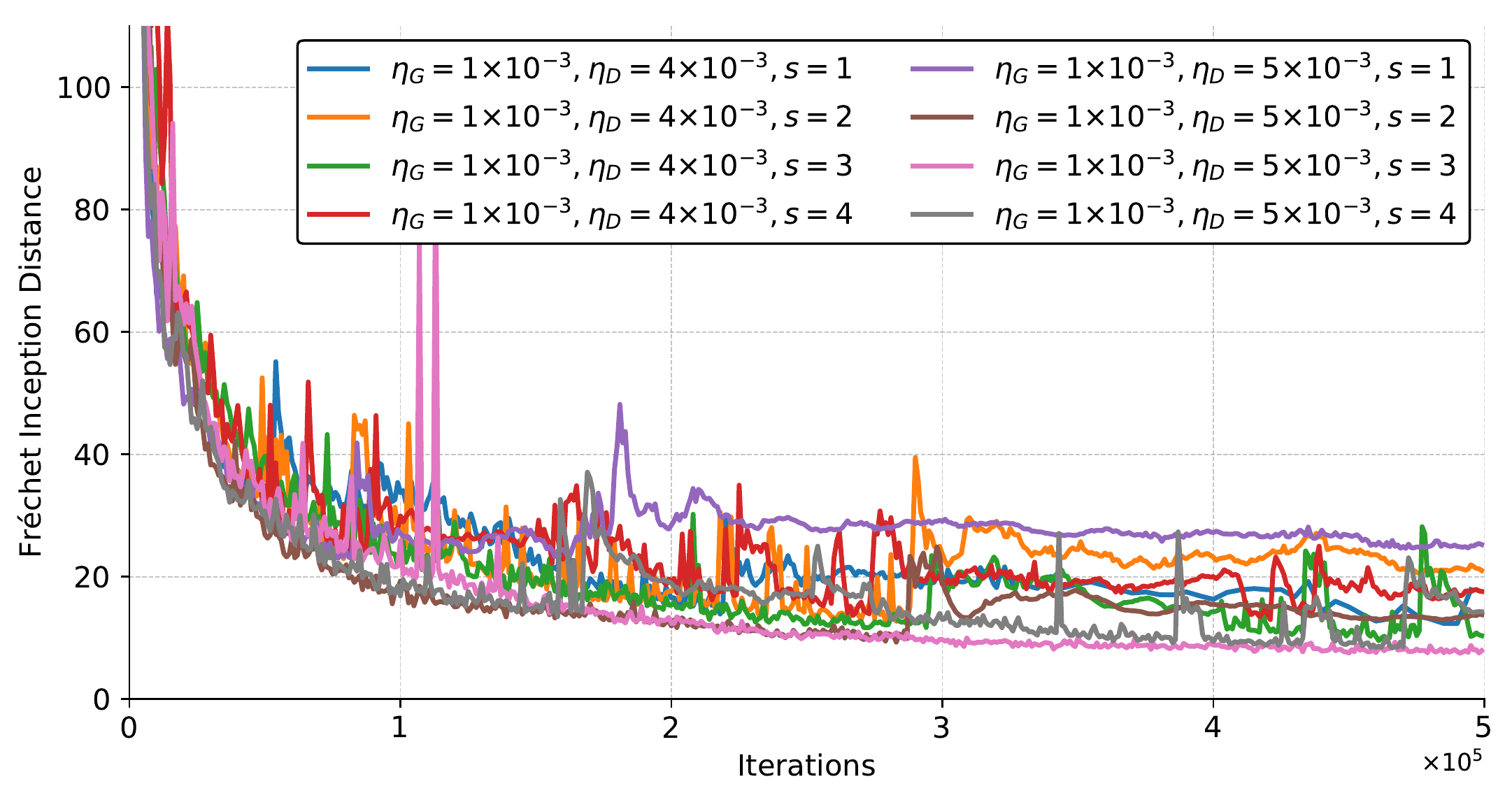}
    \caption{ Obtained FID (lower is better)  scores for SVRE, using the \textit{deep} architectures (see \S~\ref{sec:deeper_resnet_arch}) on \textbf{SVHN}. 
    With $s$ we denote the fixed random seed.
    The update ratio for all the experiments is $1:1$.
    We illustrate our results on the same plot (besides the reduced clarity) so as to summarize our observation that, contrary to the SE--A baseline for these architectures, SVRE \emph{always converges}, and does not diverge.
    }
    \label{fig:deep_svre_svhn_fid}
\end{figure}
\begin{figure}[!htb]
    \centering
    \includegraphics[width=.65\linewidth]{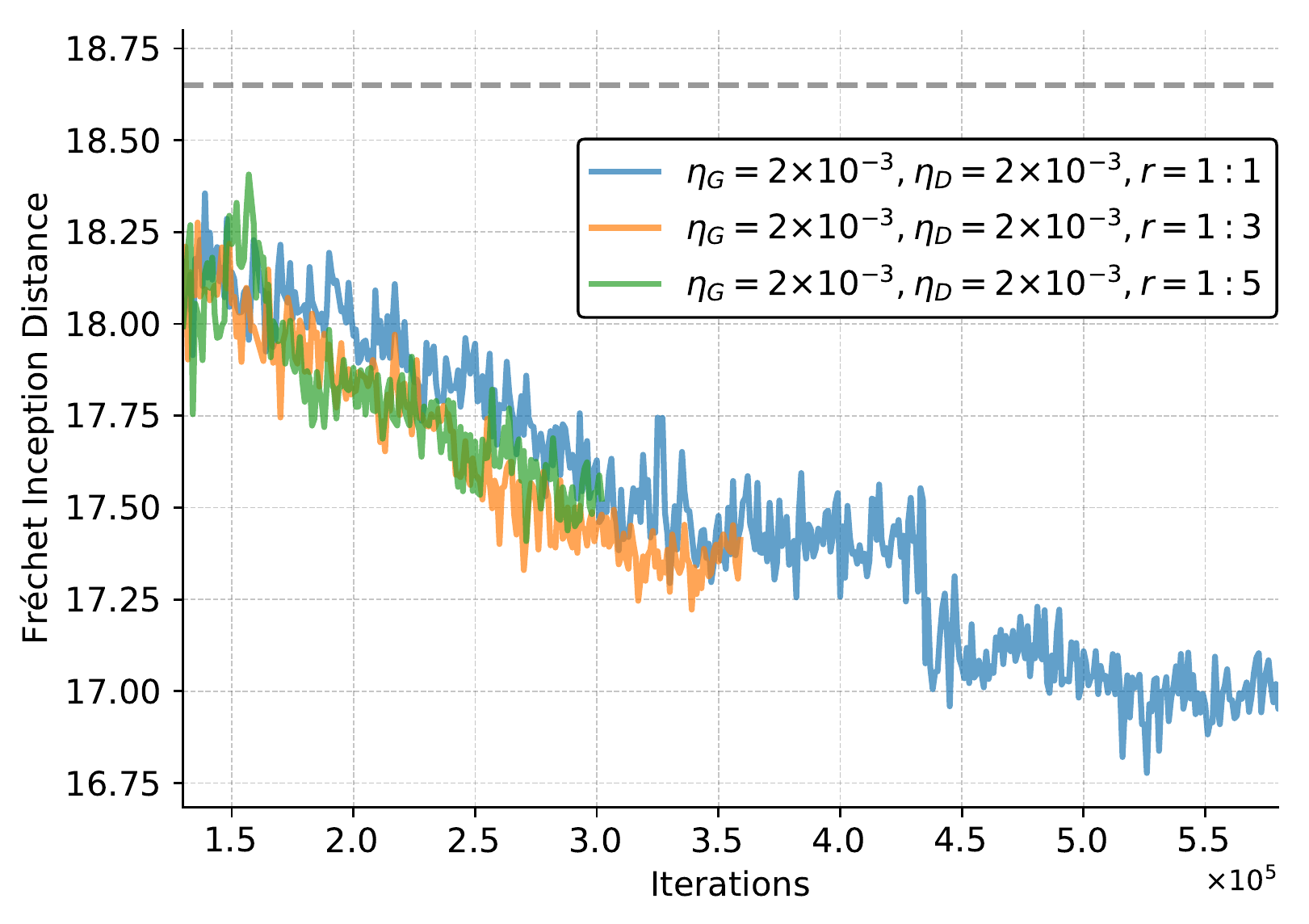}
    \caption{ Obtained FID (lower is better)  scores for WS--SVRE, using the \textit{deep} architectures (see \S~\ref{sec:deeper_resnet_arch}) on \textbf{CIFAR-10}, where the seed is fixed to $1$ for all the experiments.
    With $r$ we denote the update ratio of generator versus discriminator: in particular $1:5$ denotes that $D$ is updated $5$ times for each update of $G$.
    We start from the best obtained FID score for the stochastic baseline, i.e. FID of $18.65$ (see Table~\ref{tab:res_summary})--shown with dashed line, and we continue to train with SVRE.
    }
    \label{fig:deep_ws-svre_cif10_fid}
\end{figure}

We observe that SVRE is more stable in terms of hyperparameter selection, as it always starts to converge and \emph{does not diverge} at later iterations. Relative to experiments with shallower architectures, we observe that with deeper architectures SVRE takes longer to converge than its baseline for this architecture. With constant step size of $\eta_G=1\times 10^{-3}$, $\eta_D=4\times 10^{-3}$ we obtain FID score of $23.56$ on \textbf{CIFAR-10}. 
Note that this result outperforms the baseline when using no additional tricks (which themselves require additional hyperparameter tuning).
Fig.~\ref{fig:deep_svre_svhn_fid} depicts the FID scores obtained when training with SVRE on the \textbf{SVHN} dataset, for two different hyperparameter settings, using four different seeds for each.
From this set of experiments, we observe that contrary to the baseline that either did not converge or diverged in all our experiments, SVRE always converges.
However, we observe different performances for different seeds. 
This suggests that more exhaustive empirical hyperparameter search that aims to find an empirical setup that works \emph{best} for SVRE or further combining SVRE with adaptive step size techniques are both promising research directions (see our discussion below). 
Fig.~\ref{fig:deep_ws-svre_cif10_fid} depicts our WS--SVRE experiment, where we start from a stored checkpoint for which we obtained best FID score for the SE--A baseline, and we continue the training with SVRE.
It is interesting that besides that the baseline diverged after the stored checkpoint, SVRE further reduced the FID score.
Moreover, we observe that using different update ratios does not impact much the performance, what on the other hand was necessary to make the baseline algorithm converge.

\begin{figure}[!htb]
    \centering
    \begin{subfigure}[t]{0.495\linewidth}
        \centering
        \includegraphics[width=\linewidth]{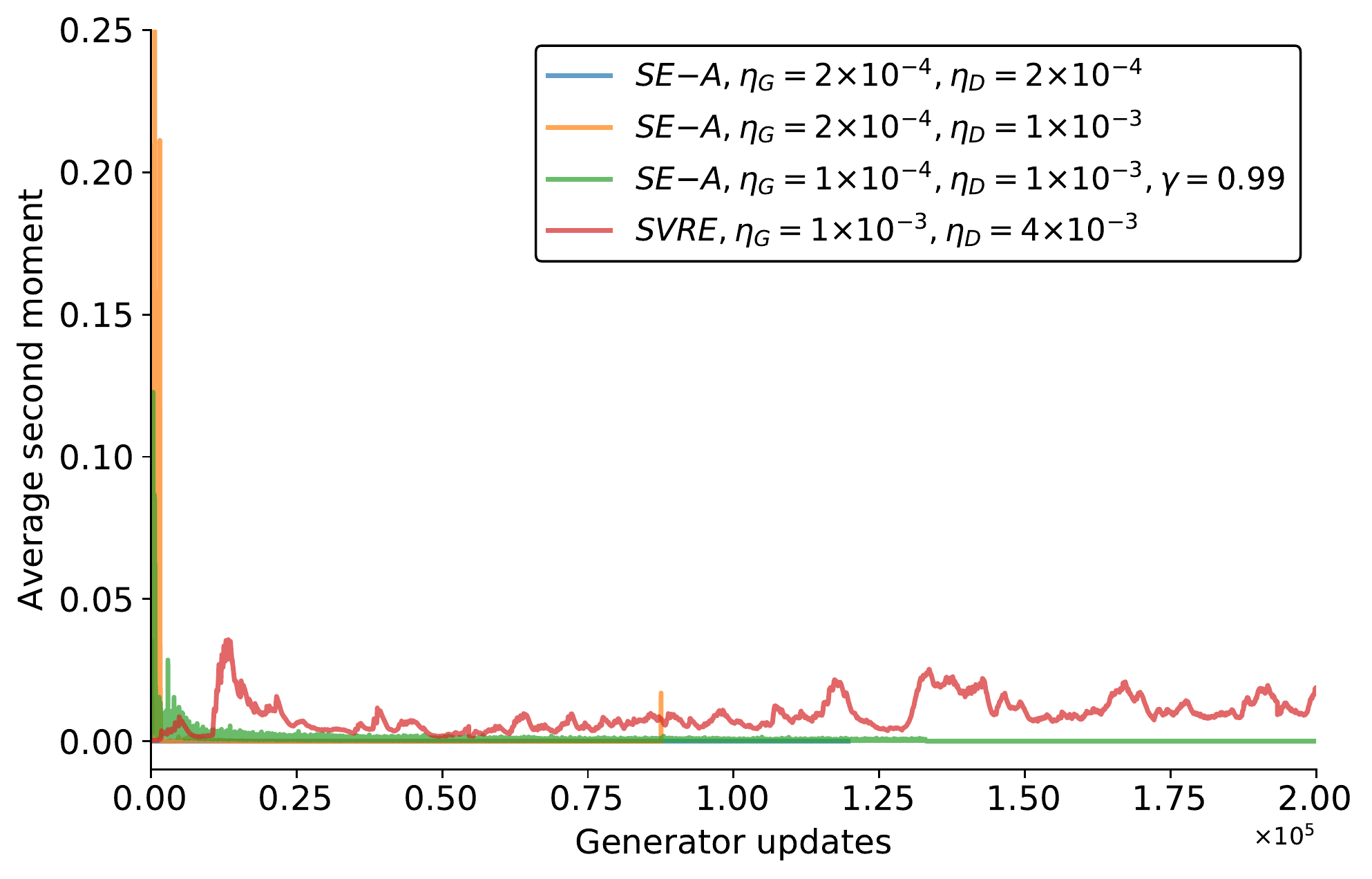}
        \caption{Generator}
    \end{subfigure}
    \begin{subfigure}[t]{0.495\linewidth}
        \centering
        \includegraphics[width=\linewidth]{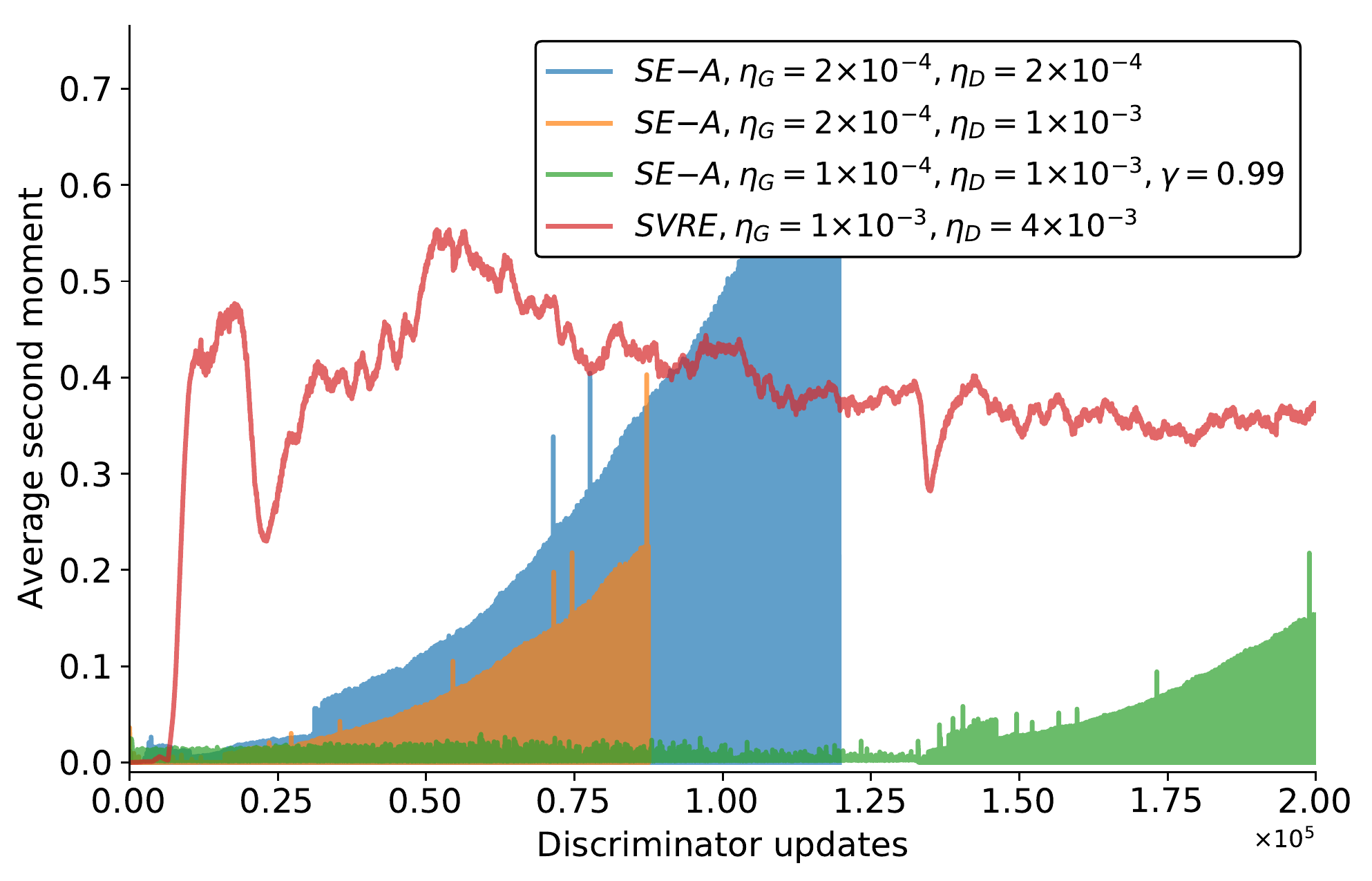}
        \caption{Discriminator}
    \end{subfigure}
    \caption{Average second moment estimate (SME, see \S~\ref{sec:sme}) on \textbf{CIFAR-10} for the Generator (left) and the Discriminator (right), using the \textit{deep} architectures described in Table~\ref{tab:resnet_arch}.
    The obtained FID scores for these experiments are shown in Fig.~\ref{fig:se-a_cif10}, where we omit some of the experiments for clarity.
    All of the baseline SE--A experiments diverge at some point, what correlates with the iterations at which large oscillations of SME appear for the Discriminator.
    Note that the SE--A experiments were stopped after the algorithm diverges, hence the plotted SME is up to a particular iteration for two of the experiments (shown in blue and orange). 
    The SE--A experiment with $\gamma=0.99$ diverged at later iteration relative to the experiments without learning rate decay, and has lower SME.}
    \label{fig:extra_cif10_deep_sme}
\end{figure}

\paragraph{Second moment estimate (SME).}
Fig.~\ref{fig:extra_cif10_deep_sme} depicts the second moment estimate (see \S~\ref{sec:sme}) for the experiments with \textit{deep} architectures. 
We observe that:
\begin{enumerate*}[series = tobecont, itemjoin = \quad, label=(\roman*)]
\item the estimated SME  quantity is more bounded and changes more smoothly for SVRE (as we do not observe large oscillations of it as it is the case for SE--A); as well as that
\item divergence of the SE--A baseline \emph{correlates} with large oscillations of SME, in this case, observed for the Discriminator.
\end{enumerate*}
Regarding the latter, there exist larger in magnitude oscillations of SME (note that the exponential moving average hyperparameter for computing SME is $\gamma=0.9$, see \S~\ref{sec:sme}).

\paragraph{Conclusion \& future directions.}
In summary, we observe the following most important advantages of SVRE when using \textit{deep} architectures:
\begin{enumerate*}[series = tobecont, itemjoin = \quad, label=(\roman*)]
\item consistency of convergence, and improved stability; as well as
\item reduced number of hyperparameters.
\end{enumerate*}
Apart from the practical benefit for applications, the former could allow for a more fair comparison of GAN variants.
The latter refers to the fact that SVRE omits the tuning of the sensitive (for the stochastic baseline) $\beta_1$ hyperparameter (see \eqref{eq:adam_beta1}), as well as $r$ and $\gamma$--as training converges for SVRE without using different update ratio and step size schedule, respectively.
It is important to note that the stochastic baseline does not converge when using constant step size (i.e. when \textit{SGD} is used instead of \textit{Adam}). 
In our experiments we compared SVRE that uses constant step size, with Adam, making the comparison unfair toward SVRE.
Hence, our results indicate that SVRE can be further combined with adaptive step size schemes, so as to obtain both stable GAN performances and fast convergence when using these architectures. 
Nonetheless, the fact that the baseline either does not start to converge or it diverges later makes SVRE and WS--SVRE a promising approach for practitioners using these \emph{deep} architectures, whereas, for \textit{shallower} ones, SVRE speeds up the convergence and often provides better \emph{final} performances.

\checknbdrafts

\end{document}